%% file: main.tex
\documentclass{article}

\usepackage{microtype}
\usepackage{graphicx}
\usepackage{booktabs} 

\usepackage[accepted]{icml2024}

\usepackage{amsmath}
\usepackage{amssymb}
\usepackage{mathtools}
\usepackage{amsthm}

\usepackage{graphicx}
\usepackage{setspace}
\usepackage{nicefrac}
\usepackage{pgfplots}
\usepackage[inline]{enumitem}
\usepackage{pifont}
\usepackage{booktabs}
\usepackage{subcaption,float}
\usepackage{tabularx,booktabs,threeparttable,makecell,colortbl,dblfloatfix}
\usepackage{nicematrix}
\usepackage{calc}
\usepackage{varwidth}

\input{krkmath}

\usepackage{etoc}
\etocframedstyle[1]{\textbf{\textsc{Table of Contents}}}
\etocsettocdepth{3}
\usepackage{tikz}
\usepackage{pgf}
\usepackage{pgfplots}
\usepackage{pgfplotstable}
\usetikzlibrary{spy,calc,dsp,chains,bayesnet}
\usetikzlibrary{matrix,decorations.pathreplacing,calc,fit,backgrounds}
\pgfplotsset{compat=1.17}
\usepgfplotslibrary{external} 

\definecolor{cherry1}{rgb}{0.215686, 0.215686, 0.215686}
\definecolor{cherry2}{rgb}{0.563899, 0.155919, 0.156577}
\definecolor{cherry3}{rgb}{0.747389, 0.178584, 0.180272}
\definecolor{cherry4}{rgb}{0.836168, 0.264453, 0.26819}
\definecolor{cherry5}{rgb}{0.880144, 0.397868, 0.404399}
\definecolor{cherry6}{rgb}{0.911942, 0.567676, 0.576412}

\definecolor{ibmred1}{rgb}{0.643137,0.086274,0.30588}
\definecolor{ibmred2}{rgb}{0.925490,0.325490,0.545098}
\definecolor{ibmred3}{rgb}{0.996078,0.623529,0.764705}
\definecolor{ibmpurple1}{rgb}{0.427451,0.1921569,0.7843137}
\definecolor{ibmpurple2}{rgb}{0.6588235,0.4352941,0.9960784}
\definecolor{ibmpurple3}{rgb}{0.8156863,0.6901961,1.0}
\definecolor{ibmblue1}{rgb}{0.0,0.345098,0.627451}
\definecolor{ibmblue2}{rgb}{0.06666667,0.5764706,0.90588a24}
\definecolor{ibmblue3}{rgb}{0.4235294,0.7921569,0.9960784}

\usepackage[textsize=tiny]{todonotes}

\icmltitlerunning{Provably Scalable BBVI with Structured Variational Families}

\begin{document}

\twocolumn[
\icmltitle{Provably Scalable Black-Box Variational Inference with\\Structured Variational Families}



\icmlsetsymbol{equal}{*}

\begin{icmlauthorlist}


\icmlauthor{Joohwan Ko}{equal,kaist}
\icmlauthor{Kyurae Kim}{equal,penn}
\icmlauthor{Woo Chang Kim}{kaist}
\icmlauthor{Jacob R. Gardner}{penn}
\end{icmlauthorlist}

\icmlaffiliation{penn}{Department of Computer and Information Sciences, University of Pennsylvania, Philadelphia, PA, U.S.A.}
\icmlaffiliation{kaist}{KAIST, Daejeon, South Korea, Republic of}

\icmlcorrespondingauthor{Kyurae Kim}{{\hypersetup{urlcolor=black}{\href{mailto:kyrkim@seas.upenn.edu}{kyrkim@seas.upenn.edu}}}}

\icmlkeywords{Machine Learning, ICML}

\vskip 0.3in
]



\printAffiliationsAndNotice{\icmlEqualContribution} 

\begin{abstract}
  Variational families with full-rank covariance approximations are known not to work well in black-box variational inference (BBVI), both empirically and theoretically. In fact, recent computational complexity results for BBVI have established that full-rank variational families scale poorly with the dimensionality of the problem compared to \textit{e.g.} mean-field families. This is particularly critical to hierarchical Bayesian models with local variables; their dimensionality increases with the size of the datasets. 
  Consequently, one gets an iteration complexity with an explicit \(\mathcal{O}(N^2)\) dependence on the dataset size \(N\).
  In this paper, we explore a theoretical middle ground \textit{between} mean-field variational families and full-rank families: \textit{structured} variational families.
  We rigorously prove that certain scale matrix structures can achieve a better iteration complexity of \(\mathcal{O}\left(N\right)\), implying better scaling with respect to \(N\). 
  We empirically verify our theoretical results on large-scale hierarchical models.
\end{abstract}

\input{sections/section_introduction}
\input{sections/section_preliminaries}
\input{sections/section_main_results}

\input{sections/section_experiments}
\input{sections/section_discussions}

\clearpage

\section*{Acknowledgements}
J. Ko and W. Kim were supported by the National Research Foundation of Korea (NRF) grant funded by the Ministry of
Science and ICT (NRF-2022M3J6A1063021 and RS-2023-
00208980); K. Kim was supported by a gift from AWS AI to Penn Engineering's ASSET Center for Trustworthy AI; J. R. Gardner was supported by NSF award [IIS-2145644].

\section*{Impact Statement}
This paper presents a theoretical analysis of black-box variational inference with the goal of broadening our understanding of the algorithm and potentially contributing to its improvement.
As such, the potential societal consequences of our work are inherited from those of Bayesian inference and general probabilistic modeling.
However, the work itself is theoretical, and we do not expect direct societal consequences.

\bibliographystyle{icml2024}
\bibliography{references}

\newpage
\appendix

\onecolumn

\newpage
{\hypersetup{linkcolor=black}
\tableofcontents
}


\clearpage
\section{Computational Resources}
\input{sections/section_computational_resources}

\section{Proofs}\label{section:proofs}

\input{sections/section_definitions}

\clearpage
\vspace{2ex}
\subsection{Auxiliary Lemmas}\label{section:auxiliary}
\vspace{2ex}
\input{theorems/thm_auxiliary}
\printProofs[auxiliary]
\input{theorems/thm_matrix_trace_holder}
\clearpage
\printProofs[scaleblockderivative]

\clearpage
\subsection{Convergence of Stochastic Proximal Gradient Descent}\label{section:proximal_sgd_proofs}
\vspace{1ex}
\input{sections/section_proxsgd_convergence}

\clearpage
\subsection{Convergence of Black-Box Variational Inference of Finite-Sum Likelihoods}
\input{sections/section_bbvi_complexity}

\clearpage
\subsubsection{Key Lemmas}\label{section:key_lemmas}
\vspace{2ex}
\printProofs[jacobianreparam]
\clearpage
\printProofs[gradnorm]
\clearpage
\printProofs[locationscalereparam]
\input{sections/section_gradient_variance_key_lemma}

\clearpage
\subsubsection{Gradient Variance Bound (\cref{thm:quadratic_gradient_variance})}\label{section:quadratic_variance}
\vspace{1ex}
\printProofs[optimalgradientvariance]

\clearpage
\subsubsection{Convex Expected Smoothness (\cref{thm:expected_smoothness})}\label{section:convex_expected_smoothness}
\vspace{1ex}
\printProofs[convexexpectedsmooth]


\clearpage
\subsubsection{Complexity with General Location-Scale Families (\cref{thm:proxsgd_bbvi_complexity})}\label{section:bbvi_complexity_proof}
\vspace{1ex}
\printProofs[complexitybbvifixed]



\clearpage
\subsection{Properties of the Non-Standardized Parameterization (\cref{thm:noncentered_param_convexity})}
\vspace{1ex}
\printProofs[nonstandardconvexity]

\newpage
\input{sections/section_experimental_setup}

\newpage
\input{sections/section_additional_experiments}

\end{document}

%% file: krkmath.tex
\usepackage{amssymb}

\usepackage{amsmath,amsthm}
\usepackage{mathtools}
\usepackage{iftex}
\usepackage{etoolbox}

\ifxetex
  \usepackage{microtype}
  \usepackage{unicode-math}
  \newcommand{\mathmat}[1]{\mathbfit{#1}}
  \newcommand{\mathvec}[1]{\mathbfit{#1}}
  \newcommand{\mathvecgreek}[1]{\mathbfit{#1}}
  \newcommand{\mathmatgreek}[1]{\mathbfit{#1}}
  
  \newcommand{\boldupright}[1]{\symbfup{#1}}

  \newcommand{\mathrv}[1]{\mathsfit{#1}}
  \newcommand{\mathrvvec}[1]{\mathbfsfit{#1}}
  \newcommand{\mathrvgreek}[1]{\mathsfit{#1}}
  \newcommand{\mathrvvecgreek}[1]{\mathbfsfit{#1}}

  \usepackage{fontspec}
\else
  \usepackage[notext,lcgreekalpha]{stix2}
  \usepackage{textcomp}

  \newcommand{\mathmat}[1]{\mathbfit{#1}}
  \newcommand{\mathvec}[1]{\mathbfit{#1}}
  \newcommand{\mathvecgreek}[1]{\mathbfit{#1}}
  \newcommand{\mathmatgreek}[1]{\mathbfit{#1}}
  
  \newcommand{\boldupright}[1]{\mathbf{#1}}

  \newcommand{\mathrv}[1]{\mathsfit{#1}}
  \newcommand{\mathrvvec}[1]{\mathbfsfit{#1}}
  \newcommand{\mathrvgreek}[1]{\mathsfit{#1}}
  \newcommand{\mathrvvecgreek}[1]{\mathbfsfit{#1}}
\fi

\usepackage{booktabs,threeparttable,arydshln}
\usepackage{multirow}
\usepackage{nicematrix}

\usepackage[algo2e, ruled]{algorithm2e}

\usepackage{pgfplots}
\pgfkeys{/pgfplots/tuftelike/.style={
  semithick,
  tick style={major tick length=4pt,semithick,black},
  separate axis lines,
  axis x line*=bottom,
  axis x line shift=2ex,
  xlabel shift=0em,
  axis y line*=left,
  axis y line shift=2ex,
  ylabel shift=0em}}
  
\usepackage{proof-at-the-end}

\usepackage{mdframed}
\mdfdefinestyle {mdleftbarcoloredstyle} {%
    leftline          = true,%
    rightline         = false,%
    topline           = false,%
    bottomline        = false,%
    linewidth         = 2.5pt,%
    backgroundcolor   = gray!10,%
    linecolor         = gray,%
    innertopmargin    = 0.7ex,%
    innerbottommargin = 0.7ex,%
    innerleftmargin   = 1.ex,%
    ntheorem          = false,%
}

\mdfdefinestyle {coloredstyle} {%
    leftline          = false,%
    rightline         = false,%
    topline           = false,%
    bottomline        = false,%
    backgroundcolor   = gray!10,%
    skipabove         = \parskip,%
    skipbelow         = 0,%
    innertopmargin    = 1.0ex,%
    innerbottommargin = 0.3ex,%
    innerleftmargin   = 1.ex,%
}

\declaretheoremstyle[
    spaceabove    = \parsep,
    spacebelow    = \parsep,
    bodyfont      = \normalfont\itshape,
]{theoremsty}

\declaretheorem[name=Theorem,    style=theoremsty, mdframed={style = coloredstyle}]{theorem}

\declaretheorem[name=Corollary,  style=theoremsty, mdframed={style = coloredstyle}]{corollary}
\declaretheorem[name=Lemma,      style=theoremsty, mdframed={style = coloredstyle}]{lemma}
\declaretheorem[name=Theorem,    style=theoremsty, mdframed={style = coloredstyle}, numbered=no]{theorem*}
\declaretheorem[name=Proposition,style=theoremsty, mdframed={style = coloredstyle}, numbered=no]{proposition*}
\declaretheorem[name=Corollary,  style=theoremsty, mdframed={style = coloredstyle}, numbered=no]{corollary*}
\declaretheorem[name=Lemma,      style=theoremsty, mdframed={style = coloredstyle}, numbered=no]{lemma*}

\declaretheorem[name=Open Problem,style=theoremsty, mdframed={style = coloredstyle}, numbered=no]{openproblem*}
\declaretheorem[name=Conjecture,  style=theoremsty, mdframed={style = coloredstyle}, numbered=no]{conjecture*}

\declaretheoremstyle[
    spaceabove=\parsep,
    spacebelow=\parsep,
    bodyfont=\normalfont,
]{normalsty}
\declaretheorem[name=Remark,      style=normalsty]{remark}
\declaretheorem[name=Definition,  style=normalsty, mdframed={style = 
 coloredstyle}]{definition}
\declaretheorem[name=Assumption,  style=normalsty, mdframed={style = 
 coloredstyle} ]{assumption}

\declaretheorem[name=Remark,      style=normalsty, numbered=no]{remark*}
\declaretheorem[name=Definition,  style=normalsty, mdframed={style = coloredstyle}, numbered=no]{definition*}
\declaretheorem[name=Assumption,  style=normalsty, mdframed={style = coloredstyle}, numbered=no]{assumption*}

\makeatletter
\providecommand\IfFormatAtLeastTF{\@ifl@t@r\fmtversion}
\IfFormatAtLeastTF{2020/10/02}{%
  \usepackage[bookmarksopen=true,bookmarks=true,colorlinks=true,backref=page]{hyperref}
  \usepackage[capitalise, nameinlink]{cleveref}
  \renewcommand*{\backref}[1]{}
  \renewcommand*{\backrefalt}[4]{({\footnotesize%
  \ifcase #1 Not cited.%
    \or page~#2%
    \else pages #2%
  \fi%
  })}

}{%
  \usepackage{hyperref}
  \usepackage[capitalise, nameinlink]{cleveref}%
}
\makeatother

\definecolor{linkcolor}{HTML}{6929C4}
\definecolor{citecolor}{HTML}{0043CE}
\hypersetup{colorlinks={true},linkcolor=linkcolor,citecolor=citecolor}

\crefname{assumption}{Assumption}{Assumption}
\crefname{condition}{Condition}{Condition}

\crefname{framedtheorem}{Theorem}{Theorems}
\crefname{framedproposition}{Proposition}{Propositions}
\crefname{framedlemma}{Lemma}{Lemmas}

\makeatletter
\def\adl@drawiv#1#2#3{%
        \hskip.5\tabcolsep
        \xleaders#3{#2.5\@tempdimb #1{1}#2.5\@tempdimb}%
                #2\z@ plus1fil minus1fil\relax
        \hskip.5\tabcolsep}
\newcommand{\cdashlinelr}[1]{%
  \noalign{\vskip\aboverulesep
           \global\let\@dashdrawstore\adl@draw
           \global\let\adl@draw\adl@drawiv}
  \cdashline{#1}
  \noalign{\global\let\adl@draw\@dashdrawstore
           \vskip\belowrulesep}}
\makeatother

\pgfkeys{/prAtEnd/global custom defaults/.style={
    end,
    restate,
    text proof={Proof},
    text link={See the \hyperref[proof:prAtEnd\pratendcountercurrent]{\textit{full proof}} in page~\pageref{proof:prAtEnd\pratendcountercurrent}.}
  }
}

\DeclareMathOperator*{\minimize}{minimize}

\DeclareMathOperator*{\argmax}{arg\,max}
\DeclareMathOperator*{\argmin}{arg\,min} 

\newcommand*\xbar[1]{%
  \hbox{%
    \vbox{%
      \hrule height 0.6pt 
      \kern0.33ex
      \hbox{%
        \kern-0.1em
        \ensuremath{#1}%
        \kern-0.1em
      }%
    }%
  }%
} 

\newcommand{\E}[1]{\mathbb{E}\left[ #1 \right]}

\newcommand{\V}[1]{\mathbb{V}\left[ #1 \right]}

\newcommand{\norm}[1]{{\left\lVert #1 \right\rVert}}
\newcommand{\abs}[1]{{\left| #1 \right|}}

\def\do#1{\csdef{v#1}{\mathvec{#1}}}
\docsvlist{a,b,c,d,e,f,g,h,i,j,k,l,m,n,o,p,q,r,s,t,u,v,w,x,y,z}
\docsvlist{A,B,C,D,E,F,G,H,I,J,K,L,M,N,O,P,Q,R,S,T,U,V,W,X,Y,Z}

\def\do#1{\csdef{v#1}{\mathvecgreek{\csname#1\endcsname}}}
\docsvlist{alpha,beta,gamma,delta,epsilon,zeta,eta,theta,iota,kappa,lambda,mu,nu,xi,omikron,pi,rho,sigma,tau,upsilon,phi,chi,psi,omega}

\def\do#1{\csdef{rv#1}{\mathrv{#1}}}
\docsvlist{a,b,c,d,e,f,g,h,i,j,k,l,m,n,o,p,q,r,s,t,u,v,w,x,y,z}
\docsvlist{A,B,C,D,E,F,G,H,I,J,K,L,M,N,O,P,Q,R,S,T,U,V,W,X,Y,Z}

\def\do#1{\csdef{rv#1}{\mathrvgreek{\csname#1\endcsname}}}
\docsvlist{alpha,beta,gamma,delta,epsilon,zeta,eta,theta,iota,kappa,lambda,mu,nu,xi,omikron,pi,rho,sigma,tau,upsilon,phi,chi,psi,omega}

\def\do#1{\csdef{rvv#1}{\mathrvvec{#1}}}
\docsvlist{a,b,c,d,e,f,g,h,i,j,k,l,m,n,o,p,q,r,s,t,u,v,w,x,y,z}

\def\do#1{\csdef{rvv#1}{\mathrvvecgreek{\csname#1\endcsname}}}
\docsvlist{alpha,beta,gamma,delta,epsilon,zeta,eta,theta,iota,kappa,lambda,mu,nu,xi,omikron,pi,rho,sigma,tau,upsilon,phi,chi,psi,omega}

\def\do#1{\csdef{rvm#1}{\mathrvvecgreek{#1}}}
\docsvlist{A,B,C,D,E,F,G,H,I,J,K,L,M,N,O,P,Q,R,S,T,U,V,W,X,Y,Z}

\def\do#1{\csdef{m#1}{\mathmat{#1}}}
\docsvlist{A,B,C,D,E,F,G,H,I,J,K,L,M,N,O,P,Q,R,S,T,U,V,W,X,Y,Z}

\def\do#1{\csdef{m#1}{\mathmatgreek{\csname#1\endcsname}}}
\docsvlist{Sigma,Lambda,Phi,Gamma,Theta,Delta}

\newcommand{\inner}[2]{\left\langle #1, #2 \right\rangle}

%% file: sections/section_introduction.tex
\section{Introduction}
A decade has passed since black-box variational inference (BBVI; \citealp{ranganath_black_2014,titsias_doubly_2014}), also known as Monte Carlo variational inference and stochastic gradient variational Bayes, has emerged.
Among various approaches to variational inference (VI; \citealp{jordan_introduction_1999,blei_variational_2017,zhang_advances_2019}), BBVI has been immensely successful in various fields such as statistics, machine learning, signal processing, and many more, thanks to its general black-box nature and scalability to large datasets.

One of the promises of BBVI has been that we can better model correlations in the posterior by using full-rank covariance approximations~\citep{kucukelbir_automatic_2017}.
That way, we should have been able to move away from mean-field~\citep{peterson_mean_1987,hinton_keeping_1993}, or diagonal covariance, approximations traditionally required when performing coordinate-ascent VI (CAVI; see the review by \citealt{blei_variational_2017}).
However, this promise has shown to be elusive.
Despite their theoretical and occasional empirical superiority of ``expressiveness,'' going full-rank does not always improve over mean-field. 
See, for example, the experiments in \S{}3 by \citet{zhang_pathfinder_2022}, \S{}B.1 by \citet{agrawal_advances_2020}, Footnote 2 by \citet{giordano_covariances_2018}.

A common explanation for the underwhelming performance of full-rank approximations has been their excessive gradient variance.
Indeed, recent results by \citet{kim_convergence_2023,domke_provable_2023} establishing the computational complexity of BBVI confirm this intuition.
Due to gradient variance, full-rank covariance approximations result in a \(\mathcal{O}\left(d \kappa^2 \epsilon^{-1}\right)\) iteration complexity for finding an \(\epsilon\)-accurate solution on strongly log-concave posteriors with a condition number of \(\kappa\).
This contrasts with mean-field for which a \(\mathcal{O}\left(\sqrt{d}\right)\) dimensional dependence~\citep{kim_practical_2023} has been established.
The poor \(d\) dependence explains why the full-rank approximation is underwhelming in practice.

Furthermore, for models with \textit{local variables}~\citep{hoffman_stochastic_2015}, the \(\mathcal{O}\left(d\right)\) dimension dependence is especially concerning; for these models, the dimensionality \(d\) scales with the size of the dataset \(N\).
This means that for a model taking \(\Theta\left(N\right)\) datapoint queries to evaluate the joint likelihood, the sample complexity of BBVI will be \(\mathcal{O}\left(N^2 \kappa^2 \epsilon^{-1}\right)\), where \(\kappa\) also increases linearly with \(N\) for Bayesian posteriors as \(\kappa = \mathcal{O}(N)\) due to posterior contraction.
In terms of scalability with respect to \(N\), this \(\mathcal{O}(N^4)\) complexity highlights a clear challenge for BBVI with full-rank variational families.
(Note that, with our alternative proof technique, this can be tightened to an \(\mathcal{O}(N^3)\) explicit dependence.)

While mean-field is scalable, it is a crude approximation as it fails to model correlations in the posterior.
Therefore, a natural question is, ``Can we find a variational family more expressive than mean-field families while maintaining computational tractability?''
In this paper, based on the theoretical framework of~\citet{kim_convergence_2023,domke_provable_2023}, we answer this question by revisiting the idea of \textit{structured scale matrices}.

Structured scale matrices, or more broadly, structured variational families, exploit structural assumptions about the factor graph of the posterior and have been a widely explored concept in VI (~\citealp{saul_exploiting_1995,hoffman_stochastic_2015,tan_gaussian_2018,ranganath_hierarchical_2016,lin_fast_2019}; for a short representative list; see \cref{section:conclusions} for a more extensive list).
However, why and when structured families can outperform full-rank ones has not been rigorously analyzed.
For instance, \citeauthor{tan_gaussian_2018}, who first performed a detailed investigation into structured scales in BBVI, only mentioned that ``unrestricted (full-rank) Gaussian variational approximation~\ldots~can be prohibitively slow for large data since the number of variational parameters scales as the square of the length of \(\vz\).''
However, given the incredible success of stochastic optimization in the over-parameterized regime, simply having fewer parameters cannot fully explain why structured families work well.

Based on the framework of \citet{domke_provable_2023,kim_convergence_2023}, we rigorously prove that, for likelihoods constituting of a finite-sum of \(N\) components, each associated to a single datapoint, structured location-scale families can improve the dimensional dependence of full-rank families.
For Hierarchical Bayesian models with local variables, we provide a scale matrix structure that reduces the order of the dependence on the dataset size \(N\).
For the canonical 1-level hierarchy where the variables can be split into global and local variables, we show that a triangular scale matrix with a bordered block-diagonal structure achieves an iteration complexity of \(\mathcal{O}\left(N\right)\).
This, in turn, corresponds to approximating the posterior with the generative process 
{\setlength{\belowdisplayskip}{1ex} \setlength{\belowdisplayshortskip}{1ex}
\setlength{\abovedisplayskip}{1ex} \setlength{\abovedisplayshortskip}{1ex}
\begin{align*}
  \rvvz   \sim q_{\vlambda}\left(\rvvz\right) \quad\text{and}\quad
  \rvvy_n \sim q_{\vlambda}\left(\rvvy_n \mid \rvvz\right),
\end{align*}
}%
where \(\rvvz\) are the global variables, \(\rvvy_n\) are the local variables belonging to the \(n\)th datapoint, and \(q_{\vlambda}\) is some location-scale distribution.
\citet{agrawal_amortized_2021} called this, where the local variables are assumed to be conditionally independent, a ``hierarchical branched distribution.''

Overall, this work extends the current line of work of proving that the choice of the variational family performs a trade-off between statistical accuracy and computational efficiency~\citep{bhatia_statistical_2022,kim_practical_2023}.

\begin{enumerate}
    \vspace{-2ex}
    \setlength\itemsep{.5ex}
    \item \textbf{General Analysis:} \cref{thm:quadratic_gradient_variance} establishes that, for finite-sum likelihoods, manipulating the structure of the scale matrix in location-scale variational families can improve the dimensional dependence of BBVI.
    \vspace{-1ex}
    \item \textbf{Scalable Solution:} For hierarchical models with local variables, \cref{thm:scalability} establishes that the scale matrix structure previously proposed by~\citet{tan_use_2021,tan_conditionally_2020} achieves an iteration complexity with a better dependence of the dataset size.
    \vspace{-1ex}
    \item \textbf{Analysis of Parameterizations:} Furthermore, among different ways to parameterize structured variational families, \cref{thm:noncentered_param_convexity} shows that ``non-standardized'' parameterizations rule out the convexity of the negative ELBO, unlike the ``standardized'' parameterization of~\citet{tan_use_2021,tan_conditionally_2020}, and is therefore suboptimal.
    
    \vspace{-1ex}
\end{enumerate}

%% file: sections/section_preliminaries.tex
\section{Preliminaries}
\vspace{-1ex}
\paragraph{Notation} 
Random variables are denoted in sans-serif, vectors in bold, and matrices in bold capitals. (\textit{i.e.,} \(\rvx\), \(\rvvx\), and \(\mA\) respectively.) 
Given a vector \(\vx \in \mathbb{R}^d\), we denote its Euclidean norm as \(\norm{\vx}_2 = \sqrt{\inner{\vx}{\vx}} = \sqrt{\vx^{\top}\vx}\). 
For a matrix \(\mA\), we denote its Frobenius norm as \(\norm{\mA}_{\mathrm{F}} = \sqrt{\mathrm{tr}\left(\mA^{\top} \mA\right)}\), where  \({\textstyle\mathrm{tr}\left(\mA\right)=\sum_{i=1}^d A_{ii}}\), and the \(\ell_2\)-operator norm as \(\norm{\mA}_{2,2}\).
Lastly, \(\mathbb{L}_{++}^d\) denotes the set of \(d\)-dimensional lower triangular matrices with strictly positive eigenvalues.

\vspace{-1ex}
\subsection{Variational Inference}
\vspace{-1ex}
Variational inference (VI; \citealp{jordan_introduction_1999,blei_variational_2017,zhang_advances_2019}) is a method for approximating an intractable probability distribution through optimization.
In general, we aim to minimize the Kullback-Leibler (KL; \citealp{kullback_information_1951}) divergence through:
{%
\setlength{\belowdisplayskip}{.5ex} \setlength{\belowdisplayshortskip}{.5ex}
\setlength{\abovedisplayskip}{1.ex} \setlength{\abovedisplayshortskip}{1.ex}
\begin{align*}
    \minimize_{\vlambda \in \Lambda}\; \mathrm{D}_{\mathrm{KL}}\left(q_{\vlambda}, \pi\right),
\end{align*}
}%
\begin{center}
  \vspace{-2ex}
  {\begingroup
  \begin{tabular}{lll}
    where 
    & \(\mathrm{D}_{\mathrm{KL}}\) & is the KL divergence, \\
    & \(\pi\) & is the (target) posterior distribution, and  \\
    & \(q_{\vlambda}\) & is the variational approximation. \\
  \end{tabular}
  \vspace{-2ex}
  \endgroup}
\end{center}

\vspace{-1ex}
\paragraph{Evidence Lower Bound}
In the context of Bayesian inference, we only have access to the joint likelihood \(p\left(\vz, \vx\right) \propto \pi\left(\vz\right)\), where \(\vz\) are the model parameters and \(\vx\) is the data. 
This results in the KL divergence also being intractable due to the intractable normalizer.
Therefore, we instead rely on minimizing a surrogate objective called the negative \textit{evidence lower bound} (ELBO,~\citealp{jordan_introduction_1999}) function \(F(\vlambda)\):
{%
\setlength{\belowdisplayskip}{.5ex} \setlength{\belowdisplayshortskip}{.5ex}
\setlength{\abovedisplayskip}{.5ex} \setlength{\abovedisplayshortskip}{.5ex}
\[
  \minimize_{\vlambda \in \Lambda}\; F\left(\vlambda\right)
  \triangleq
  - \mathbb{E}_{\rvvz \sim q_{\vlambda}} \log p \left(\rvvz, \vx\right) - \mathbb{H}\left(q_{\vlambda}\right),
\]
}%
\begin{center}
  \vspace{-1ex}
  {\begingroup
  \begin{tabular}{lll}
    where 
    & \(p\left(\vz, \vx\right)\) & is the joint likelihood, \\
    & \(\mathbb{H}\)     & is the differential entropy.\\
  \end{tabular}
  \endgroup}
  \vspace{-1ex}
\end{center}
Without loss of generality, we will also separately denote the negative log joint likelihood as 
{%
\setlength{\belowdisplayskip}{1ex} \setlength{\belowdisplayshortskip}{1ex}
\setlength{\abovedisplayskip}{1.5ex} \setlength{\abovedisplayshortskip}{1.5ex}
\[
  \ell\left(\vz\right) \triangleq - \log p\left(\vz, \vx\right).
\]
}%
Under this notation, the ELBO can be represented as a \textit{composite} optimization problem. (See the taxonomization in \S{}2.4 by~\citealt{kim_convergence_2023}.)
{%
\setlength{\belowdisplayskip}{1ex} \setlength{\belowdisplayshortskip}{1ex}
\setlength{\abovedisplayskip}{1.5ex} \setlength{\abovedisplayshortskip}{1.5ex}
\[
  F\left(\vlambda\right)
  =
  f\left(\vlambda\right) + h\left(\vlambda\right),
\]
}%
where 
\(
   f(\vlambda) \triangleq \mathbb{E}_{\rvvz \sim q_{\vlambda}} \ell \left(\rvvz\right),
\)
is called the \textit{energy} and \(h\left(\vlambda\right) \triangleq -\mathbb{H}\left(q_{\vlambda}\right)\) is an entropic regularizer.

\vspace{-1ex}
\paragraph{Finite-Sum Likelihood}
In this work, we will focus on log(-joint) likelihoods of the structure
{
\setlength{\belowdisplayskip}{.5ex} \setlength{\belowdisplayshortskip}{.5ex}
\setlength{\abovedisplayskip}{.5ex} \setlength{\abovedisplayshortskip}{.5ex}
\begin{equation}
  \ell\left(\vz\right) = {\sum_{n=1}^N} \ell_n\left(\vz\right)\label{eq:finitesum}
\end{equation}
}%
for \(n = 1, \ldots, N\), where each component likelihood \(\ell_n\).
It is common for each component to use only subsets of the variables constituting \(\vz\).

\vspace{-1ex}
\subsection{Variational Family}\label{section:family}
\vspace{-1ex}
We focus on the location-scale variational family:
\begin{definition}[\textbf{Location-Scale Family}]\label{def:family}
  Let \(\varphi\) be some \(d\)-variate distribution.
  Then, \(q_{\vlambda}\) that can be equivalently represented as
{%
\setlength{\belowdisplayskip}{-.5ex} \setlength{\belowdisplayshortskip}{-.5ex}
\setlength{\abovedisplayskip}{-.5ex} \setlength{\abovedisplayshortskip}{-.5ex}
  \begin{alignat*}{2}
    \rvvz \sim q_{\vlambda}  \quad\Leftrightarrow\quad &\rvvz \stackrel{d}{=} \mathcal{T}_{\vlambda}\left(\rvvu\right); \quad \rvvu \sim  \varphi,
  \end{alignat*}
  }%
  where \(\stackrel{d}{=}\) is equivalence in distribution, is said to be part of the location-scale family generated by the base distribution \(\varphi\) and the reparameterization function \(\mathcal{T} : \Lambda \times \mathbb{R}^d \rightarrow \mathbb{R}^{d}\) defined as
{
\setlength{\belowdisplayskip}{.5ex} \setlength{\belowdisplayshortskip}{.5ex}
\setlength{\abovedisplayskip}{1.ex} \setlength{\abovedisplayshortskip}{1.ex}
  \begin{align*}
    &\mathcal{T}_{\vlambda}\left(\vu\right) \triangleq \mC \vu + \vm
  \end{align*}
}%
  with \(\vlambda \in \Lambda \subseteq \mathbb{R}^p\) containing the parameters for forming the location \(\vm \in \mathbb{R}^{d}\) and scale \(\mC \in \mathbb{R}^{d \times {d}}\).
\end{definition}
\vspace{-2ex}
Since the reparameterization function is differentiable, this enables the use of the \(M\)-sample \textit{reparameterization gradient}~\citep{kingma_autoencoding_2014,rezende_stochastic_2014,titsias_doubly_2014}, which is an unbiased estimator of the gradient of the energy \(\nabla f \left(\vlambda\right)\), defined as:
{
\setlength{\belowdisplayskip}{1ex} \setlength{\belowdisplayshortskip}{1ex}
\setlength{\abovedisplayskip}{1ex} \setlength{\abovedisplayshortskip}{1ex}
  \begin{alignat}{2}
    \rvvg_M\left(\vlambda\right) 
    \triangleq
    \frac{1}{M} \sum^{M}_{m=1} 
    \nabla_{\vlambda} \ell(\mathcal{T}_{\vlambda}(\rvvu_m)),
    \quad
    \rvvu_1, \ldots, \rvvu_m \stackrel{\textit{i.i.d.}}{\sim} \varphi.
    \label{eq:M_sample_grad_estimator}
  \end{alignat}
}%
The reparameterization gradient is empirically known to result in lower variance compared to other gradient estimators \citep{mohamed_monte_2020, xu_variance_2019} and is presently \textit{de-facto} standard for BBVI.

For the base distribution \(\varphi\), the choice of distribution generates various other families~\citep{titsias_doubly_2014} used in practice: mean-field Gaussian, full-rank Gaussian, Laplace, Student-\(t\), and many more.
For our analysis, we impose mild assumptions on the base distribution \(\varphi\):
\begin{assumption}[\textbf{Base Distribution}]\label{assumption:symmetric_standard}
  \(\varphi\) is a \(d\)-dimensional distribution such that \(\rvvu \sim \varphi\) and \(\rvvu = \left(\rvu_1, \ldots, \rvu_d \right)\) with indepedently and identically distributed components.
  Furthermore, \(\varphi\) is
  \begin{enumerate*}[label=\textbf{(\roman*)}]
      \setlength\itemsep{-1.ex}
      \item symmetric and standardized such that \(\mathbb{E}\rvu_i = 0\), \(\mathbb{E}\rvu_i^2 = 1\), \(\mathbb{E}\rvu_i^3 = 0\), and 
      \item has finite kurtosis \(\mathbb{E}\rvu_i^4 = k_{\varphi} < \infty\).
  \end{enumerate*}
\end{assumption}
\vspace{-1ex}
These assumptions are satisfied by most practically used location-scale families.
In summary:
\begin{assumption}\label{assumption:variation_family}
  The variational family is the location-scale family formed by \cref{def:family} with the base distribution \(\varphi\) satisfying \cref{assumption:symmetric_standard}.
\end{assumption}

\vspace{-3ex}
\subsection{Scale Parameterization}\label{section:parameterization}
\vspace{-1ex}
The parameterization of the scale matrix in location-scale families \(\mC\) can result in vastly different statistical and computational performance results.
In principle, the only restriction is that it needs to result in a proper covariance matrix such that \(\mC \mC^{\top} \succ 0\).
However, how many low-rank factors we include into \(\mC\) corresponds to a statistical-computational trade-off~\citep{bhatia_statistical_2022,ong_gaussian_2018}.
A common choice is to restrict \(\mC\) to be a diagonal matrix, resulting in the \textit{mean-field} approximation~\citep{peterson_mean_1987,hinton_keeping_1993}.

\vspace{-2ex}
\paragraph{Triangular Scale}
While mean-field results in fast convergence~\citep{kim_practical_2023} and stable optimization, it is a crude approximation as it ignores correlations in the posterior.
Therefore, we are interested in scale matrices that are more complex than mean-field.
For this, we will first restrict our interest to triangular scale matrices with strictly positive eigenvalues such that \(\mC \in \mathbb{L}_{++}^d\).
Compared to other choices, such as the ``matrix square-root'' analyzed by~\citet{domke_provable_2023}, this ``Cholesky'' parameterization has the following benefits:
\begin{enumerate}[label={(\alph*)}]
    \vspace{-2ex}
    \setlength\itemsep{-1.ex}
    \item It results in lower gradient variance~\citep{kim_practical_2023}, 
    \item the entropy term \(h\) can be computed in \(\Theta\left(d\right)\) time, 
    \item the positive definiteness of \(\mC\) can be enforced only by manipulating the \(d\) diagonal elements, and\label{item:scale_diagonal}
    \item the conditional dependence between the coordinates can easily be manipulated.\label{item:coordinate_dependence}
    \vspace{-2ex}
\end{enumerate}
\vspace{-1ex}
Naturally, \cref{item:coordinate_dependence} is essential in the context of structured variational families.


\vspace{-1ex}
\subsection{Stochastic Proximal Gradient Descent}\label{section:proximal_sgd}
\vspace{-1ex}
While \(F\left(\vlambda\right)\) is commonly optimized using stochastic gradient descent (SGD; \citealp{robbins_stochastic_1951,bottou_online_1999,nemirovski_robust_2009}). 
In this work, we will focus on a ``proximal'' variant of SGD, which has favorable theoretical properties~\citep{domke_provable_2023,domke_provable_2020} compared to projected SGD, which is more commonly used in practice.

\vspace{-2ex}
\paragraph{Proximal SGD}
Proximal SGD, or stochastic proximal gradient descent, aims to minimize a composite objective expressed as a sum \(f + h\), where \(f\) is smooth and convex, while \(h\) might convex as well but non-smooth.
Given initial parameters \(\vlambda_0\) and a step size schedule \((\gamma_t)_{t=0}^{T-1}\), proximal SGD repeats the update:
{\setlength{\belowdisplayskip}{1ex} \setlength{\belowdisplayshortskip}{1ex}
\setlength{\abovedisplayskip}{1ex} \setlength{\abovedisplayshortskip}{1ex}
\[
    \vlambda_{t+1} = \textrm{prox}_{\gamma_t, h} \left(\vlambda_t - \gamma_t \, \rvvg_{M}\left(\vlambda_t\right) \right)
\]
}%
until convergence, where \(\rvvg_M\left(\vlambda_t\right)\) is a stochastic estimate of \(\nabla f\) not \(\nabla F\), and \(\mathrm{prox}\) is a \textit{proximal operator} defined as:
{\setlength{\belowdisplayskip}{1ex} \setlength{\belowdisplayshortskip}{1ex}
\setlength{\abovedisplayskip}{1ex} \setlength{\abovedisplayshortskip}{1ex}
\[
    \textrm{prox}_{\gamma_t, h} \triangleq \argmin_{\vv} \left[ h(\vv) + \frac{1}{2\gamma_t} \lVert \vlambda - \vv \rVert^2_2 \right].
\]
}%
The precise proximal operator we use is that by \citet{domke_provable_2020} and discussed in \cref{section:proxsgd_bbvi_details}.

\vspace{-1ex}
\paragraph{Proximal SGD in BBVI}
Proximal SGD has been proposed for variational inference under different motivations~\citep{altosaar_proximity_2018,khan_kullbackleibler_2015,khan_faster_2016,diao_forwardbackward_2023}.
However, \citet{domke_provable_2020} first suggested proximal SGD to circumvent the fact that the ELBO is non-smooth due to the entropic regularizer \(h\) being non-smooth.
While the non-smoothness of the ELBO can be solved through alternative means--\textit{e.g.}, projected SGD, nonlinear parameterizations--projected SGD necessitate the choice of a restricted domain, and nonlinear parameterizations result in slower convergence rate~\citep{kim_convergence_2023}.
Furthermore, projected and proximal SGD perform very similarly both in theory~\citep{domke_provable_2023} and in practice.
As such, while we focus on proximal SGD, most guarantees will also hold with projected SGD with the projection operator considered by~\citet{kim_linear_2024}. 


%% file: sections/section_main_results.tex
\vspace{-1ex}
\section{Theoretical Analysis}\label{section:theory}
\vspace{-1ex}
\subsection{Fundamental Limits of Being Full-Rank}\label{section:limitation_fullrank}
\vspace{-1ex}
First, we will discuss what we currently know about the dimensional dependence of BBVI.
We will also illustrate exactly \textit{where} the dimensional dependence is coming from.
This should make the solution more clear. 


\vspace{-1ex}
\paragraph{Iteration Complexity of Full-Rank Families}
\citet{domke_provable_2023,kim_convergence_2023} show that, with stochastic proximal gradient descent, full-rank variational families, and the 1-sample reparameterization gradient estimator, the iteration complexity is as follows:
\begin{theorem}[\citealp{domke_provable_2023,kim_convergence_2023}]\label{eq:domke_complexity_bound}
Let \(\ell\) be \(\mu\)-strongly convex and \(L\)-smooth.
Then, the iteration complexity of being \(\epsilon\)-close to the global minimizer with proximal SGD BBVI is
{%
\setlength{\belowdisplayskip}{1ex} \setlength{\belowdisplayshortskip}{1ex}
\setlength{\abovedisplayskip}{1ex} \setlength{\abovedisplayshortskip}{1ex}
\[
  \mathcal{O}\left(d \, \kappa^2 \, \frac{1}{\epsilon} \log \left( \Delta^2_0 \frac{1}{\epsilon} \right) \right),
\]
}%
where \(\kappa = L / \mu\), \(\Delta_0 = \norm{\vlambda_0 - \vlambda^*}_2\) is the distance between the initial point \(\vlambda_0\) and the global optimum \(\vlambda^* = \argmin_{\vlambda \in \Lambda} F\left(\vlambda\right) \).
\end{theorem}

\vspace{-1ex}
\paragraph{Local Variables}
In probabilistic modeling, some Hierarchical models have latent variables, \(\vy_n\) for \(n=1, \ldots, N\), unique to each datapoint \(\vx_n\)~\citep{hoffman_stochastic_2015}.
For example, in mixture and mixed-membership models, each datapoint has a latent variable indicating from which mixture component it was generated.
In state space models, at each time point \(t\), the state is often not observed directly and is a latent variable local to \(t\).

More abstractly, consider a model with local variables \(\vy_1, \ldots, \vy_N\) and a global variable \(\vz\) such that \(\vy_n \in \mathbb{R}^{d_y}\) and \(\vz \in \mathbb{R}^{d_z}\).
The total dimension of the model is
{%
\setlength{\belowdisplayskip}{1ex} \setlength{\belowdisplayshortskip}{1ex}
\setlength{\abovedisplayskip}{1ex} \setlength{\abovedisplayshortskip}{1ex}
\[
  d = N d_y + d_z.
\]
}%
Then, for these models, the complexity ends up with a dependence on the size of the dataset.
\begin{corollary*}[Informal]
  For a Hierarchical Bayesian model with \(\mathcal{O}\left(N\right)\) variables, a \(\mu\)-strongly log-concave posterior, and \(L_n\)-smooth component likelihoods \(\ell_n\) for \(n= 1, \ldots, N\) the iteration complexity is
{%
\setlength{\belowdisplayskip}{1ex} \setlength{\belowdisplayshortskip}{1ex}
\setlength{\abovedisplayskip}{1ex} \setlength{\abovedisplayshortskip}{1ex}
  \[
    \mathcal{O}\,\left( N^3 \, \frac{L_{\mathrm{max}}^2}{\mu^2} \frac{1}{\epsilon} \log \left( \Delta_0^2 \frac{1}{\epsilon} \right) \right),
  \]
  }%
  where \(L_{\mathrm{max}} = \max\left\{L_1, \ldots, L_n\right\}\).
\end{corollary*}
For Bayesian models, the smoothness of the full posterior is \(N L_{\mathrm{max}}\) since \cref{eq:finitesum} is a sum, not an average as in empirical risk minimization.
Therefore, the \(N^3\) factor in the corollary follows from plugging in \(\kappa = N L_{\rm{max}}/\mu\) and \(d = N d_y + d_z\) in \cref{eq:domke_complexity_bound}.

\vspace{1ex}
\begin{remark}[\textbf{Sample Complexity}]\label{remark:naive}
  While the cost of evaluating the gradient of the joint likelihood is at least \(\Omega\left(N\right)\) datapoint queries, we will assume it is \(\Theta(n)\) throughout the paper, as it is the most common case.
  Then, the sample complexity of BBVI with proximal SGD scales as \(\mathcal{O}\left(N^4\right)\).
\end{remark}

Clearly, full-rank families do not scale even if we ignore the \(\Theta\left(N^2\right)\) storage requirement of the scale matrix.

\vspace{1ex}
\begin{remark}[\textbf{Are fewer parameters \textit{obviously} better?}]\label{remark:obvious}
    The number of variational parameters \(p\) enters the complexity statement through \(\Delta_0\).
    Therefore, the dependence on the number of parameters alone is only logarithmic.
    This implies that simply reducing the number of parameters does not obviously improve the dimensional dependence.
\end{remark}
\vspace{1ex}
\begin{remark}[\textbf{Where does \(d\) come from?}]
  The \(\mathcal{O}\left(d\right)\) dimension dependence directly comes from gradient variance. (See Theorem 3 and 7 of \citet{domke_provable_2019}.)
  Therefore, reducing the gradient variance is key to improving the complexity.
\end{remark}

\subsection{Complexity of BBVI on Finite-Sum Likelihoods}\label{section:bbvi_complexity}

Before presenting our result, we generalize the notation we have introduced so far.
Since we consider hierarchical models with local variables, each component \(\ell_n\) will use a subset of the variables returned by \(\mathcal{T}_{\vlambda}\).
We express this through \(\mathcal{T}_{\vlambda}^n : \mathbb{R}^{d} \to \mathbb{R}^{d_n}\), the parameterization function specific to \(\ell_n\), such that
{%
\setlength{\belowdisplayskip}{1ex} \setlength{\belowdisplayshortskip}{1ex}
\setlength{\abovedisplayskip}{1ex} \setlength{\abovedisplayshortskip}{1ex}
\[
  \mathcal{T}_{\vlambda}^n\left(\vu\right)
  \triangleq
  \mC_n \vu + \vm_n,
\]
}%
where \(\mC_n \in \mathbb{R}^{d \times d_{n}}\) is a subset of rows \(\mC\) and \(\vm_n \in \mathbb{R}^{d_n}\) is a subset of components of \(\vm\) corresponding to the coordinates of \(\vz = \mathcal{T}_{\vlambda}\left(\vu\right)\) used by \(\ell_n\).

Furthermore, we are interested in the \textit{structure} of \(\mC_n\) (and correspondingly \(\mC\)).
That is, whether the columns of \(\mC_{n}\) are zero or non-zero.
For this, we introduce the indicator 
{%
\setlength{\belowdisplayskip}{1ex} \setlength{\belowdisplayshortskip}{1ex}
\setlength{\abovedisplayskip}{1ex} \setlength{\abovedisplayshortskip}{1ex}
\[
   \delta_{n,j} \triangleq \begin{cases}
        \; 1 &  \text{if the \(j\)th column  of \(\mC_n\) is non-zero} \\
        \; 0 &  \text{otherwise}.
   \end{cases}
\]
}%
Notice that \(\delta_{n,j}\) is also implicitly encoding the structure of \(\ell_n\): if \(\ell_n\) uses a certain coordinate of \(\vz\), say \(z_k = [\vz]_k\), then \(\delta_{n,k}\) \textit{has} to be non-zero, since the original matrix \(\mC\) has a non-zero diagonal to qualify as a Cholesky factor.
Therefore, for the case of a two-level hierarchical model, \(\sum_{j} \delta_{n,j}\) is \textit{at least} \(d_y + d_{z}\), where ``turning-on'' additional rows allows for representing of additional correlations.

\vspace{-2ex}
\paragraph{Upper Bound on the Gradient Variance}
We will see how this block structure affects the gradient variance.


\input{theorems/thm_jacobian_reparam}

\input{theorems/thm_gradnorm}

\input{theorems/thm_scale_block_derivative}

\input{theorems/thm_location_scale_reparam}

\input{theorems/thm_quadratic_gradient_variance}

\input{theorems/thm_expected_smoothness}

\input{theorems/thm_variance_transfer}

\vspace{.5ex}
\begin{remark}
  When \(\mC\) is dense, such as in full-rank variational families, we have \(d^* = d\).
  Therefore, we exactly retrieve Theorem 6 of \citet{domke_provable_2019}.
\end{remark}

\vspace{.5ex}
\begin{remark}\label{remark:dimension_dependence_structure}
   \(\delta_{n,j}\) is related to dimension dependence for the following reason: by setting the \(j\)th column of \(\mC_{n}\) to be \textit{non-zero} (\(\delta_{n,j} = 1\)), we are effectively deciding to use \(\rvu_{j}\) (the \(j\)th component of the \(d\)-dimensional vector \(\rvvu\)) when computing \(\mathcal{T}^{n}_{\vlambda}\).
  The exact number of components of \(\rvvu\) mixed by \(\mC_{n}\) is the ``effective'' dimension dependence \(d^*\).
\end{remark}

\vspace{.5ex}
\begin{remark}
Since the gradient variance dominates the computational complexity, \cref{thm:quadratic_gradient_variance} answers \cref{remark:obvious}.
That is, the \textit{structure} of \(\mC_n\), which depends on \(\ell_n\), affects the complexity, not the number of parameters.
\end{remark}

Notice that \cref{thm:quadratic_gradient_variance} is pointing towards a trivial way to improve the dimensional dependence of mean-field:
\begin{corollary}\label{corollary:meanfield}
  Let the posterior \(\pi\) factorize into independent univariate sub-posteriors such that \(\pi(\vz) = \prod^{N}_{n}\pi_n\left(z_n\right)\), where each \(\pi_n\) is \(L_n\)-log-smooth for \(n = 1, \ldots, N\).
  Then, the gradient variance of the mean-field approximation is dimension-independent.
\end{corollary}

\vspace{.5ex}
\begin{remark}
While \cref{corollary:meanfield} partially answers Conjecture 1 of \citet{kim_convergence_2023}, the general case for jointly-\(L\)-log-smooth posteriors remains open.
\end{remark}

\vspace{-1ex}
\paragraph{Improving the Dimension Dependence}
Admittedly, \cref{corollary:meanfield} is not very interesting since posteriors do not factorize as such for interesting problems.
However, \cref{thm:quadratic_gradient_variance} does suggest that we can shape the resulting dimension dependence of realistic problems by designing the structure of \(\mC_n\).
We will later show a specific structure for models with local variables that can improve the dependence on the number of datapoints.

\vspace{-1ex}
\paragraph{Complexity of BBVI}
Now, based on \cref{thm:quadratic_gradient_variance}, we can prove an iteration complexity bound on BBVI with proximal SGD (\cref{section:proximal_sgd}) and the reparameterization gradient (\cref{section:family}).
The proof is based on the general results on proximal SGD by \citet{gorbunov_unified_2020}, recently refined by \citet{garrigos_handbook_2023}.

\input{theorems/thm_proxsgd_bbvi_complexity}


\vspace{.5ex}
\begin{remark}
   Since each evaluation of \(\rvvg_M\) takes \(\Theta\left(N M\right)\) time, \cref{thm:proxsgd_bbvi_complexity} implies a sample complexity of
   {
   \setlength{\abovedisplayskip}{1ex} \setlength{\abovedisplayshortskip}{1ex}
   \setlength{\belowdisplayskip}{1ex} \setlength{\belowdisplayshortskip}{1ex}
   \[
     \mathcal{O}\left(N^2 \, d^* {\textstyle\sum_{n=1}^N \kappa_n^2} \frac{1}{\epsilon} \log \frac{1}{\epsilon}\right),
   \]
   }%
   or equivalently, a complexity of \(\mathcal{O}\left(d^* N^3\right)\) after taking \(\sum_{n} \kappa_n^2 \leq N \max_{n} \kappa_{n}^2\).
\end{remark}

\vspace{0.5ex}
\begin{remark}
  It also possible to prove a \(\mathcal{O}\left(\nicefrac{1}{\epsilon}\right) \) complexity using the decreasing stepsize schedules of \citet{gower_sgd_2019,stich_unified_2019}.
\end{remark}

\vspace{-.5ex}
\subsection{Hierarchical Branched Structured Families}\label{section:structured}
\vspace{-.5ex}
Equipped with the results of \cref{section:bbvi_complexity}, we present a scale matrix structure that is more scalable in terms of dataset size dependence.
Consider a canonical 2-level Hierarchical model with local variables \(\vy_1, \ldots, \vy_N\) and a global variable \(\vz\) such that \(\vy_n \in \mathbb{R}^{d_y}\) and \(\vz \in \mathbb{R}^{d_z}\).
(All local variables are assumed to have the same dimensionality.)
Then, 
{%
\setlength{\abovedisplayskip}{1ex} \setlength{\abovedisplayshortskip}{1ex}
\setlength{\belowdisplayskip}{1ex} \setlength{\belowdisplayshortskip}{1ex}
\[
  \ell_n\left(\vy_n, \vz\right) = -\log p\left(\vx_n, \vy_n \mid \vz\right) - \frac{1}{N}\log p\left(\vz\right),
\]
}%
where \(\vx_n\) is the \(n\)th datapoint and \(d_n = d_{\vz} + d_{\vy}\).

\vspace{-1ex}
\paragraph{Structured Covariance Matrices}
We assume the variables are structured as 
{%
\setlength{\belowdisplayskip}{1ex} \setlength{\belowdisplayshortskip}{1ex}
\setlength{\abovedisplayskip}{1ex} \setlength{\abovedisplayshortskip}{1ex}
\begin{equation}
  {\begin{bmatrix}
    \vz^{\top} & \vy_1^{\top} & \ldots & \vy_n^{\top} & \ldots & \vy_N^{\top}
  \end{bmatrix}}^{\top},
  \label{eq:variable_structure}
\end{equation}
}%
with the noise vector \(\rvvu\) correspondingly structured as
{%
\setlength{\belowdisplayskip}{0ex} \setlength{\belowdisplayshortskip}{0ex}
\setlength{\abovedisplayskip}{1ex} \setlength{\abovedisplayshortskip}{1ex}
\begin{equation}
  {\begin{bmatrix}
    \rvvu_{\vz}^{\top} & \rvvu_{\vy_1}^{\top} & \ldots & \rvvu_{\vy_n}^{\top} & \ldots & \rvvu_{\vy_N}^{\top}
  \end{bmatrix}}^{\top}.
  \label{eq:noise_structure}
\end{equation}
}%

\begin{figure}[t]
  \centering
  \scalebox{0.85}{
  \input{figures/tikz_structured_scale_matrix}
  }
  \vspace{-1ex}
  \caption{\textbf{Visualization of \(\mC\) under the proposed structure.} 
    The colored entries are non-zero, while the white entries are filled with zeros.
  }
  \label{fig:structure}
  \vspace{-2ex}
\end{figure}
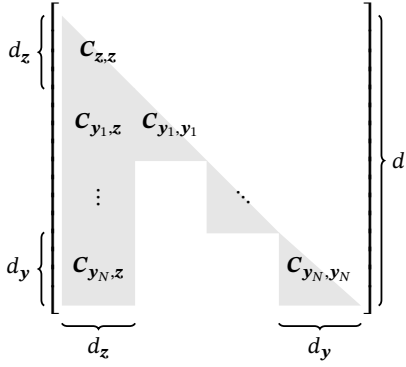

Now for the \(\mC_{n} \in \mathbb{R}^{(d_{\vz} + d_{\vy}) \times d}\), we propose the following structure:
{%
\setlength{\belowdisplayskip}{1ex} \setlength{\belowdisplayshortskip}{1ex}
\setlength{\abovedisplayskip}{1ex} \setlength{\abovedisplayshortskip}{1ex}
\[
  \mC_{n}
  =
  \begin{bmatrix}
    \mC_{\vz,\vz}  & \boldupright{0} & \ldots & \boldupright{0} & \boldupright{0} & \boldupright{0} & \ldots & \boldupright{0} \\
    \mC_{\vy_n,\vz} & \boldupright{0} & \ldots & \boldupright{0} & \mC_{\vy_n,\vy_n} & \boldupright{0} & \ldots & \boldupright{0}
  \end{bmatrix},
\]
}%
\begin{center}
  \vspace{-1ex}
  {\begingroup
  \begin{tabular}{lll}
    where 
    & \(\boldupright{0}\) & is a zero block, \\
    & \(\mC_{\vz,\vz}\) & transforms \(\rvvu_{\vz}\) into \(\rvvz\), \\
    & \(\mC_{\vy_n,\vy_n}\) & transforms \(\rvvu_{\vy_n}\) into \(\rvvy_n\), and \\
    & \(\mC_{\vy_n, \vz}\) & correlates \(\rvvu_{\vz}\) with \(\rvvy_n\). \\
  \end{tabular}
  \endgroup}
  \vspace{-1ex}
\end{center}
We are essentially assuming \(\rvvy_n\) and \(\rvvz\) are correlated, but \(\rvvy_1, \ldots, \rvvy_N\) are conditionally independent such that \(\left(\rvvy_n \mid \rvvz\right) \perp\!\!\!\perp \left(\rvvy_m \mid \rvvz\right)\) for \(n \neq m\).
\citet{agrawal_amortized_2021} called this a ``hierarchical branched distribution'' approximation.
Also, combined with \cref{eq:variable_structure,eq:noise_structure}, the full matrix \(\mC\) exhibits a \textit{bordered block-diagonal structure} as visualized in \cref{fig:structure}.

\vspace{1ex}
\begin{remark}
    The proposed structure has a space/storage complexity of \(\Theta\left( \left( d_{\vz} d_{\vy} + d_{\vy}^2 \right) N\right)\).
    This improves over full-rank, which has a storage complexity of {\( \Theta\left({\left( d_{\vy} + d_{\vz} N \right)}^2\right) \)}.
\end{remark}


\vspace{-2ex}
\paragraph{Structured Variational Family Perspective}
By the property of triangular matrices, this implicitly forms a \textit{structured variational family} of the form of
{%
\setlength{\belowdisplayskip}{.5ex} \setlength{\belowdisplayshortskip}{.5ex}
\setlength{\abovedisplayskip}{.5ex} \setlength{\abovedisplayshortskip}{.5ex}
\[
  q\left(\vz, \vy_1, \ldots, \vy_N \right)
  =
  q\left(\vz\right) {\textstyle \prod^{N}_{n=1}} q\left(\vy_n \mid \vz \right).
\]
}%
Furthermore, when \(\varphi\) is chosen to be Gaussian, \(q\) becomes
{%
\setlength{\belowdisplayskip}{1ex} \setlength{\belowdisplayshortskip}{1ex}
\setlength{\abovedisplayskip}{1ex} \setlength{\abovedisplayshortskip}{1ex}
\begin{align}
  &q\left(\vz, \vy_1, \ldots, \vy_N \right)
  =
  \mathcal{N}\left(\vz; \vm_{\vz},  \mC_{\vz,\vz} \mC_{\vz,\vz}^{\top} \right) 
  \nonumber
  \\
  &\;\;\times
  {\textstyle \prod^{N}_{n=1}}
  \mathcal{N}\left(
  \vy_n;\;  \vm_{\vy_n} + \mC_{\vy_n,\vz}  \mC_{\vz,\vz}^{-1}  \left(\vz - \vm_{\vz} \right), \;
  \mC_{\vy_n,\vy_n} \mC_{\vy_n,\vy_n}^{\top}  
  \right)
  \label{eq:structure}
\end{align}
}%
Note that we do not actually compute \(\mC^{-1}_{\vz,\vz}\), and one can avoid it by directly accessing \(\rvvu_{\vz}\).
Also, \cref{eq:structure} is identical to the structure proposed by \citet[\S 3]{tan_use_2021} and \citet{tan_conditionally_2020}.
However, they did not consider the scalability aspect of this parameterization.
\vspace{-2ex}
\paragraph{Standardized Parameterization}
Notice that, in \cref{eq:structure}, \(\rvvz\) is first ``standardized'' before interacting with \(\rvvy_n\).
(Again, this happens implicitly as we directly use \(\rvvu_{\vz}\) instead of explicitly standardizing \(\rvvz\).)
Under this parameterization, the existence of the ``full'' scale matrix \(\mC\) guarantees the ELBO to be ``regular'' according to \citet{domke_provable_2020,titsias_doubly_2014,challis_gaussian_2013}:

\begin{theorem}[Theorem 9; \citealp{domke_provable_2020}]
  Let \cref{assumption:variation_family} hold and the standardized parameterization be used.
  Then, if the posterior \(\pi\) is log-concave, the negative ELBO is convex.
  If the posterior is also \(\mu\)-strongly log-concave, the negative ELBO is then \(\mu\)-strongly convex.
\end{theorem}

\vspace{-3ex}
\paragraph{Non-Standardized Parameterization}
On the other hand, it is also possible to \textit{not} standardize \(\vz\).
This parameterization was considered by \citet[Table 1]{agrawal_amortized_2021}:
{%
\setlength{\belowdisplayskip}{1ex} \setlength{\belowdisplayshortskip}{1ex}
\setlength{\abovedisplayskip}{1ex} \setlength{\abovedisplayshortskip}{1ex}
\begin{align}
  &q\left(\vz, \vy_1, \ldots, \vy_N \right)
  =
  \mathcal{N}\left(\vz;\; \vm_{\vz}, \; \mC_{\vz,\vz} \mC_{\vz,\vz}^{\top} \right)
  \nonumber
  \\
  &\quad\times
  {\textstyle \prod^{N}_{n=1}}
  \mathcal{N}\left(\vy_n;\;  \vm_{\vy_n} + \mA_{n} \, \vz,\; 
 \mC_{\vy_n,\vy_n} \mC_{\vy_n,\vy_n}^{\top}  \right),
 \label{eq:conventional_param}
\end{align}
}%
where \(\mA_n\) is also an optimized parameter.
The ``expressiveness'' of this parameterization is equivalent to \cref{eq:structure}.
However, the loss-landscape is vastly different:

\input{theorems/thm_noncentered_convexity}

\begin{remark}
    Even if the target posterior is strongly log-concave, the negative ELBO will fail to be strongly convex under the non-standardized parameterization.
    Therefore, convergence will be slower.
\end{remark}
\vspace{-2ex}
Overall, we have the following results:
\begin{lemma}\label{thm:structured_dimensionality}
  The effective dimensionality \(d^*\) of the bordered block-diagonal scale matrix structure is
{%
\setlength{\belowdisplayskip}{1ex} \setlength{\belowdisplayshortskip}{1ex}
\setlength{\abovedisplayskip}{1ex} \setlength{\abovedisplayshortskip}{1ex}
  \[
     d^* = d_{\vz} + d_{\vy}.
  \]
}
\end{lemma}
\vspace{-2ex}
This leads to our key result:


\input{theorems/thm_structured_family_complexity}

\begin{remark}
  Using the bordered block-diagonal scale matrix structure the sample complexity of BBVI is 
  {%
  \setlength{\belowdisplayskip}{1ex} \setlength{\belowdisplayshortskip}{1ex}
  \setlength{\abovedisplayskip}{1ex} \setlength{\abovedisplayshortskip}{1ex}
  \[
    \mathcal{O}\left(N^2 {\textstyle\sum^{N}_{n=1} \kappa_n^2} \frac{1}{\epsilon}\log \left(\frac{1}{\epsilon}\right)\right).
  \]
  }%
  and equivalently \(\mathcal{O}\left(N^3\right)\) after taking \(\sum^{N}_{n=1} \kappa_n^2 \leq N \max_{n} \kappa_{n}^2\).
  This is a factor of \(N\) improvement from full-rank families (\cref{remark:naive}).
\end{remark}

%% file: theorems/thm_jacobian_reparam.tex
\begin{theoremEnd}[all end, category=jacobianreparam]{lemma}
\label{thm:jacobian_reparam}
    The squared Jacobian of the reparameterization function \(\mathcal{T}^{n}_{\vlambda}\) for \(\ell_n\), \(\mJ_n\), is a diagonal matrix given as
    {%
\setlength{\belowdisplayskip}{1ex} \setlength{\belowdisplayshortskip}{1ex}
\setlength{\abovedisplayskip}{1ex} \setlength{\abovedisplayshortskip}{1ex}
    \begin{align*}
        {\mJ_n\left(\vu\right)}^{\top}
        \mJ_{n}\left(\vu\right)
        \;=\;
        \left(1
        +
        {\textstyle\sum_j
        \delta_{n,j}} \rvu_j^2 \right) \boldupright{I}_n
    \end{align*}
    }%
    where \(\boldupright{I}_n\) is the identity matrix of the subspace of \(\ell_n\) (\cref{def:n_identity}) and \(\rvu_j\) is \(j\)th element of \(\rvvu\).
\end{theoremEnd}
\begin{proofEnd}
    Given the location-scale reparameterization function, \(\mathcal{T}_{\vlambda}\), we define the function for \(n\)th datapoint as
    \begin{align*}
        \mathcal{T}_{\vlambda}^{n} \left(\vu\right)
        &=
        \mC_n \vu + \vm_n,
    \end{align*}
    where \(\mC_n\) and \(\vm_n\) has non-zero elements that are only related to the \(n\)th datapoint and makes all others zero.

    For the derivatives, \citet[Lemma 8]{domke_provable_2019} show that 
    \[
      \frac{\partial \mathcal{T}_{\vlambda}\left(\vu\right)}{\partial \vm_i}
      =
      \boldupright{e}_i
      \quad\text{and}\quad
      \frac{\partial \mathcal{T}_{\vlambda}\left(\vu\right)}{\partial \mC_{i,j}}
      =
      \boldupright{e}_i \rvu_j,
    \]
    where \(\boldupright{e}_i\) is the unit bases for the \(i\)th coordinate.
    Here, we are interested in the effect of the structure of \(\mC_n\).
        
    Using the indicator \(\delta\), we have
    \begin{equation}
      \frac{\partial \mathcal{T}_{\vlambda}^{n}\left(\vu\right)}{\partial \mC_{n, i,j}}
      =
      \delta_{n, i,j} \, \boldupright{e}_{n,i} \, \rvu_{j}
    \end{equation}
    where 
    \[
      \delta_{n,i, j} = \begin{cases}
        \; 1 &  \text{if the \(i,j\)th element of \(\mC_n\) is non-zero} \\
        \; 0 &  \text{otherwise.}
      \end{cases}
    \]

    Therefore,
    \begin{align*}
      \frac{\partial \mathcal{T}_{\vlambda}^{n}\left(\vu\right)}{\partial \mC_{n,i,j}}
      {\left(
      \frac{\partial \mathcal{T}_{\vlambda}^{n}\left(\vu\right)}{\partial \mC_{n, i,j}}
      \right)}^{\top}
      &=
      \left( \delta_{n, i,j} \, \boldupright{e}_{n,i} \, \rvu_{j} \right)
      {\left(
      \delta_{n, i,j} \, \boldupright{e}_{n,i} \, \rvu_{j}
      \right)}^{\top}
      =
      \delta_{n, i,j} \; \rvu_{j}^2
      \left( \boldupright{e}_{n,i} \, \boldupright{e}_{n,i}^{\top} \right)
    \end{align*}
    and thus, now denoting the Jacobian as \(\mJ_{n}\), we have
    \begin{align*}
        {\mJ_n\left(\vu\right)}^{\top} \mJ_n\left(\vu\right)
        &=
        {\left(\frac{\partial \mathcal{T}_{\vlambda}^{n}\left(\vu\right)}{\partial \vlambda}\right)}^{\top}
        \frac{\partial \mathcal{T}_{\vlambda}^{n}\left(\vu\right)}{\partial \vlambda}
        \\
        &=
        \sum_{i=1} \frac{\partial \mathcal{T}_{\vlambda}^{n}\left(\vu\right)}{\partial \vm_{i}}
        {\left(\frac{\partial \mathcal{T}_{\vlambda}^{n}\left(\vu\right)}{\partial \vm_{i}}\right)}^{\top} 
        +
        \sum_{i} \sum_{j}
        \frac{\partial \mathcal{T}_{\vlambda}^{n}\left(\vu\right)}{\partial \mC_{n,i,j}}
        {\left(
        \frac{\partial \mathcal{T}_{\vlambda}^{n}\left(\vu\right)}{\partial \mC_{n,i,j}}
        \right)}^{\top}
        \\
        &=
        \sum_{i} \boldupright{e}_{n,i} \, \boldupright{e}_{n,i}^{\top}
        +
        \sum_{i} \sum_{j}
        \delta_{n, i,j} \; \rvu_{j}^2 
        \left( \boldupright{e}_{n,i} \, \boldupright{e}_{n,i}^{\top}\right)
        \\
        &=
        \boldupright{I}_n
        +
        \sum_{i} \sum_{j}
        \delta_{n, i,j} \; \rvu_{j}^2 
        \left( \boldupright{e}_{n,i} \, \boldupright{e}_{n,i}^{\top}\right),
\shortintertext{assuming \(\delta_{n,i,j} = \delta_{n,j}\) for all \(n,j\), then}
        &=
        \boldupright{I}_n
        +
        \sum_j
        \delta_{n,j} \rvu_j^2 \; 
        \left(
        \sum_i \boldupright{e}_{n,i} \, \boldupright{e}_{n,i}^{\top}
        \right)
        \\
        &=
        \boldupright{I}_n
        +
        \sum_j
        \delta_{n,j} \rvu_j^2 \; \boldupright{I}_n
        \end{align*}
    where 
    \[  
      \delta_{n,j} = \begin{cases}
        \; 1 &  \text{if the \(j\)th column  of \(\mC_n\) is non-zero} \\
        \; 0 &  \text{otherwise}.
      \end{cases}
    \]
\end{proofEnd}

%% file: theorems/thm_gradnorm.tex
\begin{theoremEnd}[all end, category=gradnorm]{lemma}
\label{thm:gradnorm}
    Let \(\ell_n : \mathbb{R}^d \to \mathbb{R}\) be a differentiable function and \(\mathcal{T}^{n}\) be the location-scale reparameterization function for \(\ell_n\).
    Then, we have
    \[
        {\lVert
          \nabla_{\vlambda} \ell_n(\mathcal{T}_{\vlambda}^{n}(\vu)) 
          -
          \nabla_{\vlambda} \ell_n(\mathcal{T}_{\vlambda'}^{n}(\vu)) 
        \rVert}_2^2
        =
        \left(1 + 
        {\textstyle \sum_j}
        \delta_{n,j} \rvu_j^2
        \right)
        {\lVert 
          \nabla \ell_n(\mathcal{T}_{\vlambda}^{n}(\vu)) 
          -
          \nabla \ell_n(\mathcal{T}_{\vlambda'}^{n}(\vu)) 
        \rVert}_2^2
    \]
    for any \(\vlambda, \vlambda' \in \Lambda\) and any \(\vu \in \mathbb{R}^d\).
\end{theoremEnd}
\begin{proofEnd}
    \begin{align*}
        &\norm{
          \nabla_{\vlambda} \ell_n(\mathcal{T}_{\vlambda}^{n}(\vu)) 
          -
          \nabla_{\vlambda} \ell_n(\mathcal{T}_{\vlambda'}^{n}(\vu)) 
        }_2^2
        \\
        &\;=
        \norm{
          \frac{
            \partial \mathcal{T}^{n}_{\vlambda}\left(\vu\right)
          }{
            \partial \vlambda
          }
          \nabla \ell_n(\mathcal{T}_{\vlambda}^{n}(\vu)) 
          -
          \frac{
            \partial \mathcal{T}^{n}_{\vlambda'}\left(\vu\right)
          }{
            \partial \vlambda'
          }
          \nabla \ell_n(\mathcal{T}_{\vlambda'}^{n}(\vu)) 
        }_2^2.
\shortintertext{Now notice that \(\frac{\partial \mathcal{T}_{\vlambda}^{n}}{\partial \vlambda} = \mJ^n\) independently of \(\vlambda\). Then,}
        &\;=
        \norm{
          \mJ^{n}\left(\vu\right)
          \nabla \ell_n(\mathcal{T}_{\vlambda}^{n}(\vu)) 
          -
          \mJ^{n}\left(\vu\right)
          \nabla \ell_n(\mathcal{T}_{\vlambda'}^{n}(\vu)) 
        }_2^2
        \\
        &\;=
        \norm{
          \mJ^{n}\left(\vu\right)
          \Big(
          \nabla \ell_n(\mathcal{T}_{\vlambda}^{n}(\vu)) 
          -
          \nabla \ell_n(\mathcal{T}_{\vlambda'}^{n}(\vu)) 
          \Big)
        }_2^2
        \\
        &\;=
        {\Big(
          \nabla \ell_n(\mathcal{T}_{\vlambda}^{n}(\vu)) 
          -
          \nabla \ell_n(\mathcal{T}_{\vlambda'}^{n}(\vu)) 
        \Big)}^{\top}
        {\mJ^{n}\left(\vu\right)}^{\top}
        \mJ^{n}\left(\vu\right)
        \Big(
          \nabla \ell_n(\mathcal{T}_{\vlambda}^{n}(\vu)) 
          -
          \nabla \ell_n(\mathcal{T}_{\vlambda'}^{n}(\vu)) 
        \Big),
\shortintertext{and applying \cref{thm:jacobian_reparam},} 
        &\;=
        {\Big(
          \nabla \ell_n(\mathcal{T}_{\vlambda}^{n}(\vu)) 
          -
          \nabla \ell_n(\mathcal{T}_{\vlambda'}^{n}(\vu)) 
        \Big)}^{\top}
        \left(\left(
        1
        +
        \sum_j
        \delta_{n,j} u_j^2 \; \right)\boldupright{I}_n\right)
        \Big(
          \nabla \ell_n(\mathcal{T}_{\vlambda}^{n}(\vu)) 
          -
          \nabla \ell_n(\mathcal{T}_{\vlambda'}^{n}(\vu)) 
        \Big)
        \\
        &\;=
        \left(1 + 
        \sum_j
        \delta_{n,j} u_j^2
        \right)
        \norm{ 
          \nabla \ell_n(\mathcal{T}_{\vlambda}^{n}(\vu)) 
          -
          \nabla \ell_n(\mathcal{T}_{\vlambda'}^{n}(\vu)) 
        }_2^2.
    \end{align*}
\end{proofEnd}

%% file: theorems/thm_scale_block_derivative.tex
\begin{theoremEnd}[all end,category=scaleblockderivative]{lemma}\label{thm:scale_block_derivative}
\label{thm:split_upper_bound}
Let \(\delta_{j}\) be an indicator such that \(\delta_{j} \in \{0, 1\}\) depending on the index \(j = 1, \ldots, d\), \(\rvvu\) satisfy \cref{assumption:symmetric_standard}, and all expectations be respect to \(\rvvu\).
Then,
\[
    \norm{
    \mathbb{E}
      {\textstyle\sum_{j}}
      \delta_{j} \; \rvu_{j}^2
    \rvvu \rvvu^{\top} 
    }_{2,2}
    \;\leq\;
    \sum_{j} \delta_{j} + k_{\varphi} - 1.
\]

\end{theoremEnd}
\begin{proofEnd}

Let us denote 
\[
  \mA
  =
  \mathbb{E}
  {\textstyle\sum_{j}}
  \delta_{j} \; \rvu_{j}^2
  \rvvu \rvvu^{\top}.
\]
Then, we have
\begin{alignat*}{2}
  A_{ij}
  &=
  \mathbb{E}
  {\textstyle\sum_{k}}
  \delta_{k} \; \rvu_{k}^2 \; {(\rvvu \rvvu^{\top})}_{ij}
  \\
  &=
  {\textstyle\sum_{k}}
  \delta_{k} \;  \mathbb{E} \rvu_{k}^2 \rvu_{i} \rvu_{j}.
\end{alignat*}
Especially for the diagonal elements, we have
\begin{alignat*}{2}
  A_{ii}
  &=
  \sum_{k}
  \delta_{k} \; \mathbb{E} \rvu_{k}^2 \mathbb{E} \rvu_{i}^2
  \\
  &=
  \left(
  \sum_{k \neq i}
  \delta_{k} \;
  \mathbb{E}
  \rvu_{k}^2
  \mathbb{E}
  \rvu_{i}^2
  \right)
  +
  \delta_{i} \mathbb{E} \rvu_{i}^4
  &&\qquad\text{(\(\rvu_{i} \perp\!\!\!\perp \rvu_{k}\))}
  \\
  &=
  \left(
  \sum_{k \neq  i}
  \delta_{k}
  \right)
  +
  \delta_{i} \, k_{\varphi}
  &&\qquad\text{(\cref{assumption:symmetric_standard})}
  \\
  &\leq
  \left(
  \sum_{k \neq i }
  \delta_{ k} \;
  \right)
  +
  \delta_{i} \, k_{\varphi}
  +
  \underbrace{
  \left( k_{\varphi} - \delta_{i} \, k_{\varphi} + \delta_{i} - 1 \right)
  }_{\text{\(\geq 0\) since \(k_{\varphi} \geq 1\)}}
  \\
  &=
  \left(
  \sum_{k}
  \delta_{ k}
  \right)
  +
  k_{\varphi} - 1.
\end{alignat*}

For the off-diagonal elements, 
\begin{alignat*}{2}
  A_{ij}
  &=
  \sum_{k}
  \delta_{k} \;  \mathbb{E} \rvu_{k}^2 \rvu_{i} \rvu_{j}
  \\
  &=
  \mathbb{E} 
  \left(
  \sum_{ \left(k \neq i\right) \land \left(k \neq j\right)}
  \delta_{k} \; \rvu_{k}^2 \rvu_{i} \rvu_{j}
  \right)
  +
  \mathbb{E} 
  \delta_{i} \; \rvu_{i}^2 \rvu_{i} \rvu_{j}
  +
  \mathbb{E} 
  \delta_{j} \; \rvu_{j}^2 \rvu_{i} \rvu_{j}
  \\
  &=
  \left(
  \sum_{ \left(k \neq i\right) \land \left(k \neq j\right)}
  \delta_{k} \; \mathbb{E} \rvu_{k}^2 \mathbb{E} \rvu_{i} \mathbb{E} \rvu_{j}
  \right)
  +
  \delta_{i} \; \mathbb{E} \rvu_{i}^3 \mathbb{E} \rvu_{j}
  +
  \delta_{j} \; \mathbb{E} \rvu_{i} \mathbb{E} \rvu_{j}^3
  &&\qquad\text{(\(\rvu_{i} \perp\!\!\!\perp \rvu_{j}\), \(\rvu_i \perp\!\!\!\perp \rvu_{k}\), and \(\rvu_j \perp\!\!\!\perp \rvu_{k}\))}
  \\
  &=
  0.
  &&\qquad\text{(\cref{assumption:symmetric_standard})}
\end{alignat*}
Therefore, \(\mA\) is a diagonal matrix and the \(\ell_2\) matrix norm follows as the maximal diagonal element such that
\[
 {\lVert \mA \rVert}_{2,2}
 \leq
\sum_{j} \delta_{j}
  +
  k_{\varphi} - 1.
\]

\end{proofEnd}

%% file: theorems/thm_location_scale_reparam.tex
\begin{theoremEnd}[all end, category=locationscalereparam]{lemma}
\label{thm:location_scale_reparam}
    Let \(\mathcal{T}_{\vlambda}^{n}\) be the location-scale reparameterization function for \(\ell_n\) and let the vector-valued random variable \(\rvvu = \left(\rvu_1, \ldots, \rvu_j \right)\) satisfy \cref{assumption:symmetric_standard}. 
    Then,
    \begin{align*}
      \mathbb{E}
        \left(
        1 +  {\textstyle \sum_j} \delta_{n,j} \rvu_j^2
        \right)
        {\lVert 
          \mathcal{T}_{\vlambda}^{n}\left(\rvvu\right)
          -
          \bar{\vz}_{n}
        \rVert}_2^2
      \leq
    \left({\textstyle {\textstyle \sum_j}}
        \delta_{n,j}
        +1
      \right)
      {\lVert \vm_n -  \bar{\vz}_{n} \rVert}_2^2
      +
      \left(
      {\textstyle\sum_{j}}
      \delta_{n,j}
      +
      k_{\varphi}
      \right)
      {\lVert \mC_n \rVert}_{\mathrm{F}}^2
    \end{align*}
    for any vector \(\bar{\vz}_{n}\) matching the output dimension of \(\mathcal{T}_{\vlambda}^n\) and any \(\vlambda \in \Lambda\).
\end{theoremEnd}
\begin{proofEnd}
First,
\begin{alignat}{2}
  \mathbb{E}
  \left(
    1
    +
     {\textstyle \sum_j} \delta_{n,j} \rvu_j^2
  \right)
    {\lVert \mathcal{T}_{\vlambda}^{n}\left(\rvvu\right) - \bar{\vz}_n
    \rVert}_2^2
  \nonumber
  &=
  \mathbb{E}
  \left(
    1
    +
     {\textstyle \sum_j} \delta_{n,j} \rvu_j^2
  \right)
  \norm{
  \mC_n \rvvu + \vm_n - \bar{\vz}_{n}
  }_2^2,
  \nonumber
  \\
\shortintertext{expanding the square,} 
  &=
  \mathbb{E}
  \left(
    1
    +
     {\textstyle \sum_j} \delta_{n,j} \rvu_j^2
  \right)
  \left(
    \rvvu^{\top} \mC_{n}^{\top} \mC_{n} \rvvu
    -
    2 \, {\left( \vm_n -  \bar{\vz}_{n} \right)}^{\top} \mC_{n} \rvvu
    +
    {\lVert \vm_n -  \bar{\vz}_{n} \rVert}_2^2
  \right)
  \nonumber
  \\
  &=
  \underbrace{
  \mathbb{E}
  \left(
    1
    +
     {\textstyle \sum_j} \delta_{n,j} \rvu_j^2
  \right)
  \rvvu^{\top} \mC_{n}^{\top} \mC_{n} \rvvu
  }_{T^{(1)}}
  -
  2
  \underbrace{
  \mathbb{E}
  \left(
    1
    +
     {\textstyle \sum_j} \delta_{n,j} \rvu_j^2
  \right)
  {\left( \vm_n -  \bar{\vz}_{n} \right)}^{\top} \mC_{n} \rvvu
  }_{T^{(2)}}
  \nonumber
  \\
  &\qquad
  +
  \underbrace{
  \mathbb{E}
  \left(
    1
    +
     {\textstyle \sum_j} \delta_{n,j} \rvu_j^2
  \right)
  {\lVert \vm_n -  \bar{\vz}_{n} \rVert}_2^2.
  }_{T^{(3)}}
  \label{eq:thm_location_scale_reparam_main}
\end{alignat}
We now solve for the individual terms \(T^{(1)}\), \(T^{(2)}\), and \(T^{(3)}\).

\paragraph{Derivation of \(T^{(1)}\)}
First,
\begin{align}
  T^{(1)}
  &=
  \mathbb{E}
  \left(
    1
    +
     {\textstyle \sum_j} \delta_{n,j} \rvu_j^2
  \right)
  \rvvu^{\top} \mC_{n}^{\top} \mC_{n} \rvvu
  \nonumber
  \\
  &=
  \mathbb{E}
  \left(
    1
    +
     {\textstyle \sum_j} \delta_{n,j} \rvu_j^2
  \right)
  \mathrm{tr}
  \left(
    \rvvu^{\top} \mC_{n}^{\top} \mC_{n} \rvvu
  \right)
  \nonumber
  \\
  &=
  \mathbb{E}
  \left(
    \rvvu^{\top} \mC_{n}^{\top} \mC_{n} \rvvu
  \right)
  +
  \mathbb{E}
  \left(
     {\textstyle \sum_j} \delta_{n,j} \rvu_j^2
  \right)
  \mathrm{tr}
  \left(
    \rvvu^{\top} \mC_{n}^{\top} \mC_{n} \rvvu
  \right),
  \nonumber
\shortintertext{applying \cref{thm:trace_estimator} on the first term,}   &=
  \norm{\mC_{n}}_{\mathrm{F}}^2
  +
  \mathbb{E}
  \left(
     {\textstyle \sum_j} \delta_{n,j} \rvu_j^2
  \right)
  \mathrm{tr}
  \left(
    \rvvu^{\top} \mC_{n}^{\top} \mC_{n} \rvvu
  \right),
  \nonumber
\shortintertext{rotating the elements in the trace and pushing the scalar coefficient into the trace,}
  &=
  \norm{\mC_{n}}_{\mathrm{F}}^2
  +
  \mathrm{tr}
  \left(
    \mC_{n}^{\top} \mC_{n} 
    \mathbb{E}
    \left(
      {\textstyle \sum_j} \delta_{n,j} \rvu_j^2
      \rvvu \rvvu^{\top} 
    \right)
  \right).
  \nonumber
\shortintertext{Since we only deal with real matrices, we have  \(\mC_{n}^{\top} \mC_{n} \succeq 0\), enabling the use of \cref{thm:matrix_trace_pd_bound} as}
  &\leq
  {\lVert \mC_{n} \rVert}_{\mathrm{F}}^2
  +
  \norm{
    \mathbb{E}{\textstyle \sum_j}
    \delta_{n,j} \; \rvu_{j}^2
    \rvvu \rvvu^{\top} 
  }_{2,2}
  \mathrm{tr}
  \left(
    \mC_{n}^{\top} \mC_{n} 
  \right)
  \nonumber
  \\
  &=
  {\lVert \mC_{n} \rVert}_{\mathrm{F}}^2
  +
  \norm{
    \mathbb{E}{\textstyle \sum_j}
    \delta_{n,j} \; \rvu_{j}^2
    \rvvu \rvvu^{\top} 
  }_{2,2}
  {\lVert \mC_{n} \rVert}_{\mathrm{F}}^2
  \nonumber
  \\
  &=
  \left(
  1
  +
  \norm{
    \mathbb{E} {\textstyle \sum_j}
    \delta_{n,j} \; \rvu_{j}^2
    \rvvu \rvvu^{\top} 
  }_{2,2}
  \right)
  {\lVert \mC_{n} \rVert}_{\mathrm{F}}^2.
  \label{eq:thm_location_scale_t1_partial}
\end{align}
By using \cref{thm:scale_block_derivative}, we have
\begin{align*}
  \norm{
    \mathbb{E}
      {\textstyle \sum_{j}}
      \delta_{n,j} \; \rvu_{j}^2
    \rvvu \rvvu^{\top} 
  }_{2,2}
  \leq
  {\textstyle \sum_{j}}
  \delta_{n,j}
  +
  k_{\varphi} - 1
\end{align*}
Therefore, back to \cref{eq:thm_location_scale_t1_partial},
\begin{align}
  T^{(1)}
  &=
  \left(
  1 + 
  \norm{
    \mathbb{E}
    {\textstyle \sum_{j}}
    \delta_{n,j} \; \rvu_{j}^2
    \rvvu \rvvu^{\top} 
  }_{2,2}
  \right)
  {\lVert \mC_{n} \rVert}_{\mathrm{F}}^2
  \leq
  \left(
  {\textstyle \sum_{j}}
  \delta_{n,j}
  +
  k_{\varphi}
  \right)
  {\lVert \mC_{n} \rVert}_{\mathrm{F}}^2.
  \label{eq:thm_location_scale_t1}
\end{align}

\paragraph{Derivation of \(T^{(2)}\)}
Meanwhile, for \(T^{(2)}\),
\begin{alignat}{2}
  T^{(2)}
  &=
  \mathbb{E}
  \left(
    1
    +
     {\textstyle \sum_j}
    \delta_{n,j} \; u_{j}^2
  \right)
  \left( \vm_n -  \bar{\vz}_{n} \right) \mC_{n} \rvvu,
  \nonumber
\shortintertext{and from \cref{assumption:symmetric_standard,thm:u_identities},}  
  &=
  \left( \vm_n -  \bar{\vz}_{n} \right) \mC_{n}  \mathbb{E} \rvvu
  \;+\;
  \left( \vm_n -  \bar{\vz}_{n} \right) \mC_{n}
  \left( {\textstyle \sum_j}
    \delta_{n,j} \; 
    \mathbb{E}
    \rvu_{j}^2 \rvvu
  \right)
  \nonumber
  \\
  &=
  0.
  \label{eq:thm_location_scale_t2}
\end{alignat}

\paragraph{Derivation of \(T^{(3)}\)}
Finally,
\begin{alignat}{2}
  T^{(3)}
  &=
  \mathbb{E}
  \left(
    1
    +
     {\textstyle \sum_j}
    \delta_{n,j} \; u_{j}^2
  \right)
  {\lVert \vm_n -  \bar{\vz}_{n} \rVert}_2^2
  \nonumber
  \\
  &=
  \left(
    1
    +
     {\textstyle \sum_j}
    \delta_{n,j} \; \mathbb{E} u_{j}^2
  \right)
  {\lVert \vm_n-  \bar{\vz}_{n} \rVert}_2^2,
  \nonumber
\shortintertext{and from \cref{assumption:symmetric_standard},}  
  &=
  \left(
    1
    +
     {\textstyle \sum_j}
    \delta_{n,j} 
  \right)
  {\lVert \vm_n -  \bar{\vz}_{n} \rVert}_2^2.
\label{eq:thm_location_scale_t3}
\end{alignat}

Combining all the results to \cref{eq:thm_location_scale_reparam_main},
\begin{alignat*}{2}
  \mathbb{E}
  \left(
    1
    +
     {\textstyle \sum_j}
    \delta_{n,j} \; \rvu_{j}^2
  \right)
  \norm{
    \mathcal{T}_{\vlambda}^n\left(\rvvu\right)
    -
    \bar{\vz}_{n}
  }_2^2
  &\leq
  T^{(1)}
  -2 T^{(2)}
  +
  T^{(3)},
\shortintertext{applying \cref{eq:thm_location_scale_t1,eq:thm_location_scale_t2,eq:thm_location_scale_t3},}
  &\leq
  \left({\textstyle \sum_j}
    \delta_{n,j} 
    +1
  \right)
  {\lVert \vm_n -  \bar{\vz}_{n} \rVert}_2^2
  +
  \left(
  {\textstyle \sum_{j}}
  \delta_{n,j}
  +
  k_{\varphi}
  \right)
  {\lVert \mC_n \rVert}_{\mathrm{F}}^2.
\end{alignat*}

\end{proofEnd}

%% file: theorems/thm_quadratic_gradient_variance.tex
\begin{theoremEnd}[category=optimalgradientvariance]{theorem}\label{thm:quadratic_gradient_variance}
    Let \(\ell_n\) be \(L_n\)-smooth for some \(n = 1, \ldots, N\) and \cref{assumption:variation_family} hold. 
    Then, the gradient variance of \(\rvvg_M\) is bounded as
    {\small
    \setlength{\belowdisplayskip}{1ex} \setlength{\belowdisplayshortskip}{1ex}
    \setlength{\abovedisplayskip}{1ex} \setlength{\abovedisplayshortskip}{1ex}
    \begin{align*}
      \mathrm{tr}\mathbb{V} \, \rvvg_M\left(\vlambda\right)
      \leq
      \frac{N}{M}
      \left(d^* + k_{\varphi}\right)
      \sum_{n=1}^N
      L_{n}^2 \,
      \left(
        \norm{
          \vm_n - \bar{\vz}_n 
        }_2^2
        +
        \norm{
          \mC_n
        }_{\mathrm{F}}^2
      \right),
    \end{align*}
    }%
    where \(\bar{\vz}_n\) is a stationary point of \(\ell_n\) and 
    {%
    \setlength{\belowdisplayskip}{0ex} \setlength{\belowdisplayshortskip}{0ex}
    \setlength{\abovedisplayskip}{1ex} \setlength{\abovedisplayshortskip}{1ex}
    \[ 
      d^* \triangleq  \max_{n} {\textstyle\sum_{j}} \delta_{n,j} 
    \]
    }%
    is the effective dimensionality.
\end{theoremEnd}
\vspace{-1.5ex}
\begin{proofEnd}
    The proof can be seen as a generalization of \citet[Theorem 1]{domke_provable_2019} to the case where \(\mC\) is structured.
    Firstly, 
    \begin{alignat*}{2}
        \mathrm{tr}\mathbb{V}{\rvvg_M\left(\vlambda\right)}
        &\leq
        \mathrm{tr}\V{  
        \frac{1}{M}
        \sum_{m=1}^M
        \sum_{n=1}^N
        \nabla_{\vlambda} \ell_{n} \left(\mathcal{T}_{\vlambda}^n\left(\rvvu_m\right)\right) 
        }
        \\
        &=
        \frac{1}{M}
        \mathrm{tr}\V{  
        \sum_{n=1}^N
        \nabla_{\vlambda} \ell_{n} \left(\mathcal{T}_{\vlambda}^n\left(\rvvu\right)\right) 
        }
        &&\qquad\text{(\(\rvvu_1, \ldots, \rvvu_M\) are \textit{i.i.d.})}
        \\
        &\leq
        \frac{N}{M}
        \sum_{n=1}^N
        \mathrm{tr}\V{  
        \nabla_{\vlambda} \ell_{n} \left(\mathcal{T}_{\vlambda}^n\left(\rvvu\right)\right) 
        }
        &&\qquad\text{(\cref{thm:variance_of_sum_bound})}
        \\
        &\leq
        \frac{N}{M}
        \sum_{n=1}^N
        \mathbb{E}\norm{
          \nabla_{\vlambda} \ell_{n} \left(\mathcal{T}_{\vlambda}^n\left(\rvvu\right)\right)
        }_2^2.
        \nonumber
    \end{alignat*}

    Now, the individual expected-squared norms can be bounded as
    \begin{alignat*}{2}
        \mathbb{E}\norm{
          \nabla_{\vlambda} \ell_{n} \left(\mathcal{T}_{\vlambda}^n\left(\rvvu\right)\right)
        }_2^2
        &=
        \mathbb{E}
        \left( 1 + {\textstyle\sum_{j}} \delta_{n,j} \rvu_j^2 \right)
        \norm{
          \nabla \ell_{n} \left(\mathcal{T}_{\vlambda}^n\left(\rvvu\right)\right)
        }_2^2
        &&\qquad\text{(\cref{thm:gradnorm})}
        \\
        &=
        \mathbb{E}
        \left( 1 + {\textstyle\sum_{j}} \delta_{n,j} \rvu_j^2 \right)
        \norm{
          \nabla \ell_{n} \left(\mathcal{T}_{\vlambda}^n\left(\rvvu\right)\right)
          -
          \nabla \ell_{n}\left( \bar{\vz}_n \right)
        }_2^2
        &&\qquad\text{(since \(\nabla \ell_{n}\left( \bar{\vz}_n \right) = \boldupright{0}\))}
        \\
        &\leq
        L_{n}^2 \,
        \mathbb{E}
        \left( 1 + {\textstyle\sum_{j}} \delta_{n,j} \rvu_j^2 \right)
        \norm{
          \mathcal{T}_{\vlambda}^n\left(\rvvu\right)
          -
          \bar{\vz}_n
        }_2^2
        &&\qquad\text{(\(L_n\)-smoothness)}
        \\
        &\leq
        L_{n}^2 \,
        \left({\textstyle\sum_{j}} \delta_{n,j} + k_{\varphi}\right)
        \left(
        \norm{
          \vm_n  - \bar{\vz}_n
        }_2^2
        +
        \norm{
          \mC_n
        }_{\mathrm{F}}^2
        \right).
        &&\qquad\text{(\cref{thm:location_scale_reparam})}
    \end{alignat*}

    The sum of all the datapoints can then be bounded as
    \begin{alignat*}{2}
        \mathrm{tr}\mathbb{V}{\rvvg_M\left(\vlambda\right)}
        &=
        \frac{N}{M}
        \sum_{n=1}^N
        L_{n}^2 \,
        \left({\textstyle\sum_{j}} \delta_{n,j} + k_{\varphi}\right)
        \left(
        \norm{
          \vm_n - \bar{\vz}_n
        }_2^2
        +
        \norm{
          \mC_n
        }_{\mathrm{F}}^2
        \right)
        \\
        &\leq
        \frac{N}{M}
        \left(d^* + k_{\varphi}\right)
        \sum_{n=1}^N
        L_{n}^2 \,
        \left(
        \norm{
          \vm_n - \bar{\vz}_n 
        }_2^2
        +
        \norm{
          \mC_n
        }_{\mathrm{F}}^2
        \right).
    \end{alignat*}
\end{proofEnd}

%% file: theorems/thm_expected_smoothness.tex
\begin{theoremEnd}[all end, category=convexexpectedsmooth]{lemma}[\textbf{Convex Expected Smoothness}]\label{thm:expected_smoothness}
    Let \(\ell\) be \(\mu\)-strongly convex and \(L\)-smooth, \(\ell_n\) be \(L_n\)-smooth for \(n=1, \dots, N\), and \cref{assumption:variation_family} hold. Then, we have
    \[
      \mathbb{E}\lVert 
      \rvvg_M\left(\vlambda\right) 
      - 
      \rvvg_M\left(\vlambda'\right) 
      \rVert_2^2 
      \leq
        2
        \left(
        \frac{N}{M}
        \left(
        d^*
        +
        k_{\varphi}
        \right)
        \sum_{n=1}^N
        \frac{L_n^2}{\mu}
        +
        L
        \right)
        \mathrm{D}_{f}\left(\vlambda, \vlambda'\right)
    \]
    for any \(\vlambda, \vlambda' \in \Lambda\), 
    where 
    \(\mathrm{D}_{f}\) is the Bregman divergence generated by \(f\), and \(d^* = \max_{n} \sum_{j} \delta_{n,j}\) is the effective dimensionality.
\end{theoremEnd}
\begin{proofEnd}
    The proof is a generalization of Lemma 3 by \citet{kim_convergence_2023}, which a strategy of applying smoothness~\citep{domke_provable_2019} and then applying ``quadratic functional growth'' \citep{kim_practical_2023}.
    
    First notice that if \(\ell\) is \(L\)-smooth, then the energy is also \(L\)-smooth by virtue of \citet[Theorem 1]{domke_provable_2020}.
    In our case, \(\ell\) is \(N L_{\mathrm{max}}\) smooth.
    Therefore, 
    \begin{alignat*}{2}
        \mathbb{E}{\lVert 
          \rvvg_M\left(\vlambda\right) 
          - 
          \rvvg_M\left(\vlambda'\right) 
        \rVert}_2^2 
        &=
        \mathrm{tr} \V{
          \rvvg_M\left(\vlambda\right) - \rvvg_M\left(\vlambda'\right)
        }
        +
        \norm{\nabla f\left(\vlambda\right) - \nabla f\left(\vlambda'\right)}_2^2
        \\
        &\leq
        \mathrm{tr} \V{
          \rvvg_M\left(\vlambda\right) - \rvvg_M\left(\vlambda'\right)
        }
        +
        2 L
        \left(
          f(\vlambda) - f\left(\vlambda'\right)
          -
          \inner{
           \nabla f(\vlambda')
          }{
            \vlambda - \vlambda'
          }
        \right)
        &&\quad\text{(\(L\)-smoothness of \(\ell\))}
        \\
        &=
        \mathrm{tr} \V{
          \rvvg_M\left(\vlambda\right) - \rvvg_M\left(\vlambda'\right)
        }
        +
        2 L \, \mathrm{D}_{f}\left(\vlambda, \vlambda'\right).
    \end{alignat*}

    The variance term can be bounded as
    \begin{align*}
        \mathrm{tr} \V{
          \rvvg_M\left(\vlambda\right) - \rvvg_M\left(\vlambda'\right)
        }
        &=
        \frac{1}{M^2} 
        \mathrm{tr}\V{
          \sum_{m=1}^M
          \sum_{n=1}^N
          \ell_{n}\left(\mathcal{T}_{\vlambda}^n\left(\rvvu_m\right)\right)
          - 
          \ell_{n}\left(\mathcal{T}_{\vlambda'}^n\left(\rvvu_m\right)\right)
        }
        \\
        &=
        \frac{1}{M} 
        \mathrm{tr}\V{
          \sum_{n=1}^N
          \ell_{n}\left(\mathcal{T}_{\vlambda}^n\left(\rvvu\right)\right)
          - 
          \ell_{n}\left(\mathcal{T}_{\vlambda'}^n\left(\rvvu\right)\right)
        }
        &&\quad\text{(\(\rvvu_1, \ldots, \rvvu_m\) are \textit{i.i.d.})}
        \\
        &\leq
        \frac{N}{M} 
        \sum_{n=1}^N
        \mathrm{tr}\V{
          \ell_{n}\left(\mathcal{T}_{\vlambda}^n\left(\rvvu\right)\right)
          - 
          \ell_{n}\left(\mathcal{T}_{\vlambda'}^n\left(\rvvu\right)\right)
        }
        &&\quad\text{(\cref{thm:variance_of_sum_bound})}
        \\
        &\leq
        \frac{N}{M} 
        \sum_{n=1}^N
        \mathbb{E}{
        \norm{
          \ell_{n}\left(\mathcal{T}_{\vlambda}^n\left(\rvvu\right)\right)
          - 
          \ell_{n}\left(\mathcal{T}_{\vlambda'}^n\left(\rvvu\right)\right)
        }_2^2
        }
        \\
        &=
        \frac{N}{M}
        \sum_{n=1}^N
        \mathbb{E} \left( 1 + {\textstyle\sum_j} \delta_{n,j} \rvu_j^2 \right) 
        {\lVert
          \nabla \ell_n (\mathcal{T}^n_{\vlambda}(\rvvu)) - \nabla \ell_n (\mathcal{T}^n_{\vlambda'}(\rvvu)) 
        \rVert}_2^2
        &&\quad\text{(\cref{thm:gradnorm})}
        \\
        &\leq
        \frac{N}{M}
        \sum_{n=1}^N
        L_n^2
        \mathbb{E} \left( 1 + {\textstyle\sum_j} \delta_{n,j} \rvu_j^2 \right)
        {\lVert 
          \mathcal{T}^n_{\vlambda}\left(\rvvu\right) - \mathcal{T}^n_{\vlambda'}\left(\rvvu\right)
        \rVert}_2^2
        &&\quad\text{(\(L_n\)-smoothness)}
        \\
        &\leq
        \frac{N}{M}
        \left\{
        \sum_{n=1}^N
        L_n^2
        \left(
        {\textstyle\sum_{j}}
        \delta_{n,j}
        +
        k_{\varphi}
        \right)
        \right\}
        \norm{\vlambda - \vlambda'}_2^2.
        &&\quad\text{(\cref{thm:reparam_gradnorm})}
    \end{align*}
    
    Now, we bound \(\norm{\vlambda - \vlambda'}\) by the Bregman divergence generated by \(f\) as done by \citet[Theorem 1]{kim_practical_2023}.
    For this, we convert the squared distance of the variational parameters (\(\vlambda\)-space) into the squared distance in model parameters (\(\vz\)-space).
    That is, 
    \begin{align}
      \norm{
        \vlambda - \vlambda'
      }_2^2
      &=
      \norm{
        \mC - \mC'
      }_{\mathrm{F}}^2
      +
      \norm{
        \vm - \vm'
      }_2^2,
      \nonumber
\shortintertext{using the identity in \cref{thm:trace_estimator},}
      &=
      \mathbb{E}\norm{
        \left( \mC - \mC' \right) \rvvu
      }_{2}^2
      +
      \norm{
        \vm - \vm'
      }_2^2
      \nonumber
      \\
      &=
      \mathbb{E}
      {\lVert
        \mathcal{T}_{\vlambda}\left(\rvvu\right) - \mathcal{T}_{\vlambda'}\left(\rvvu\right)
      \rVert}_2^2
      \nonumber
\shortintertext{by the \(\mu\)-strong log-concavity of the posterior,}
      &\leq
      \frac{2}{\mu} 
      \mathbb{E} \left( \ell \left(\mathcal{T}_{\vlambda}\left(\rvvu\right)\right) - \ell \left(\mathcal{T}_{\vlambda}\left(\rvvu\right)\right) - 
      \langle 
      \nabla \ell \left(\mathcal{T}_{\vlambda^{\prime}}\left(\rvvu\right)\right), \mathcal{T}_{\vlambda}\left(\rvvu\right) - \mathcal{T}_{\vlambda^{'}}\left(\rvvu\right)\rangle
      \right) 
      \nonumber
      \\
      &=
      \frac{2}{\mu}
      \left( f\left( \vlambda \right) - f\left( \vlambda'\right) - \mathbb{E} \langle \nabla \ell \left(\mathcal{T}_{\vlambda^{\prime}}\left(\rvvu\right)\right), \mathcal{T}_{\vlambda}\left(\rvvu \right) - \mathcal{T}_{\vlambda'}\left(\rvvu \right) \rangle \right),
      \nonumber
\shortintertext{and applying Lemma 10 by \citet{kim_convergence_2023} on the inner product term,}
      &=
      \frac{2}{\mu} 
      \left( f\left( \vlambda \right) - f\left( \vlambda'\right) -  \langle \nabla f\left( \vlambda'\right), \vlambda - \vlambda^{'} \rangle \right)
      \nonumber
      \\
      &=
      \frac{2}{\mu} \mathrm{D}_{f}\left(\vlambda, \vlambda'\right).
      \label{eq:functional_growth}
    \end{align}

    Combining the results, 
    \begin{alignat*}{2}
        \mathbb{E}{\lVert 
          \rvvg_M\left(\vlambda\right) 
          - 
          \rvvg_M\left(\vlambda'\right) 
        \rVert}_2^2 
        &=
        \mathrm{tr} \V{
          \rvvg_M\left(\vlambda\right) - \rvvg_M\left(\vlambda'\right)
        }
        +
        2 L \, \mathrm{D}_{f}\left(\vlambda, \vlambda'\right)
        \\
        &=
        \frac{N}{M}
        \left\{
        \sum_{n=1}^N
        L_n^2
        \left(
        {\textstyle\sum_{j}}
        \delta_{n,j}
        +
        k_{\varphi}
        \right)
        \right\}
        \norm{\vlambda - \vlambda'}_2^2
        +
        2 L \, \mathrm{D}_{f}\left(\vlambda, \vlambda'\right)
        \\
        &\leq
        \frac{N}{M}
        \left\{
        \sum_{n=1}^N
        L_n^2
        \left(
        {\textstyle\sum_{j}}
        \delta_{n,j}
        +
        k_{\varphi}
        \right)
        \right\}
        \frac{2}{\mu}
        \mathrm{D}_{f}\left(\vlambda, \vlambda'\right)
        +
        2 L \, \mathrm{D}_{f}\left(\vlambda, \vlambda'\right)
        &&\quad\text{(\cref{eq:functional_growth})}
        \\
        &=
        2
        \left(
        \frac{N}{M}
        \sum_{n=1}^N
        \frac{L_n^2}{\mu}
        \left(
        {\textstyle\sum_{j}}
        \delta_{n,j}
        +
        k_{\varphi}
        \right)
        +
        L
        \right)
        \mathrm{D}_{f}\left(\vlambda, \vlambda'\right),
        &&\quad\text{(Reorganized)}
        \\
        &\leq
        2
        \left(
        \frac{N}{M}
        \sum_{n=1}^N
        \frac{L_n^2}{\mu}
        \left(
        d^*
        +
        k_{\varphi}
        \right)
        +
        L
        \right)
        \mathrm{D}_{f}\left(\vlambda, \vlambda'\right)
        &&\quad\text{(\(\sum_{j}\delta_{n,j} \leq d^*\) for all \(n = 1, \ldots, N\))}
        \\
        &=
        2
        \left(
        \frac{N}{M}
        \left(
        d^*
        +
        k_{\varphi}
        \right)
        \frac{\sum_{n=1}^N L_n^2}{\mu}
        +
        L
        \right)
        \mathrm{D}_{f}\left(\vlambda, \vlambda'\right).
        &&\quad\text{(Pushed-in the summation)}
    \end{alignat*}
\end{proofEnd}

%% file: theorems/thm_variance_transfer.tex
\begin{theoremEnd}[all end, category=variancetransfer]{lemma}
    Let \(\ell\) be \(\mu\)-strongly convex, \(\ell_n\) be \(L_n\)-smooth for \(n=1, \dots, N\), and the \cref{assumption:variation_family} hold.
    Then, the variance of the \(M\)-sample reparameterization gradient of the energy is bounded as
    \[
      \mathrm{tr} \mathbb{V} \, \rvvg\left(\vlambda\right)
      \leq
      \frac{4 N}{M}
      \left(d^* + k_{\varphi}\right) 
      \frac{\overline{L^2}}{\mu}\mathrm{D}_f \left(\vlambda , \vlambda^{\prime}\right)
      +
      \frac{4 N}{M} \, \sum_{n=1}^N \mathrm{tr}\V{ \nabla_{\vlambda^{\prime}} 
      \ell_n\left(\mathcal{T}_{\vlambda'}\left(\rvvu\right)\right)
      },
    \]
    for any \(\vlambda, \vlambda' \in \Lambda\), where \(\overline{L^2} = \sum_{n=1}^N L_n^2\) and \(\mathrm{D}_{f}\) is the Bregman divergence generated from the energy \(f\).
\end{theoremEnd}
\begin{proofEnd}
    From the definition of our gradient estimator, we have
    \begin{alignat}{2}
        \mathrm{tr} \mathbb{V} \rvvg \left(\vlambda\right)
        &= 
        \mathrm{tr} \V{ \frac{1}{M}\sum_{m=1}^M \sum_{n=1}^N \nabla_{\vlambda} 
        \ell_{n} \left(\mathcal{T}_{\vlambda}\left(\rvvu_m\right)\right) }
        \nonumber
        \\
        &= 
        \frac{1}{M^2}\sum_{m=1}^M   
        \mathrm{tr} \, \V{
          \sum_{n=1}^N
          \nabla_{\vlambda} \ell_{n} \left(\mathcal{T}_{\vlambda}\left(\rvvu_m\right)\right)
        }
        &&\qquad\text{(Independence of \(\rvvu_1, \ldots, \rvvu_m\))} 
        \nonumber
        \\
        &= 
        \frac{1}{M}
        \mathrm{tr} \, \mathbb{V} \left[
        \sum_{n=1}^N 
        \nabla_{\vlambda} 
        \ell_n \left(\mathcal{T}_{\vlambda}\left(\rvvu\right)\right)\right]
        &&\qquad\text{(since \(\rvvu_m\) is \textit{i.i.d.})}
        \nonumber
        \\
        &\leq
        \frac{N}{M}
        \sum_{n=1}^N 
        \mathrm{tr} \mathbb{V} \, 
        \nabla_{\vlambda} 
        \ell_n \left(\mathcal{T}_{\vlambda}\left(\rvvu\right)\right)
        &&\qquad\text{(\cref{thm:variance_of_sum_bound})}
        \label{eq:variance_transfer_eq1}
    \end{alignat}
    From here, the proof proceeds identically to Lemma 8.20 in \citet{garrigos_handbook_2023}.
    Starting from \cref{thm:variance_bound_squared_difference}
    \begin{alignat}{2}
        \mathrm{tr}\mathbb{V} \,
        \nabla_{\vlambda} 
        \ell_n \left(\mathcal{T}_{\vlambda}\left(\rvvu\right)\right)
        &\leq
        \mathbb{E} 
        \norm{
          \nabla_{\vlambda} 
          \ell_n \left(\mathcal{T}_{\vlambda}\left(\rvvu\right)\right)
          -
          \nabla f_n\left(\vlambda'\right)
        }_2^2 
        \nonumber
        \\
        &=
        \mathbb{E} 
        \norm{
          \nabla_{\vlambda} 
          \ell_n \left(\mathcal{T}_{\vlambda}\left(\rvvu\right)\right)
          -
          \nabla_{\vlambda'} \ell_n \left(\mathcal{T}_{\vlambda'}\left(\rvvu\right)\right)
          +
          \nabla_{\vlambda'} \ell_n \left(\mathcal{T}_{\vlambda'}\left(\rvvu\right)\right)
          -
          \nabla f_n\left(\vlambda'\right)
        }_2^2, 
        \nonumber
\shortintertext{applying \({(a+b)}^2 \leq 2 a^2 + 2 b^2\),}
        &\leq
        2 \,
        \mathbb{E} 
        \norm{
          \nabla_{\vlambda} 
          \ell_n \left(\mathcal{T}_{\vlambda}\left(\rvvu\right)\right)
          -
          \nabla_{\vlambda'} \ell_n \left(\mathcal{T}_{\vlambda'}\left(\rvvu\right)\right)
        }_2^2
        +
        2 \,
        \mathbb{E} 
        \norm{
          \nabla_{\vlambda'} \ell_n \left(\mathcal{T}_{\vlambda'}\left(\rvvu\right)\right)
          -
          \nabla f_n\left(\vlambda'\right)
        }_2^2 
        \nonumber
        \\
        &=
        2 \, \mathbb{E} 
        \norm{
          \nabla_{\vlambda} 
          \ell_n \left(\mathcal{T}_{\vlambda}\left(\rvvu\right)\right)
          -
          \nabla_{\vlambda'} \ell_n \left(\mathcal{T}_{\vlambda'}\left(\rvvu\right)\right)
        }_2^2
        +
        2 \, \mathrm{tr}\mathbb{V} \, \nabla_{\vlambda^{\prime}} 
        \ell_n\left(\mathcal{T}_{\vlambda'}\left(\rvvu\right)\right),
        \nonumber
\shortintertext{and \cref{thm:convex_expected_smoothness},}
        &\leq
        4\left(\sum_{j} \delta_{n,j} + k_{\varphi}\right) \frac{L_n^2}{\mu}\mathrm{D}_f \left(\vlambda , \vlambda^{\prime}\right)
        +
        2 \, \mathrm{tr}\mathbb{V} \, \nabla_{\vlambda^{\prime}} 
        \ell_n\left(\mathcal{T}_{\vlambda'}\left(\rvvu\right)\right).
        \label{thm:variance_transfer_eq2}
    \end{alignat}
    Combining this into \cref{eq:variance_transfer_eq1,thm:variance_transfer_eq2}, we get
    \begin{alignat*}{2}
        \mathrm{tr} \mathbb{V} \rvvg \left(\vlambda\right)
        &= 
        \frac{N}{M}
        \sum_{n=1}^N
        \mathrm{tr}\mathbb{V} \,
        \nabla_{\vlambda} 
        \ell_n \left(\mathcal{T}_{\vlambda}\left(\rvvu\right)\right)
        \\
        &\leq
        \underbrace{
        \frac{N}{M}
        \sum_{n=1}^N
        4\left(\sum_{j} \delta_{n,j} + k_{\varphi}\right) 
        \frac{L_n^2}{\mu}\mathrm{D}_f \left(\vlambda , \vlambda^{\prime}\right)
        }_{V_{\vlambda}}
        +
        \underbrace{
        \frac{N}{M} \sum_{n=1}^N \, \mathrm{tr}\V{ 2    \nabla_{\vlambda^{\prime}} 
        \ell_n\left(\mathcal{T}_{\vlambda'}\left(\rvvu\right)\right)
        }
        }_{V_{\vlambda'}}.
    \end{alignat*}

    \(V_{\vlambda}\) follows as
    \begin{alignat*}{2}
        V_{\vlambda}
        &=
        \frac{N}{M}
        \sum_{n=1}^N
        4\left(\sum_{j} \delta_{n,j} + k_{\varphi}\right) 
        \frac{L_n^2}{\mu}\mathrm{D}_f \left(\vlambda , \vlambda^{\prime}\right),
\shortintertext{where the definition of the effective dimensionality yields \(\textstyle\sum_{j} \delta_{n,j} + k_{\varphi} \leq d^* + k_{\varphi}\) for all \(n = 1, \ldots, N\) such that}
        &\leq
        \frac{N}{M}
        \left(d^* + k_{\varphi}\right) 
        \sum_{n=1}^N 
        4 \frac{L_n^2}{\mu}\mathrm{D}_f \left(\vlambda , \vlambda^{\prime}\right),
\shortintertext{and reorganizing the constants and the sum,}
        &=
        \frac{4 N}{M}
        \left(d^* + k_{\varphi}\right) 
        \frac{\sum_{n=1}^N L_n^2}{\mu}\mathrm{D}_f \left(\vlambda , \vlambda^{\prime}\right).
    \end{alignat*}

    Lastly, \(V_{\vlambda'}\) follows as
    \begin{alignat*}{2}
        V_{\vlambda'}
        =
        \frac{N}{M} \, \sum_{n=1}^N \mathrm{tr}\V{ 2 \nabla_{\vlambda^{\prime}} 
        \ell_n\left(\mathcal{T}_{\vlambda'}\left(\rvvu\right)\right)
        }
        =
        \frac{4 N}{M} \, \sum_{n=1}^N \mathrm{tr}\V{ \nabla_{\vlambda^{\prime}} 
        \ell_n\left(\mathcal{T}_{\vlambda'}\left(\rvvu\right)\right)
        }.
    \end{alignat*}
    Therefore, 
    \begin{alignat*}{2}
        \mathrm{tr} \mathbb{V} \rvvg \left(\vlambda\right)
        &\leq
        \frac{4 N}{M}
        \left(d^* + k_{\varphi}\right) 
        \frac{\sum_{n=1}^N L_n^2}{\mu}\mathrm{D}_f \left(\vlambda , \vlambda^{\prime}\right)
        +
        \frac{4 N}{M} \, \sum_{n=1}^N \mathrm{tr}\V{ \nabla_{\vlambda^{\prime}} 
        \ell_n\left(\mathcal{T}_{\vlambda'}\left(\rvvu\right)\right)
        }.
    \end{alignat*}
\end{proofEnd}

%% file: theorems/thm_proxsgd_bbvi_complexity.tex
\begin{theoremEnd}[category=complexitybbvifixed]{theorem}\label{thm:proxsgd_bbvi_complexity}
    Let \(\ell\) be \(\mu\)-strongly convex and \(L\)-smooth, \(\ell_n\) be \(L_n\)-smooth for \(n = 1, \ldots, N\), and
    \cref{assumption:variation_family} hold.
    Then, the last iterate \(\vlambda_{T+1}\) of BBVI with proximal SGD and \(\rvvg_M\) is \(\epsilon\)-close as \(\mathbb{E}\norm{ \vlambda_{T+1} - \vlambda^* }_2^2 \leq \epsilon\) to the global optimum \(\vlambda^* = \argmin F\left(\vlambda\right)\) if
{
\setlength{\abovedisplayskip}{1ex} \setlength{\abovedisplayshortskip}{1ex}
\setlength{\belowdisplayskip}{1.5ex} \setlength{\belowdisplayshortskip}{1.5ex}
  \begin{align*}
    T
    &\geq 
    \max\left(
    C_{\mathrm{var}} \frac{1}{\epsilon},\; C_{\mathrm{bias}} 
    \right)
    \log \left( 2 \Delta_0^2 \, \frac{1}{\epsilon} \right)
  \end{align*}
}%
  for some fixed stepsize \(\gamma\), where \(\Delta_0 = {\lVert \vlambda_0 - \vlambda^*\rVert}_2\) is the distance to the optimum,
{
\setlength{\abovedisplayskip}{1ex} \setlength{\abovedisplayshortskip}{1ex}
\setlength{\belowdisplayskip}{1ex} \setlength{\belowdisplayshortskip}{1ex}
  \begin{align*}
     C_{\mathrm{var}} &= 4 \frac{N}{M} \left(d^* + k_{\varphi}\right) {\textstyle\sum_{n=1}^N \kappa_n^2 \left(\norm{\vm_n^* - \bar{\vz}_n}_2^2 + \norm{\mC^*_n}_{\mathrm{F}}^2 \right)}
     \\
     C_{\mathrm{bias}} &= 2 \frac{N}{M} \left(d^* + k_{\varphi}\right) {\textstyle\sum_{n=1}^N} \kappa_n^2 + \kappa,
  \end{align*}
}%
  \(\kappa_n = L_n/\mu\), \(\kappa = L/\mu\) are the condition numbers, \(d^*\) is the effective dimensionality defined in \cref{thm:quadratic_gradient_variance}, \(\bar{\vz}_n\) is a stationary point of \(\ell_n\), and \(\vm_n^*, \mC^*_n\) are part of \(\vlambda^*\).
\end{theoremEnd}
\vspace{-1.5ex}
\begin{proofEnd}
    First, 
    \begin{enumerate}
        \item \cref{assumption:proximal_non_expansive} is satisfied as discussed in \cref{section:proxsgd_bbvi_details}, 
        \item \cref{assumption:convex_expected_smoothness} is satisfied by \cref{thm:expected_smoothness}, while
        \item \cref{assumption:bounded_gradient_variance} is satisfied by \cref{thm:quadratic_gradient_variance}.
    \end{enumerate}
    Furthermore, the constants are 
    \begin{align*}
        \mathcal{L} 
        = \frac{N}{M} \left(d^* + k_{\varphi}\right) \frac{\sum^N_{n=1} {L^2_n}}{\mu} + L
        \qquad\text{and}\qquad
        \sigma^2
        = 
        \frac{N}{M} \left(d^* + k_{\varphi}\right) \sum_{n=1}^N L_n^2\left(\norm{\vm_n^* - \bar{\vz}_n}_2^2 + \norm{\mC^*_n}_{\mathrm{F}}^2\right).
    \end{align*}
    Therefore, we can apply the results of \cref{thm:proxsgd_convergence_fixed} and consequently \cref{thm:proxsgd_complexity_fixed}, which states that an \(\epsilon\)-accurate solution can be achieved by a number of iterations of at least
    \begin{align*}
      T
      &\geq
      \max \left(
        \frac{4 \sigma^2}{\mu^2} \frac{1}{\epsilon},
        \frac{2 \mathcal{L}}{\mu},
        1
      \right)
      \log \left( 2 \lVert\vlambda_0 - \vlambda^*\rVert_2^2 \, \frac{1}{\epsilon} \right)
      \\
      &=
      \max \left(
        \frac{4 N}{M} \left(d^* + k_{\varphi}\right) \sum_{n=1}^N \frac{L^2_n}{\mu^2} \left(\norm{\vm_n^* - \bar{\vz}_n}_2^2 + \norm{\mC^*_n}_{\mathrm{F}}^2\right)
        \frac{1}{\epsilon},\;
        \frac{2 N}{M} \left(d^* + k_{\varphi}\right) \sum^N_{n=1} \frac{{L_n^2}}{\mu^2} + \frac{L}{\mu},\;
        1
      \right)
      \log \left( 2 \lVert\vlambda_0 - \vlambda^*\rVert_2^2 \, \frac{1}{\epsilon} \right),
\shortintertext{noticing that the second entry of the max function is always larger than 1, we have}
      &=
      \max \left(
        \frac{4 N}{M} \left(d^* + k_{\varphi}\right) \sum_{n=1}^N \frac{L_n^2}{\mu^2} \left(\norm{\vm_n^* - \bar{\vz}_n}_2^2 + \norm{\mC^*_n}_{\mathrm{F}}^2\right)
        \frac{1}{\epsilon},\;
        \frac{2 N}{M} \left(d^* + k_{\varphi}\right) \sum^N_{n=1} \frac{{L_n^2}}{\mu^2} + \frac{L}{\mu}
      \right)
      \log \left( 2 \lVert\vlambda_0 - \vlambda^*\rVert_2^2 \, \frac{1}{\epsilon} \right).
    \end{align*}
\end{proofEnd}

%% file: figures/tikz_structured_scale_matrix.tex
\pgfkeys{
  tikz/matrixbracestyle/.style={
    decoration={brace,raise=1ex},
    decorate, 
    thick
  }
}

\newcommand*\matrixbraceleft[4][C]{
    \draw[matrixbracestyle]
      (#1.west|-#1-#3-1.south west)
      -- 
      node[left=1.5ex] {#4} 
      (#1.west|-#1-#2-1.north west);
}

\newcommand*\matrixbraceright[4][C]{
    \draw[matrixbracestyle]
      (#1.east|-#1-#2-1.north east)
      -- 
      node[right=1.5ex] {#4} 
      (#1.east|-#1-#3-1.south east);
}

\newcommand*\matrixbracebottom[4][C]{
    \draw[matrixbracestyle] 
      (#1.south-|#1-4-#2.south east) 
      -- 
      node[below=1.5ex] {#4} 
      (#1.south-|#1-4-#3.south west);
}

\tikzset{
  greenish/.style={
    fill=green!50!lime!60,draw opacity=0.4,
    draw=green!50!lime!60,fill opacity=0.1,
  },
  cyanish/.style={
    fill=cyan!90!blue!60, draw opacity=0.4,
    draw=blue!70!cyan!30,fill opacity=0.1,
  },
  orangeish/.style={
    fill=orange!90, draw opacity=0.8,
    draw=orange!90, fill opacity=0.3,
  },
  brownish/.style={
    fill=brown!70!orange!40, draw opacity=0.4,
    draw=brown, fill opacity=0.3,
  },
  purpleish/.style={
    fill=violet!90!pink!20, draw opacity=0.5,
    draw=violet, fill opacity=0.3,    
  }
}

\begin{tikzpicture}[
   every left  delimiter/.style={xshift=1ex},
   every right delimiter/.style={xshift=-1ex},
]
    \matrix [
      matrix of math nodes,
      nodes in empty cells,
      left delimiter  = {[},
      right delimiter = {]},
      nodes = {
        minimum size = 3.2em,
        rectangle,
        anchor=center,
      }
    ] at (0,0) (C) {
      \mC_{\vz,\vz} \\
      \mC_{\vy_{1},\vz} & \mC_{\vy_{1},\vy_{1}} \\
      \vdots &  & \ddots \\
      \mC_{\vy_{N},\vz} & & & \mC_{\vy_{N},\vy_{N}} \\
    };
    

    \begin{scope}[on background layer]
      \draw[fill=gray!20, draw=none]
        (C-1-1.north west) 
        -- 
        (C-1-1.south east) 
        --
        (C-4-1.south east) 
        --
        (C-4-1.south west) 
        --
        cycle;
        
      \draw[fill=gray!20, draw=none]
        (C-2-2.north west) 
        -- 
        (C-2-2.south east) 
        --
        (C-2-2.south east) 
        --
        (C-2-2.south west) 
        --
        cycle;
        
      \draw[fill=gray!20, draw=none]
        (C-4-4.north west) 
        -- 
        (C-4-4.south east) 
        --
        (C-4-4.south east) 
        --
        (C-4-4.south west) 
        --
        cycle;
    
     
        
      \draw[fill=gray!20, draw=none]
        (C-3-3.north west) 
        -- 
        (C-3-3.south east) 
        --
        (C-3-3.south west) 
        --
        cycle;
    \end{scope}

    \begin{scope}[
      every node/.append style={transform shape},
    ]
      \matrixbraceleft{1}{1}{\(d_{\vz}\)}
      \matrixbraceleft{4}{4}{\(d_{\vy}\)}
      \matrixbraceright{1}{4}{\(d\)}
      \matrixbracebottom{1}{1}{\(d_{\vz}\)}
      \matrixbracebottom{4}{4}{\(d_{\vy}\)}
    \end{scope} 
\end{tikzpicture}

%% file: theorems/thm_noncentered_convexity.tex
\begin{theoremEnd}[category=nonstandardconvexity]{theorem}\label{thm:noncentered_param_convexity}
Let \cref{assumption:variation_family} hold.
Then, under the non-standardized parameterization, there exists a strongly log-concave posterior for which the negative ELBO is not convex.
\end{theoremEnd}
\vspace{-2ex}
\begin{proofEnd}
Our proof is based on a negative example of the global convexity of the energy \( \vlambda \mapsto \mathbb{E} \ell\left( \mathcal{T}_{\vlambda}\left(\rvvu\right) \right)\).
This directly implies that the negative ELBO may not be convex even if the posterior is log-concave.

The non-standardized parameterization applies reparameterization as
\begin{align}
  \rvvz = \mC_{\vz, \vz} \rvvu_{\vz} + \vm_{\vz} \qquad\text{and}\qquad
  \rvvy \mid \rvvz = \mC_{\vy, \vy} \rvvu_{\vy} + \vm_{\vy} + \mC_{\vy, \vz} \rvvz
  \label{eq:nonstandard_reparam_functions}
\end{align}

Now, consider the negative log-joint likelihood 
\(
  \ell(\vz,\vy) = \norm{\vz}_2^2 + \norm{\vy}_2^2,
\)
which corresponds to a standard Gaussian posterior.
Naturally, the corresponding (normalized) posterior is strongly convex.
This model can also be viewed as a 2-level hierarchical model with a single data point such that \(N=1\).

Now, the energy can be computed as
\begin{alignat*}{2}
    \mathbb{E} \ell(\mathcal{T}_{\vlambda}\left(\rvvu\right))
    &=
    \mathbb{E} 
    \norm{\rvvz}_2^2 + \norm{\rvvy}_2^2
    \\
    &=
    \mathbb{E} 
    {\lVert \mC_{\vz, \vz} \rvvu_{\vz} + \vm_{\vz} \rVert}_2^2 
    + \mathbb{E}{\lVert \mC_{\vy, \vy} \rvvu_{\vy} + \vm_{\vy} + \mC_{\vy, \vz} \rvvz \rVert}_2^2
    &&\qquad\text{(\cref{eq:nonstandard_reparam_functions})}
    \\
    &=
    \mathbb{E} 
    {\lVert \mC_{\vz, \vz} \rvvu_{\vz} + \vm_{\vz} \rVert}_2^2 
    + \mathbb{E}{\lVert \mC_{\vy, \vy} \rvvu_{\vy} + \vm_{\vy} + \mC_{\vy, \vz} \,( \mC_{\vz, \vz} \rvvu_{\vz} + \vm_{\vz})  \rVert}_2^2.
    &&\qquad\text{(\cref{eq:nonstandard_reparam_functions})}
\end{alignat*}
For clarity, we will set \(\vm_{\vz} = \boldupright{0}\) and \(\vm_{\vy} = \boldupright{0}\).
Then, 
\begin{align*}
    \mathbb{E} \ell(\mathcal{T}_{\vlambda}\left(\rvvu\right))
    &=
    \mathbb{E} 
    {\lVert \mC_{\vz, \vz} \rvvu_{\vz}\rVert}_2^2 
    + 
    \mathbb{E}{\lVert \mC_{\vy, \vy} \rvvu_{\vy} + \mC_{\vy, \vz} \mC_{\vz, \vz} \rvvu_{\vz} \rVert}_2^2.
\end{align*}
For the first term,
\begin{alignat*}{2}
    \mathbb{E} 
    {\lVert \mC_{\vz, \vz} \rvvu_{\vz}\rVert}_2^2 
    =
    {\lVert \mC_{\vz, \vz} \rVert}_{\mathrm{F}}^2.
    &&\qquad\text{(\cref{assumption:symmetric_standard,thm:trace_estimator})}
\end{alignat*}
And for the second term,
\begin{alignat*}{2}
    \mathbb{E}{\lVert \mC_{\vy, \vy} \rvvu_{\vy} + \mC_{\vy, \vz} \mC_{\vz, \vz} \rvvu_{\vz} \rVert}_2^2
    &=
    \mathbb{E}{\lVert \mC_{\vy, \vy} \rvvu_{\vy} \rVert}_{2}^2
    +
    2 \mathbb{E}\inner{
      \mC_{\vy, \vy} \rvvu_{\vy}
    }{
      \mC_{\vy, \vz} \mC_{\vz, \vz} \rvvu_{\vz}
    }
    +
    \mathbb{E}{\lVert \mC_{\vy, \vz} \mC_{\vz, \vz} \rvvu_{\vz} \rVert}_{2}^2
    \\
    &=
    \mathbb{E}{\lVert \mC_{\vy, \vy} \rvvu_{\vy} \rVert}_{2}^2
    +
    2 \inner{
      \mC_{\vy, \vy} \mathbb{E}\rvvu_{\vy}
    }{
      \mC_{\vy, \vz} \mC_{\vz, \vz} \mathbb{E}\rvvu_{\vz}
    }
    +
    \mathbb{E}{\lVert \mC_{\vy, \vz} \mC_{\vz, \vz} \rvvu_{\vz} \rVert}_{2}^2
    &&\quad\text{(since \(\rvvu_{\vy} \perp\!\!\!\perp \rvvu_{\vz}\))}
    \\
    &=
    \mathbb{E}{\lVert \mC_{\vy, \vy} \rvvu_{\vy} \rVert}_{2}^2
    +
    \mathbb{E}{\lVert \mC_{\vy, \vz} \mC_{\vz, \vz} \rvvu_{\vz} \rVert}_{2}^2
    &&\quad\text{(\cref{assumption:symmetric_standard})}
    \\
    &=
    {\lVert \mC_{\vy, \vy} \rVert}_{\mathrm{F}}^2
    +
    {\lVert \mC_{\vy, \vz} \mC_{\vz, \vz} \rVert}_{\mathrm{F}}^2.
    &&\quad\text{(\cref{assumption:symmetric_standard,thm:trace_estimator})}
\end{alignat*}
Therefore, 
\begin{align*}
    \mathbb{E} \ell(\mathcal{T}_{\vlambda}\left(\rvvu\right))
    &=
    {\lVert \mC_{\vz, \vz} \rVert}_{\mathrm{F}}^2
    +
    {\lVert \mC_{\vy, \vy} \rVert}_{\mathrm{F}}^2
    +
    {\lVert \mC_{\vy, \vz} \mC_{\vz, \vz} \rVert}_{\mathrm{F}}^2.
\end{align*}
Now, consider the case where \(d_{\vy} = 1\) and \(d_{\rvvz} = 1\).
Then, the scale matrices are all scalars such that
\begin{align*}
    \mathbb{E} \ell(\mathcal{T}_{\vlambda}\left(\rvvu\right))
    &=
    C_{z, z}^2
    +
    C_{y, y}^2
    +
    {\left( C_{y, z} C_{z, z} \right)}^2.
\end{align*}
The convexity of this function with respect to \((C_{z, z}, C_{y, y}, C_{y, z})\) is equivalent to the convexity of 
\[
  f(x, y, z) = x^2 + z^2 + x^2 y^2
\]
on \(\mathbb{R}_{+} \times \mathbb{R}_+ \times \mathbb{R}_+\).
Unfortunately, this function is not convex:
notice that the Hessian determinant is given as
\[
   \mathrm{det} \, \nabla^2 f\left(x, y, z\right) = 8 x^2 (1 - 3 y^2),
\]
which is negative for some \(y \in \mathbb{R}_+\).
Specifically, for \(0 < y < 1/\sqrt{3}\).
The fact that the determinant is negative implies that, for this region, some of the eigenvalues of \(\nabla^2 f\) are negative, ruling out convexity.

\end{proofEnd}

%% file: theorems/thm_structured_family_complexity.tex
\begin{theoremEnd}[category=structuredfamilycomplexity]{corollary}\label{thm:scalability}
    Let the assumptions of \cref{thm:proxsgd_bbvi_complexity} hold and the structured variational family with a bordered block-diagonal structure matching that of \(\ell_1, \ldots, \ell_n\) be used.
    Then, the iteration complexity of finding an \(\epsilon\)-accurate solution using BBVI with proximal SGD, \(\rvvg_M\), and some fixed stepsize is 
{
\setlength{\abovedisplayskip}{.5ex} \setlength{\abovedisplayshortskip}{.5ex}
\setlength{\belowdisplayskip}{1.ex} \setlength{\belowdisplayshortskip}{1.ex}
  \begin{align*}
    \mathcal{O}
    \left(
      \frac{N}{M}
      \left(d_{\vz} + d_{\vy} + k_{\varphi}\right) 
      \sum_{n=1}^N \kappa_{n}^2 \,
      \frac{1}{\epsilon}
      \log \left( 2 \Delta_0^2 \, \frac{1}{\epsilon} \right)
    \right).
  \end{align*}
}%
\end{theoremEnd}

%% file: sections/section_experiments.tex
\begin{figure*}
    \centering
    \begin{minipage}[t]{0.70\textwidth}
        \subfloat[\textsf{full-rank}]{
            \centering
            \input{figures/tikz_synthetic_fullrank}
            \vspace{-1ex}
        }\hspace{-1em}
        \subfloat[\textsf{structured}]{
            \centering
            \input{figures/tikz_synthetic_structured}
            \vspace{-1ex}
        }\hspace{-1em}
        \subfloat[\textsf{mean-field}]{
            \centering
            \input{figures/tikz_synthetic_meanfield}
            \vspace{-1ex}
        }
        \caption{
            \textbf{Number of iterations \(T\) required to obtain \(\epsilon\) accuracy of variational families for a given stepsize \(\gamma\).}
            \textsf{structured} behaves similarly to \textsf{mean-field}, while \textsf{full-rank} requires significantly more number of iterations, which also scales worse with respect to the number of datapoints \(n\).
        }\label{fig:synthetic}
    \end{minipage}
    \hfill
    \begin{minipage}[t]{0.27\textwidth}
        \centering
        \input{figures/tikz_scaling}
        \vspace{-1ex}
        \caption{\textbf{Scaling of variational families with respect to the number of datapoints \(n\)}.
        \textsf{full-rank} exhibits a worst scaling than \textsf{structured} and \textsf{mean-field}.
        }\label{fig:synthetic_scaling}
    \end{minipage}
    \vspace{-2ex}
\end{figure*}

\input{sections/table_problem_dimensions}
\begin{figure*}[t]
    \newcommand{\figurescale}{0.95}
    \newcommand{\figurevspace}{-1.5ex}
    \centering
    \subfloat[\textsf{rpoisson-small}]{
      \includegraphics[scale=\figurescale]{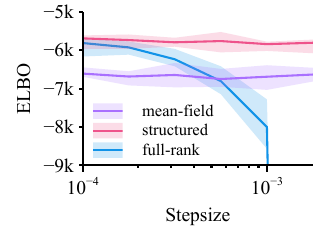}
      \vspace{\figurevspace}
    }
    \subfloat[\textsf{rpoisson-middle}]{
      \includegraphics[scale=\figurescale]{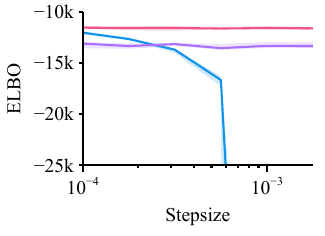}
      \vspace{\figurevspace}
    }
    \subfloat[\textsf{rpoisson-large}]{
      \includegraphics[scale=\figurescale]{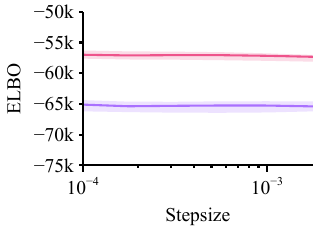}
      \vspace{\figurevspace}
    }
    \\
    \subfloat[\textsf{volatility-small}]{
      \includegraphics[scale=\figurescale]{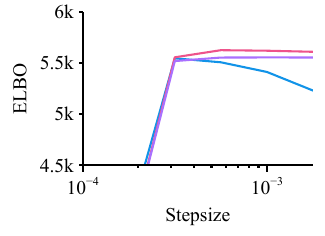}
      \vspace{\figurevspace}
    }
    \subfloat[\textsf{volatility-middle}]{
      \includegraphics[scale=\figurescale]{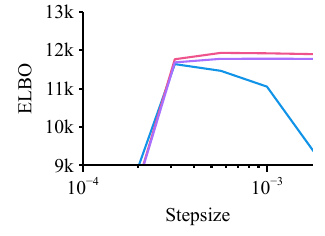}
      \vspace{\figurevspace}
    }
    \subfloat[\textsf{volatility-large}]{
      \includegraphics[scale=0.8]{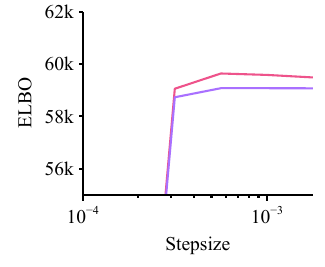}
      \vspace{\figurevspace}
    }
    \\
    \subfloat[\textsf{irt-small}]{
      \includegraphics[scale=\figurescale]{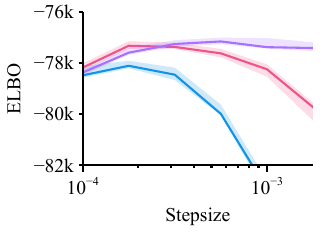}
      \vspace{\figurevspace}
    }
    \subfloat[\textsf{irt-middle}]{
      \includegraphics[scale=\figurescale]{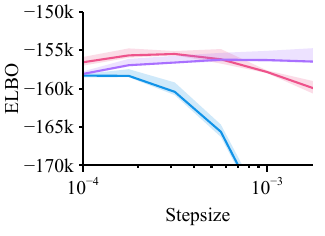}
      \vspace{\figurevspace}
    }
    \subfloat[\textsf{irt-large}]{
      \includegraphics[scale=\figurescale]{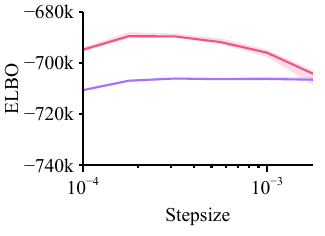}
      \vspace{\figurevspace}
    }
    \caption{
    \textbf{ELBO at \(T = 5 \times 10^4\) versus the optimizer stepsize (\(\gamma\)) on the considered problems with varying dataset sizes.} 
    The solid lines are the median over 8 independent replications, while the colored bands mark the 80\% empirical percentiles.
    }
    \label{fig:irt_stepsize}
    \vspace{-2ex}
\end{figure*}

\vspace{-0.5ex}
\section{Experiments}
\vspace{-0.5ex}
We now empirically evaluate our theoretical analysis in \cref{section:theory}.
Mainly, we will compare the scalability of mean-field, full-rank, and the structured variational family described in~\cref{section:structured}.

\vspace{-0.5ex}
\subsection{Synthetic Experiments}
\vspace{-0.5ex}
\subsubsection{Setup}
To quantitatively verify the theoretical results in \cref{section:theory}, we use proximal SGD with the proximal operator described in \cref{section:proxsgd_bbvi_details} to match the theory.
For the target distribution, we use an isotropic Gaussian target distribution of \(\ell_n\left(\vy_n, \vz\right) = -\log \mathcal{N}\left(\vy_n; 5 \boldupright{1}_{d_y}, 0.1 \boldupright{I}_{d_y}\right) - \log \mathcal{N}\left(\vz_n; 5 \boldupright{1}_{d_z}, 0.1 \boldupright{I}_{d_z}\right)\) (mean 5 and variance 0.1) where we set \(d_z = 5\), \(d_y = 3\), and vary the ``number of datapoints'' \(n\).
All variational families are initialized with a standard Gaussian.
We then run BBVI with \(M=8\) Monte Carlo samples, 50 different stepsizes over \([10^{-6}, 1]\), and estimate the sequence of expected distance to the optimum \({(r_1)}_{t \geq 0}\), where \(r_t \triangleq \mathbb{E}\norm{\vlambda_t - \vlambda^*}_2^2\).
We then estimate the minimum number of iterations \(T\) required to hit \(\epsilon\) accuracy such that \(r_{T-1} > \epsilon\) and \(r_{T} \leq \epsilon\).
We set \(\epsilon = 1\) in all cases.

\vspace{-0.5ex}
\subsubsection{Results}
\vspace{-0.5ex}
\paragraph{Effect of Stepsize}
We demonstrate the effect of stepsize on the number of iterations required to achieve \(\epsilon\)-accuracy in \cref{fig:synthetic}.
Under our setup, the distance to the optimum scales as \(\norm{\vlambda_0 - \vlambda^*}_2^2 = \Theta(n)\).
Therefore, all methods show an increase in the minimum \(T\) as \(n\) increases.
However, the \(T\) required by \textsf{full-rank} appears to increase much faster than \textsf{structured} and \textsf{mean-field}.
This is because, as predicted in \cref{section:theory}, the variance of \textsf{full-rank} also increases in \(n\), forcing the use of a smaller stepsize \(\gamma\).
Clearly, the range of stepsizes \textsf{full-rank} achieves the \(\epsilon\) accuracy threshold is much narrower and shrinks faster as \(n\) increases.
In sharp contrast, \textsf{structured} behaves very similarly to \(\textsf{mean-field}\).

\vspace{-1ex}
\paragraph{Scaling w.r.t. \(N\)}
Now, by picking the stepsize that achieves the lowest \(T\) for hitting the \(\epsilon\) accuracy target for each curve in \cref{fig:synthetic}, we can directly evaluate the iteration complexity bounds in \cref{section:theory}.
For this, we examine the scaling with respect to \(n\), which is shown in \cref{fig:synthetic_scaling}.
We can see that \textsf{full-rank} scales much worse than both \textsf{mean-field} and \textsf{structured}.
In fact, \textsf{full-rank} exhibits a quadratic scaling while \textsf{mean-field} and \textsf{structured} exhibit a linear scaling.
This confirms our theory \cref{section:theory} that \textsf{structured} is able to shave off a polynomial order in iteration complexity.
Now, one may notice that the scaling observed in \cref{section:theory} are all a polynomial order better than what the theory in \cref{section:theory} predicts (\(\mathcal{O}(N^3)\) for \textsf{full-rank} and \(\mathcal{O}(N^2)\) for structured when taking \(\sum_{n=1}^N \kappa_n = \Theta(N)\)).
This is because our target posterior is an isotropic Gaussian; there are no correlations between the gradient estimator of each component.
Therefore, our theory is a polynomial-order pessimistic due to the use of \cref{thm:variance_of_sum_bound} in \cref{section:auxiliary}.

Note that, in principle, since the dimensionality of the posterior is increasing, \(\epsilon\) needs to be increased at a rate of \(n\) so that we allocate an equal ``accuracy'' budget to each coordinate.
Therefore, considering the fact that we fix \(\epsilon\) irrespective of \(n\), the linear scaling of \(T\) w.r.t. to \(n\) is actually benign.


\subsection{Realistic Experiments}\label{section:realworld}
We will now qualitatively verify our theory on more complex models and real-world datasets.
For a fixed budget of gradient evaluations, we will compare the ELBO value obtained at the end of BBVI for the three variational families.
Recall that, in SGD, the stepsize \(\gamma\) performs a trade-off between the accuracy of the final solution versus the convergence speed.
We will demonstrate that, as the theory predicts, \textsf{structured} provide a more favorable trade-off compared to full-rank and mean-field, especially in the large data regime (large \(n\)).

\vspace{-.5ex}
\subsubsection{Setup}
\vspace{-.5ex}
\paragraph{Stochastic Optimization}
While our theoretical results use proximal SGD, we use regular SGD without projection for our experiments.
In practice, as long as the initialization is sensible, proximal SGD does not provide additional benefits both in practice~\citep{kim_convergence_2023} and in theory~\citep{domke_provable_2023}, and simple SGD is the most common method for performing BBVI in practice.
Furthermore, we will consider only fixed stepsizes.
In practice, step-decay stepsize schedules~\citep{goffin_convergence_1977} are believed to be superior to polynomial decay schedules or naive fixed stepsizes.
However, the theoretical evidence for this is not yet clear.
(See \citealt{wang_convergence_2021,ge_step_2019} for related investigations.)
Also, for the sake of experimentation, step-decay stepsizes are hard to control due to the additional number of configurations, such as the number of stages, amount of decay per stage, and such.


\vspace{-1.5ex}
\paragraph{Implementation}
We implemented our experiments using the Turing ecosystem~\citep{ge_turing_2018} in the Julia language~\citep{bezanson_julia_2017}.
The structured covariances were implemented using the compressed sparse column (CSC) sparse matrix interface provided by the \texttt{CUDA.jl} library~\citep{besard_effective_2019}, while the sparse derivatives were implemented using the \texttt{Zygote.jl} framework~\citep{innes_don_2018}.
We use \(8\) Monte Carlo samples and the Adam optimizer~\citep{kingma_adam_2015} for all problems, while the reported ELBOs are estimated using 1024 Monte Carlo samples every 100 iterations.
The variational families are Gaussian such that \(\varphi = \mathcal{N}\left(0,1\right)\).

\vspace{-1.5ex}
\paragraph{Models and Dataset}
We use three different Hierarchical Bayesian models, \textsf{volatility}, \textsf{rpoisson}, and \textsf{irt}, which have local variables.
\textsf{rpoisson} is a robust generalized linear regression model known as the Poisson-log-normal model~\citep[\S 4.2.4]{cameron_regression_2013}.
We treat the regression coefficients and the hyperparameters as global variables, while the individual-level response, which is modeled to include log-normal noise, is treated as the local variable.
\textsf{volatility} is a multivariate stochastic volatility model~\citep{chib_multivariate_2009,naesseth_variational_2018}.
We treat the hyperparameters as the global variables, while the latent volatilities are treated as the local variables.
Lastly, \textsf{irt} is a two-parameter logistic item response theory (IRT) model~\citep{lord_statistical_2008,wu_variational_2020}.
We treat the ``ability'' of each student as local variables, while the hyperparameters and item-level variables are treated as global variables.
The full description of these models can be found in \cref{section:models}.
To evaluate the effect of dataset size, we use subsets of the full datasets, as shown in \cref{table:datasets}.
For the initial point, we use \(q_{\vlambda_0} = \mathcal{N}\left(\boldupright{0}, 10^{-2} \boldupright{I}\right) \) for all experiments.

\subsubsection{Results}


The effect of the stepsize \(\gamma\) are shown in \cref{fig:irt_stepsize}.
Additional results can be found in \cref{section:additional_results}.
Notice that full-rank is much more sensitive to the stepsize compared to \textsf{structured} and \textsf{mean-field}.
This exactly aligns with the theory: larger gradient variance means that the size of the stationary region of SGD is wider.
Therefore, the quality of the solution is much more sensitive to the stepsize.
Since \textsf{full-rank} has the largest amount of gradient variance, followed by \textsf{structured}, and then \textsf{mean-field}, the sensitivity to the stepsize follows the same ordering.

Overall, our experimental results demonstrate that, for fixed stepsize SGD, the complexity of the variational family trades optimization speed for the statistical accuracy of the variational approximation.
This re-affirms the results of \citet{bhatia_statistical_2022} on a more realistic setup.

%% file: figures/tikz_synthetic_fullrank.tex
\begin{tikzpicture}
\pgfkeys{/pgf/number format/.cd,
sci,
}
\begin{axis}[
    tuftelike,
    width=0.34\textwidth,
    height=0.35\textwidth,
    xmode=log,
    log basis x={10},
    scaled y ticks = false,
    ymax   = 40000.0,
    ymin   = 0.0,
    xlabel = {\footnotesize{}Stepsize \(\gamma\)},
    ylabel = {\footnotesize{}\(T\) for \(\epsilon\) accuracy},
    xmin={1e-5},
    xmax={0.1},
    max space between ticks=20
    scaled x ticks=false,
    axis line style = thick,
    axis x line shift=2ex,
    axis y line shift=2ex,
    every tick/.style={black,thick},
    tick label style={font=\small},  
    legend style={
        fill opacity=0.8,
        text opacity=1.0,
        draw=none,
        anchor=north east,
        at={(1.05,1.05)},
    },
    legend cell align=left,
]

     \draw[red,dashed] (axis cs:0,406) -- (axis cs:1,406);
	\addplot[color=ibmblue3, very thick] coordinates {
 (1.0e-6, 59999.0)
(1.2648552168552957e-6, 59999.0)
(1.5998587196060574e-6, 59999.0)
(2.0235896477251553e-6, 59999.0)
(2.559547922699533e-6, 59999.0)
(3.2374575428176463e-6, 59999.0)
(4.094915062380427e-6, 59999.0)
(5.179474679231212e-6, 59999.0)
(6.551285568595509e-6, 59999.0)
(8.286427728546843e-6, 54091.333333333336)
(1.0481131341546853e-5, 42806.0)
(1.3257113655901107e-5, 33872.0)
(1.67683293681101e-5, 26811.0)
(2.1209508879201923e-5, 21230.333333333332)
(2.682695795279727e-5, 16820.0)
(3.39322177189533e-5, 13331.333333333334)
(4.291934260128778e-5, 10571.666666666666)
(5.4286754393238594e-5, 8397.0)
(6.866488450042999e-5, 6678.333333333333)
(8.68511373751352e-5, 5318.333333333333)
(0.00010985411419875583, 4243.0)
(0.0001389495494373137, 3397.3333333333335)
(0.0001757510624854793, 2731.6666666666665)
(0.00022229964825261955, 2216.0)
(0.00028117686979742306, 1829.3333333333333)
(0.00035564803062231287, 1589.0)
(0.0004498432668969444, 59999.0)
(0.0005689866029018299, 59999.0)
(0.0007196856730011522, 59999.0)
(0.0009102981779915216, 59999.0)
(0.001151395399326448, 59999.0)
(0.0014563484775012444, 59999.0)
(0.0018420699693267163, 59999.0)
(0.0023299518105153716, 59999.0)
(0.0029470517025518097, 59999.0)
(0.0037275937203149418, 59999.0)
(0.004714866363457394, 59999.0)
(0.005963623316594642, 59999.0)
(0.007543120063354623, 59999.0)
(0.009540954763499945, 59999.0)
(0.012067926406393288, 59999.0)
(0.015264179671752335, 59999.0)
(0.019306977288832496, 59999.0)
(0.024420530945486497, 59999.0)
(0.03088843596477485, 59999.0)
(0.03906939937054621, 59999.0)
(0.04941713361323838, 59999.0)
(0.06250551925273976, 59999.0)
(0.07906043210907701, 59999.0)
(0.1, 59999.0)
    };
    \addlegendentry{\footnotesize\(n=100\)}
    
	\addplot[color=ibmblue2, very thick] coordinates {
 (1.0e-6, 59999.0)
(1.2648552168552957e-6, 59999.0)
(1.5998587196060574e-6, 59999.0)
(2.0235896477251553e-6, 59999.0)
(2.559547922699533e-6, 59999.0)
(3.2374575428176463e-6, 59999.0)
(4.094915062380427e-6, 59999.0)
(5.179474679231212e-6, 59999.0)
(6.551285568595509e-6, 59999.0)
(8.286427728546843e-6, 58620.666666666664)
(1.0481131341546853e-5, 46466.0)
(1.3257113655901107e-5, 36862.666666666664)
(1.67683293681101e-5, 29253.666666666668)
(2.1209508879201923e-5, 23262.333333333332)
(2.682695795279727e-5, 18537.0)
(3.39322177189533e-5, 14804.0)
(4.291934260128778e-5, 11872.0)
(5.4286754393238594e-5, 9579.333333333334)
(6.866488450042999e-5, 7827.333333333333)
(8.68511373751352e-5, 6587.0)
(0.00010985411419875583, 59999.0)
(0.0001389495494373137, 59999.0)
(0.0001757510624854793, 59999.0)
(0.00022229964825261955, 59999.0)
(0.00028117686979742306, 59999.0)
(0.00035564803062231287, 59999.0)
(0.0004498432668969444, 59999.0)
(0.0005689866029018299, 59999.0)
(0.0007196856730011522, 59999.0)
(0.0009102981779915216, 59999.0)
(0.001151395399326448, 59999.0)
(0.0014563484775012444, 59999.0)
(0.0018420699693267163, 59999.0)
(0.0023299518105153716, 59999.0)
(0.0029470517025518097, 59999.0)
(0.0037275937203149418, 59999.0)
(0.004714866363457394, 59999.0)
(0.005963623316594642, 59999.0)
(0.007543120063354623, 59999.0)
(0.009540954763499945, 59999.0)
(0.012067926406393288, 59999.0)
(0.015264179671752335, 59999.0)
(0.019306977288832496, 59999.0)
(0.024420530945486497, 59999.0)
(0.03088843596477485, 59999.0)
(0.03906939937054621, 59999.0)
(0.04941713361323838, 59999.0)
(0.06250551925273976, 59999.0)
(0.07906043210907701, 59999.0)
(0.1, 59999.0)
    };
    \addlegendentry{\footnotesize\(n=200\)}
    
	\addplot[color=ibmblue1, very thick] coordinates {
 (1.0e-6, 59999.0)
(1.2648552168552957e-6, 59999.0)
(1.5998587196060574e-6, 59999.0)
(2.0235896477251553e-6, 59999.0)
(2.559547922699533e-6, 59999.0)
(3.2374575428176463e-6, 59999.0)
(4.094915062380427e-6, 59999.0)
(5.179474679231212e-6, 59999.0)
(6.551285568595509e-6, 59999.0)
(8.286427728546843e-6, 59999.0)
(1.0481131341546853e-5, 49113.333333333336)
(1.3257113655901107e-5, 39146.666666666664)
(1.67683293681101e-5, 31280.666666666668)
(2.1209508879201923e-5, 25114.333333333332)
(2.682695795279727e-5, 20317.666666666668)
(3.39322177189533e-5, 16691.666666666668)
(4.291934260128778e-5, 14398.333333333334)
(5.4286754393238594e-5, 59999.0)
(6.866488450042999e-5, 59999.0)
(8.68511373751352e-5, 59999.0)
(0.00010985411419875583, 59999.0)
(0.0001389495494373137, 59999.0)
(0.0001757510624854793, 59999.0)
(0.00022229964825261955, 59999.0)
(0.00028117686979742306, 59999.0)
(0.00035564803062231287, 59999.0)
(0.0004498432668969444, 59999.0)
(0.0005689866029018299, 59999.0)
(0.0007196856730011522, 59999.0)
(0.0009102981779915216, 59999.0)
(0.001151395399326448, 59999.0)
(0.0014563484775012444, 59999.0)
(0.0018420699693267163, 59999.0)
(0.0023299518105153716, 59999.0)
(0.0029470517025518097, 59999.0)
(0.0037275937203149418, 59999.0)
(0.004714866363457394, 59999.0)
(0.005963623316594642, 59999.0)
(0.007543120063354623, 59999.0)
(0.009540954763499945, 59999.0)
(0.012067926406393288, 59999.0)
(0.015264179671752335, 59999.0)
(0.019306977288832496, 59999.0)
(0.024420530945486497, 59999.0)
(0.03088843596477485, 59999.0)
(0.03906939937054621, 59999.0)
(0.04941713361323838, 59999.0)
(0.06250551925273976, 59999.0)
(0.07906043210907701, 59999.0)
(0.1, 59999.0)
    };
    \addlegendentry{\footnotesize\(n=300\)}

\end{axis}
\end{tikzpicture}

%% file: figures/tikz_synthetic_structured.tex
\begin{tikzpicture}
\begin{axis}[
    tuftelike,
    width=0.34\textwidth,
    height=0.35\textwidth,
    xmode=log,
    log basis x={10},
    ymax   = 40000.0,
    ymin   = 0.0,
    xlabel = {\footnotesize{}Stepsize \(\gamma\)},
    ylabel = {},
    xmin={1e-5},
    xmax={0.1},
    axis y line=none,
    max space between ticks=20,
    scaled x ticks=false,
    axis line style = thick,
    axis x line shift=1.5ex,
    axis y line shift=1.5ex,
    every tick/.style={black,thick},
    tick label style={font=\small},  
    legend style={
        fill opacity=0.8,
        text opacity=1.0,
        draw=none,
        anchor=north east,
        at={(1.05,1.05)},
    },
    legend cell align=left,
]

	\addplot[color=ibmred3, very thick] coordinates {
(1.0e-6, 59999.0)
(1.2648552168552957e-6, 59999.0)
(1.5998587196060574e-6, 59999.0)
(2.0235896477251553e-6, 59999.0)
(2.559547922699533e-6, 59999.0)
(3.2374575428176463e-6, 59999.0)
(4.094915062380427e-6, 59999.0)
(5.179474679231212e-6, 59999.0)
(6.551285568595509e-6, 59999.0)
(8.286427728546843e-6, 53947.666666666664)
(1.0481131341546853e-5, 42658.333333333336)
(1.3257113655901107e-5, 33723.333333333336)
(1.67683293681101e-5, 26658.0)
(2.1209508879201923e-5, 21077.0)
(2.682695795279727e-5, 16667.333333333332)
(3.39322177189533e-5, 13179.333333333334)
(4.291934260128778e-5, 10418.666666666666)
(5.4286754393238594e-5, 8238.333333333334)
(6.866488450042999e-5, 6517.666666666667)
(8.68511373751352e-5, 5152.333333333333)
(0.00010985411419875583, 4072.6666666666665)
(0.0001389495494373137, 3219.6666666666665)
(0.0001757510624854793, 2545.3333333333335)
(0.00022229964825261955, 2014.0)
(0.00028117686979742306, 1592.3333333333333)
(0.00035564803062231287, 1259.0)
(0.0004498432668969444, 997.3333333333334)
(0.0005689866029018299, 789.0)
(0.0007196856730011522, 624.3333333333334)
(0.0009102981779915216, 495.6666666666667)
(0.001151395399326448, 393.3333333333333)
(0.0014563484775012444, 312.0)
(0.0018420699693267163, 248.66666666666666)
(0.0023299518105153716, 198.66666666666666)
(0.0029470517025518097, 159.33333333333334)
(0.0037275937203149418, 127.66666666666667)
(0.004714866363457394, 103.33333333333333)
(0.005963623316594642, 85.0)
(0.007543120063354623, 82.33333333333333)
(0.009540954763499945, 59999.0)
(0.012067926406393288, 59999.0)
(0.015264179671752335, 59999.0)
(0.019306977288832496, 59999.0)
(0.024420530945486497, 59999.0)
(0.03088843596477485, 59999.0)
(0.03906939937054621, 59999.0)
(0.04941713361323838, 59999.0)
(0.06250551925273976, 59999.0)
(0.07906043210907701, 59999.0)
(0.1, 59999.0)
    };
    \addlegendentry{\footnotesize\(n=100\)}
    
	\addplot[color=ibmred2, very thick] coordinates {
(1.0e-6, 59999.0)
(1.2648552168552957e-6, 59999.0)
(1.5998587196060574e-6, 59999.0)
(2.0235896477251553e-6, 59999.0)
(2.559547922699533e-6, 59999.0)
(3.2374575428176463e-6, 59999.0)
(4.094915062380427e-6, 59999.0)
(5.179474679231212e-6, 59999.0)
(6.551285568595509e-6, 59999.0)
(8.286427728546843e-6, 58084.0)
(1.0481131341546853e-5, 45930.333333333336)
(1.3257113655901107e-5, 36314.333333333336)
(1.67683293681101e-5, 28704.0)
(2.1209508879201923e-5, 22695.333333333332)
(2.682695795279727e-5, 17950.333333333332)
(3.39322177189533e-5, 14193.666666666666)
(4.291934260128778e-5, 11225.333333333334)
(5.4286754393238594e-5, 8875.333333333334)
(6.866488450042999e-5, 7020.666666666667)
(8.68511373751352e-5, 5554.0)
(0.00010985411419875583, 4394.0)
(0.0001389495494373137, 3473.0)
(0.0001757510624854793, 2747.6666666666665)
(0.00022229964825261955, 2175.3333333333335)
(0.00028117686979742306, 1722.0)
(0.00035564803062231287, 1363.0)
(0.0004498432668969444, 1081.3333333333333)
(0.0005689866029018299, 857.0)
(0.0007196856730011522, 680.6666666666666)
(0.0009102981779915216, 540.3333333333334)
(0.001151395399326448, 430.3333333333333)
(0.0014563484775012444, 344.0)
(0.0018420699693267163, 276.0)
(0.0023299518105153716, 222.66666666666666)
(0.0029470517025518097, 183.66666666666666)
(0.0037275937203149418, 162.0)
(0.004714866363457394, 59999.0)
(0.005963623316594642, 59999.0)
(0.007543120063354623, 59999.0)
(0.009540954763499945, 59999.0)
(0.012067926406393288, 59999.0)
(0.015264179671752335, 59999.0)
(0.019306977288832496, 59999.0)
(0.024420530945486497, 59999.0)
(0.03088843596477485, 59999.0)
(0.03906939937054621, 59999.0)
(0.04941713361323838, 59999.0)
(0.06250551925273976, 59999.0)
(0.07906043210907701, 59999.0)
(0.1, 59999.0)
    };
    \addlegendentry{\footnotesize\(n=200\)}
    
	\addplot[color=ibmred1, very thick] coordinates {
 (1.0e-6, 59999.0)
(1.2648552168552957e-6, 59999.0)
(1.5998587196060574e-6, 59999.0)
(2.0235896477251553e-6, 59999.0)
(2.559547922699533e-6, 59999.0)
(3.2374575428176463e-6, 59999.0)
(4.094915062380427e-6, 59999.0)
(5.179474679231212e-6, 59999.0)
(6.551285568595509e-6, 59999.0)
(8.286427728546843e-6, 59999.0)
(1.0481131341546853e-5, 47856.333333333336)
(1.3257113655901107e-5, 37837.333333333336)
(1.67683293681101e-5, 29909.0)
(2.1209508879201923e-5, 23649.666666666668)
(2.682695795279727e-5, 18704.666666666668)
(3.39322177189533e-5, 14790.666666666666)
(4.291934260128778e-5, 11699.666666666666)
(5.4286754393238594e-5, 9248.0)
(6.866488450042999e-5, 7319.0)
(8.68511373751352e-5, 5791.0)
(0.00010985411419875583, 4579.0)
(0.0001389495494373137, 3623.3333333333335)
(0.0001757510624854793, 2867.3333333333335)
(0.00022229964825261955, 2270.3333333333335)
(0.00028117686979742306, 1799.3333333333333)
(0.00035564803062231287, 1423.6666666666667)
(0.0004498432668969444, 1130.0)
(0.0005689866029018299, 898.0)
(0.0007196856730011522, 715.3333333333334)
(0.0009102981779915216, 570.6666666666666)
(0.001151395399326448, 458.0)
(0.0014563484775012444, 369.6666666666667)
(0.0018420699693267163, 301.3333333333333)
(0.0023299518105153716, 255.66666666666666)
(0.0029470517025518097, 59999.0)
(0.0037275937203149418, 59999.0)
(0.004714866363457394, 59999.0)
(0.005963623316594642, 59999.0)
(0.007543120063354623, 59999.0)
(0.009540954763499945, 59999.0)
(0.012067926406393288, 59999.0)
(0.015264179671752335, 59999.0)
(0.019306977288832496, 59999.0)
(0.024420530945486497, 59999.0)
(0.03088843596477485, 59999.0)
(0.03906939937054621, 59999.0)
(0.04941713361323838, 59999.0)
(0.06250551925273976, 59999.0)
(0.07906043210907701, 59999.0)
(0.1, 59999.0)
     };
   \addlegendentry{\footnotesize\(n=300\)}
\end{axis}
\end{tikzpicture}

%% file: figures/tikz_synthetic_meanfield.tex
\begin{tikzpicture}
\begin{axis}[
    tuftelike,
    width=0.35\textwidth,
    height=0.35\textwidth,
    xmode=log,
    log basis x={10},
    ymax   = 40000.0,
    ymin   = 0.0,
    xlabel = {\footnotesize{}Stepsize \(\gamma\)},
    ylabel = {},
    axis y line=none,
    xmin={1e-5},
    xmax={0.1},
    max space between ticks=20,
    scaled x ticks=false,
    axis line style = thick,
    axis x line shift=1.5ex,
    axis y line shift=1.5ex,
    every tick/.style={black,thick},
    tick label style={font=\small},  
    legend style={
        draw=none,
        fill opacity=0.8,
        text opacity=1.0,
        anchor=north east,
        at={(1.05,1.05)},
    },
    legend cell align=left,
]

	\addplot[color=ibmpurple3, very thick] coordinates {
 (1.0e-6, 59999.0)
(1.2648552168552957e-6, 59999.0)
(1.5998587196060574e-6, 59999.0)
(2.0235896477251553e-6, 59999.0)
(2.559547922699533e-6, 59999.0)
(3.2374575428176463e-6, 59999.0)
(4.094915062380427e-6, 59999.0)
(5.179474679231212e-6, 59999.0)
(6.551285568595509e-6, 59999.0)
(8.286427728546843e-6, 53941.666666666664)
(1.0481131341546853e-5, 42651.0)
(1.3257113655901107e-5, 33716.666666666664)
(1.67683293681101e-5, 26651.333333333332)
(2.1209508879201923e-5, 21070.0)
(2.682695795279727e-5, 16660.0)
(3.39322177189533e-5, 13172.333333333334)
(4.291934260128778e-5, 10412.333333333334)
(5.4286754393238594e-5, 8232.0)
(6.866488450042999e-5, 6510.333333333333)
(8.68511373751352e-5, 5146.0)
(0.00010985411419875583, 4067.0)
(0.0001389495494373137, 3214.0)
(0.0001757510624854793, 2540.0)
(0.00022229964825261955, 2008.3333333333333)
(0.00028117686979742306, 1586.6666666666667)
(0.00035564803062231287, 1254.3333333333333)
(0.0004498432668969444, 992.0)
(0.0005689866029018299, 784.3333333333334)
(0.0007196856730011522, 619.0)
(0.0009102981779915216, 489.6666666666667)
(0.001151395399326448, 386.3333333333333)
(0.0014563484775012444, 305.3333333333333)
(0.0018420699693267163, 241.66666666666666)
(0.0023299518105153716, 191.0)
(0.0029470517025518097, 151.0)
(0.0037275937203149418, 119.0)
(0.004714866363457394, 94.0)
(0.005963623316594642, 74.0)
(0.007543120063354623, 58.0)
(0.009540954763499945, 46.0)
(0.012067926406393288, 36.0)
(0.015264179671752335, 29.0)
(0.019306977288832496, 23.0)
(0.024420530945486497, 18.333333333333332)
(0.03088843596477485, 19.333333333333332)
(0.03906939937054621, 59999.0)
(0.04941713361323838, 59999.0)
(0.06250551925273976, 59999.0)
(0.07906043210907701, 59999.0)
(0.1, 59999.0)
    };
    \addlegendentry{\footnotesize\(n=100\)}
    
	\addplot[color=ibmpurple2, very thick] coordinates {
(1.0e-6, 59999.0)
(1.2648552168552957e-6, 59999.0)
(1.5998587196060574e-6, 59999.0)
(2.0235896477251553e-6, 59999.0)
(2.559547922699533e-6, 59999.0)
(3.2374575428176463e-6, 59999.0)
(4.094915062380427e-6, 59999.0)
(5.179474679231212e-6, 59999.0)
(6.551285568595509e-6, 59999.0)
(8.286427728546843e-6, 58074.0)
(1.0481131341546853e-5, 45920.666666666664)
(1.3257113655901107e-5, 36304.666666666664)
(1.67683293681101e-5, 28694.333333333332)
(2.1209508879201923e-5, 22684.666666666668)
(2.682695795279727e-5, 17939.666666666668)
(3.39322177189533e-5, 14182.666666666666)
(4.291934260128778e-5, 11214.0)
(5.4286754393238594e-5, 8863.666666666666)
(6.866488450042999e-5, 7009.333333333333)
(8.68511373751352e-5, 5542.666666666667)
(0.00010985411419875583, 4382.333333333333)
(0.0001389495494373137, 3462.6666666666665)
(0.0001757510624854793, 2736.3333333333335)
(0.00022229964825261955, 2163.6666666666665)
(0.00028117686979742306, 1710.0)
(0.00035564803062231287, 1351.3333333333333)
(0.0004498432668969444, 1069.3333333333333)
(0.0005689866029018299, 845.3333333333334)
(0.0007196856730011522, 668.3333333333334)
(0.0009102981779915216, 528.3333333333334)
(0.001151395399326448, 418.3333333333333)
(0.0014563484775012444, 330.6666666666667)
(0.0018420699693267163, 261.6666666666667)
(0.0023299518105153716, 207.0)
(0.0029470517025518097, 163.66666666666666)
(0.0037275937203149418, 129.33333333333334)
(0.004714866363457394, 103.0)
(0.005963623316594642, 81.33333333333333)
(0.007543120063354623, 65.0)
(0.009540954763499945, 52.666666666666664)
(0.012067926406393288, 43.0)
(0.015264179671752335, 39.333333333333336)
(0.019306977288832496, 59999.0)
(0.024420530945486497, 59999.0)
(0.03088843596477485, 59999.0)
(0.03906939937054621, 59999.0)
(0.04941713361323838, 59999.0)
(0.06250551925273976, 59999.0)
(0.07906043210907701, 59999.0)
(0.1, 59999.0)
    };
    \addlegendentry{\footnotesize\(n=200\)}
    
	\addplot[color=ibmpurple1, very thick] coordinates {
 (1.0e-6, 59999.0)
(1.2648552168552957e-6, 59999.0)
(1.5998587196060574e-6, 59999.0)
(2.0235896477251553e-6, 59999.0)
(2.559547922699533e-6, 59999.0)
(3.2374575428176463e-6, 59999.0)
(4.094915062380427e-6, 59999.0)
(5.179474679231212e-6, 59999.0)
(6.551285568595509e-6, 59999.0)
(8.286427728546843e-6, 59999.0)
(1.0481131341546853e-5, 47841.666666666664)
(1.3257113655901107e-5, 37822.0)
(1.67683293681101e-5, 29894.333333333332)
(2.1209508879201923e-5, 23634.666666666668)
(2.682695795279727e-5, 18689.666666666668)
(3.39322177189533e-5, 14776.333333333334)
(4.291934260128778e-5, 11685.333333333334)
(5.4286754393238594e-5, 9233.666666666666)
(6.866488450042999e-5, 7304.333333333333)
(8.68511373751352e-5, 5776.666666666667)
(0.00010985411419875583, 4564.666666666667)
(0.0001389495494373137, 3607.6666666666665)
(0.0001757510624854793, 2853.0)
(0.00022229964825261955, 2255.0)
(0.00028117686979742306, 1782.3333333333333)
(0.00035564803062231287, 1409.0)
(0.0004498432668969444, 1114.3333333333333)
(0.0005689866029018299, 881.3333333333334)
(0.0007196856730011522, 698.0)
(0.0009102981779915216, 552.0)
(0.001151395399326448, 437.6666666666667)
(0.0014563484775012444, 346.6666666666667)
(0.0018420699693267163, 274.3333333333333)
(0.0023299518105153716, 217.66666666666666)
(0.0029470517025518097, 172.33333333333334)
(0.0037275937203149418, 137.33333333333334)
(0.004714866363457394, 109.33333333333333)
(0.005963623316594642, 87.66666666666667)
(0.007543120063354623, 70.66666666666667)
(0.009540954763499945, 60.666666666666664)
(0.012067926406393288, 13660.333333333334)
(0.015264179671752335, 59999.0)
(0.019306977288832496, 59999.0)
(0.024420530945486497, 59999.0)
(0.03088843596477485, 59999.0)
(0.03906939937054621, 59999.0)
(0.04941713361323838, 59999.0)
(0.06250551925273976, 59999.0)
(0.07906043210907701, 59999.0)
(0.1, 59999.0)
    };
    \addlegendentry{\footnotesize\(n=300\)}

\end{axis}
\end{tikzpicture}

%% file: figures/tikz_scaling.tex
\begin{tikzpicture}
\begin{axis}[
    tuftelike,
    height =0.90\textwidth,
    width  =0.90\textwidth,
    xlabel = {\small{}\# of datapoints \(n\)},
    ylabel = {\small{}\(T\) for \(\epsilon\)-accuracy},
    xmin={100},
    xmax={500},
    ymin={0},
    ymax={10000},
    y tick label style={/pgf/number format/sci},
    scaled x ticks  =false,
    scaled y ticks  = false,
    axis line style = thick,
    axis x line shift=2ex,
    axis y line shift=2ex,
    every tick/.style={black,thick},
    tick label style={font=\small},  
    legend style={
        fill opacity=0.,
        text opacity=1.0,
         draw=none,
         anchor=north east,
         at={(1.1,0.6)},
    },
    legend cell align=left,
]
	\addplot[color=ibmred2, very thick] coordinates {
 (100, 82.33333333333333)
(150, 119.66666666666667)
(200, 162.0)
(250, 208.33333333333334)
(300, 255.66666666666666)
(350, 317.0)
(400, 341.3333333333333)
(450, 411.0)
(500, 437.0)
    };
    \addlegendentry{\footnotesize\textsf{structured}}
    
	\addplot[color=ibmpurple2, very thick] coordinates {
(100, 18.333333333333332)
(150, 28.333333333333332)
(200, 39.333333333333336)
(250, 49.666666666666664)
(300, 60.666666666666664)
(350, 76.0)
(400, 82.66666666666667)
(450, 98.0)
(500, 105.0)
    };
    \addlegendentry{\footnotesize\textsf{mean-field}}
    
	\addplot[color=ibmblue2, very thick] coordinates {
(100, 1589.0)
(150, 3555.6666666666665)
(200, 6587.0)
(250, 10600.0)
(300, 14398.333333333334)
(350, 19864.0)
(400, 27448.333333333332)
(450, 34703.0)
(500, 41314.0)
    };
    \addlegendentry{\footnotesize\textsf{full-rank}}
\end{axis}
\end{tikzpicture}

%% file: sections/table_problem_dimensions.tex
\begin{table*}[b]
\vspace{-1ex}
\caption{Models and Datasets used in the Experiments}\label{table:datasets}
\vspace{-1ex}
\begin{center}
\begin{tabular}{lrrrrrrr}
\toprule
\multicolumn{1}{c}{\multirow{2}{*}{\textbf{Problem}}}
& \multicolumn{4}{c}{\textbf{Dimensions}}
& \multicolumn{3}{c}{\textbf{\# of Variational Parameters} (\(p\))}
\\
\cmidrule(lr){2-5} \cmidrule(lr){6-8}
& \multicolumn{1}{c}{\(N\)} & \multicolumn{1}{c}{\(d_y\)} & \multicolumn{1}{c}{\(d_z\)} & \multicolumn{1}{c}{\(d\)} & \textsf{mean-field} & \textsf{structured} & \textsf{full-rank}
\\
\midrule
\textsf{rpoisson-small}    &  1,961 & \multirow{3}{*}{1} & \multirow{3}{*}{16}   &  1,977 &   3,954 & 35,450 & 1,957,230 \\
\textsf{rpoisson-middle}    &  3,922 &                    &                       &  3,938  &   7,876  & 70,748 & 7,759,829 \\
\textsf{rpoisson-large}   & 19,609 &                    &                       & 19,625  & 39,282  & 353,114 & %
\\
\midrule
\textsf{volatility-small}  &   262 & \multirow{3}{*}{6} & \multirow{3}{*}{33}   &  1,605 & 3,210  & 59,544 & 1,290,420 \\
\textsf{volatility-middle}  &   522 &                    &                       &  3,165 & 6,210 & 118,044 & 4,825,170 \\
\textsf{volatility-large} &  2,579 &                    &                       & 15,507 & 31,014 & 580,869 & 
\\
\midrule
\textsf{irt-small}         &  3,348 & \multirow{3}{*}{1} &  \multirow{3}{*}{193} &  3,541 & 7,082 & 671,774 & 6,274,652 \\
\textsf{irt-middle}          &  6,695 &                    &                       &  6,888 & 13,776 & 1,324,439 & 23,732,604 \\
\textsf{irt-large}          & 33,475 &                    &                       & 33,668 & 67,336 & 6,546,539 &  %
\\
\bottomrule
\end{tabular}
\end{center}
\vspace{-2ex}
\end{table*}

%% file: sections/section_discussions.tex
\section{Discussions}\label{section:conclusions}
\paragraph{Conclusions}
In this work, we have theoretically investigated the limitations of full-rank variational families in BBVI. 
Specifically, the dimensional scaling of the gradient variance of full-rank variational families.
This is particularly problematic for Bayesian models with local variables, implying that BBVI with full-rank variational families do not scale to larger datasets.
Fortunately, we have rigorously shown that variational families with structured scale matrices are able to improve this scaling issue.
We have evaluated this theoretical insight on large-scale Hierarchical Bayesian models and have confirmed that practice agrees with the theory.
Furthermore, our analysis provides a precise quantitative analysis of how certain scale matrix structures would improve the computational complexity of BBVI.

\vspace{-1.5ex}
\paragraph{Related Works}
Structured variational approximations have a long history since their use in hidden Markov models~\citep{saul_exploiting_1995,ghahramani_factorial_1997}, CAVI~\citep{hoffman_stochastic_2015,mimno_sparse_2012,rohde_semiparametric_2016}, fixed-form variational Bayes~\citep{salimans_fixedform_2013}, BBVI~\citep{ong_gaussian_2018,quiroz_gaussian_2023,ranganath_hierarchical_2016,tan_use_2021,tan_conditionally_2020}, natural gradient VI~\citep{lin_fast_2019}, and now modern approaches such as amortized VI~\citep{archer_black_2015,johnson_composing_2016,webb_faithful_2018,sonderby_ladder_2016,agrawal_amortized_2021,maaloe_auxiliary_2016,gao_linear_2016,yu_structured_2022} and normalizing flows~\citep{ambrogioni_automatic_2021,ambrogioni_automatic_2021a}.
Our proposed covariance structure is identical to that of \citet{tan_use_2021,tan_conditionally_2020}.
But, they did not consider the computational complexity of this parameterization, for which we rigorously prove its scalability with respect to \(N\).


\vspace{-1.5ex}
\paragraph{Limitations of the Current Theory}
In the \(\mathcal{O}(N^3)\) sample complexity we obtained (after making the implicit \(N\) dependence in \(\sum_{n} \kappa_n = \mathcal{O}(N)\) explicit), an excess factor of \(N\) comes from the fact that the gradient variance analysis strategy of \citet{domke_provable_2019} results in an \(\mathcal{O}\left(L^2\right)\) dependence on the smoothness constant \(L = \mathcal{O}\left(N\right)\).
We conjecture that an \(\mathcal{O}(L)\) dependence is realistic, which would match that of empirical risk minimization and also imply a dataset size dependence of \(\mathcal{O}(N^2)\) for hierarchical models.
This also matches the following intuition: computing the gradient costs \(\Theta(N)\), while, by posterior contraction, the smoothness naturally scales as \(\mathcal{O}(N)\).
Therefore, we conjecture a \(\mathcal{O}(N^2)\) complexity.
Note that this is also closely related to the fact that the complexity of BBVI currently has a \(\mathcal{O}\left(\kappa^2\right)\) dependence on \(\kappa\), which is believed to be loose.


%% file: sections/section_computational_resources.tex
\begin{table}[H]
  \centering
  \begin{threeparttable}
  \caption{Computational Resources}\label{table:resources}
  \begin{tabular}{ll}
    \toprule
    \multicolumn{1}{c}{\textbf{Type}}
    & \multicolumn{1}{c}{\textbf{Model and Specifications}}
    \\ \midrule
    System Topology & 1 socket with 8 physical cores \\
    Processor       & 1 Intel i9-11900F, 2.5 GHz (maximum 5.2 GHz) per socket \\
    Cache           & 80 KB L1, 512 KB L2, and 16 MB L3 \\
    Memory          & 64 GiB RAM \\
    Accelerator     & 1 NVIDIA GeForce RTX 3090 per node, 1.7 GHZ, 24GiB RAM 
    \\ \bottomrule
  \end{tabular}
  \end{threeparttable}
\end{table}

All of the experiments took approximately 1 week to run.

%% file: sections/section_definitions.tex
\subsection{Definitions}\label{section:definitions}
\vspace{2ex}

\begin{definition*}[\textbf{\(L\)-Smoothness}]
  A function \(f : \mathcal{X} \to \mathbb{R}\) is \(L\)-smooth if it satisfies
  \[
    \norm{ \nabla f\left(\vx\right) - \nabla f\left(\vy\right) }_2 \leq L \norm{ \vx - \vy }_2
  \]
  for all \(\vx, \vy \in \mathcal{X}\) and some \(L > 0\).
\end{definition*}
\vspace{1ex}

\begin{remark}
    Equivalently, we say a function \(f\) is \(L\)-log-smooth if \(\log f\) is \(L\)-smooth.
\end{remark}
\vspace{1ex}

\begin{definition*}[\textbf{\(\mu\)-Strong Convexity}]
  A function \(f : \mathcal{X} \to \mathbb{R}\) is \(\mu\)-strongly convex if it satisfies
  \[
     \frac{\mu}{2} \norm{\vx - \vy}_2^2  \leq f\left(\vy\right) - f\left(\vx\right) - \inner{ \nabla f\left(\vx\right) }{ \vy - \vx }
  \]
  for all \(\vx, \vy \in \mathcal{X}\) and some \(\mu > 0\).
\end{definition*}

\vspace{1ex}
\begin{remark}
    Equivalently, we say a function \(f\) is only convex if it satisfies the strong convexity inequality with \(\mu = 0\).
\end{remark}
\vspace{1ex}

\begin{remark}[\textbf{Log-Concave Measures}]
  For a probability measure \(\Pi\), we say it is \(\mu\)-strongly log-concave if, in a \(d\)-dimensional Euclidean measurable space \((\mathbb{R}^d, \mathcal{B}\left(\mathbb{R}^d\right), \mathbb{P})\), where \(\mathcal{B}\left(\mathbb{R}^d\right)\) is the \(\sigma\)-algebra of Borel-measurable subsets of \(\mathbb{R}^d\) and \(\mathbb{P}\) is the Lebesgue measure, its log probability density function \(x \mapsto -\log \pi\left(x\right)\) is \(\mu\)-strongly convex.
\end{remark}

\begin{definition*}[\textbf{Bregman Divergence}]
  Let \(\mathcal{X}\) be a convex set.
  Then, we define the function 
  \[
    \mathrm{D}_{\phi}\left(\vx, \vx'\right) \triangleq \phi\left(\vx\right) - \phi\left(\vx'\right) - \inner{\nabla \phi\left(\vx'\right)}{\vx - \vx'}
  \]
  to be the Bregman divergence generated by some continuously differentiable function \(\phi : \mathcal{X} \to \mathbb{R}\) convex on \(\mathcal{X}\).
\end{definition*}





\begin{definition}[\textbf{Subspace Identity Matrix of \(\ell_n\)}]\label{def:n_identity}
    We define \(\boldupright{I}_n \in \mathbb{R}^{d \times d}\), where the entries are set as
    \[
      \left[\,\boldupright{I}_n\,\right]_{ij} = \begin{cases}
       \;  1 & \text{if \(i = j\) and the \(i\)th element of \(\vz\) is used by \(\ell_n\)} \\
       \;  0 & \text{otherwise},
      \end{cases}
    \]
    where \([\cdot]_{ij}\) denotes the \(i\)th row and \(j\)th column, which is a variation of the standard identity matrix. 
\end{definition}

\cref{def:n_identity} can also be understood as a ``masking matrix.''
That is, \(\boldupright{I}_n \vz\) effectively ``selects'' the entries of \(\vz\) used by \(\ell_n\) by setting the remaining entries to \(0\).

%% file: theorems/thm_auxiliary.tex
\begin{theoremEnd}[all end, category=auxiliary]{lemma}
\label{thm:u_identities}
  Let 
  \(
    \rvvu = \left(\rvu_1, \rvu_2, \ldots, \rvu_d\right)
  \)
  be a \(d\)-dimensional vector-valued random variable satisfying \cref{assumption:symmetric_standard}.
  Then,
  \[
        \text{1. \(\mathbb{E}\rvvu \rvvu^{\top} = \boldupright{I}\)}
        \qquad\text{and}\qquad
        \text{2. \(\mathbb{E} \rvu_i^2 \, \rvvu = \boldupright{0}\)}
  \]
  for any \(i = 1, \ldots, d\).
\end{theoremEnd}
\begin{proofEnd}
    The first identity is derived by \citet[Lemma 9]{domke_provable_2019}.
    \begin{align*}
      \mathbb{E} \rvu_i^2 \, \rvvu
      =
      \mathbb{E} {\begin{bmatrix}
          \rvu_1 & \ldots & \rvu_{i-1} & \rvu_i^3 & \rvu_{i+1} & \ldots & \rvu_d
      \end{bmatrix}}^{\top}
      =
       {\begin{bmatrix}
          \mathbb{E}\rvu_1 & \ldots & \mathbb{E} \rvu_{i-1} & \mathbb{E} \rvu_i^3 & \mathbb{E} \rvu_{i+1} & \ldots & \mathbb{E}\rvu_d
      \end{bmatrix}}^{\top}
      =
      \boldupright{0},
    \end{align*}
    where the last equality follows from \cref{assumption:symmetric_standard}.
\end{proofEnd}

\begin{theoremEnd}[all end, category=auxiliary]{lemma}\label{thm:trace_estimator}
Let \(\mA \in \mathbb{R}^{d \times d}\) be some matrix and \(\rvvu\) be a \(d\)-dimensional vector-valued random variable satisfying \cref{assumption:symmetric_standard}.
Then, 
\begin{align*}
  \mathbb{E}\norm{\mA \rvvu}_{2}^2
  =
  \norm{\mA}_{\mathrm{F}}^2.
\end{align*}
\end{theoremEnd}
\begin{proofEnd}
  \begin{alignat*}{2}
    \mathbb{E}\norm{\mA \rvvu}_{2}^2
    &=
    \mathbb{E} \, \mathrm{tr} \, \rvvu^{\top} \mA^{\top} \mA \rvvu
    &&\qquad\text{(quadratic form is equal to its trace)}
    \\
    &=
    \mathrm{tr} \, \mA^{\top} \mA \mathbb{E} \rvvu \rvvu^{\top} 
    &&\qquad\text{(cyclic property of trace)}
    \\
    &=
    \mathrm{tr} \, \mA^{\top} \mA
    &&\qquad\text{(\cref{assumption:symmetric_standard})}
    \\
    &=
    \norm{\mA}_{\mathrm{F}}^2     
    &&\qquad\text{(definition of Frobenius norm)}.
  \end{alignat*}
\end{proofEnd}

\begin{theoremEnd}[all end, category=auxiliary]{lemma}
\label{thm:variance_of_sum_bound}
Let \(\rvvx_1, \ldots, \rvvx_N\) be vector-variate random variables.
Then, the variance of the sum is upper-bounded as
\[
   \mathrm{tr}\V{ \sum_{i=1}^N \rvvx_i }
   \leq
   N \sum_{i=1}^N  \mathrm{tr}\mathbb{V} \, \rvvx_i.
\]
\end{theoremEnd}
\begin{proofEnd}
\begin{alignat*}{2}
    \mathrm{tr}\V{ \sum_{i=1}^N \rvvx_i }
    &=
    \mathbb{E} \norm{ \sum_{i=1}^N \left( \rvvx_i - \mathbb{E} \rvvx_i \right) }_2^2
    =
    \mathbb{E} \sum_{i=1}^N \sum_{j=1}^N {\left( \rvvx_i - \mathbb{E} \rvvx_i \right)}^{\top} \left( \rvvx_j - \mathbb{E} \rvvx_j \right)
    \\
    &\leq
    \sum_{i=1}^N \sum_{j=1}^N 
    \frac{1}{2}
    \left(
    \mathbb{E} \norm{ \rvvx_i - \mathbb{E} \rvvx_i}_2^2
    +
    \mathbb{E} {\lVert \rvvx_j - \mathbb{E} \rvvx_j \rVert}_2^2
    \right)
    &&\quad\text{(Young's inequality)}
    \\
    &=
    \frac{1}{2}
    \sum_{j=1}^N \sum_{i=1}^N
    \mathbb{E} \norm{ \rvvx_i - \mathbb{E} \rvvx_i}_2^2
    +
    \frac{1}{2}
    \sum_{i=1}^N \sum_{j=1}^N 
    \mathbb{E} {\lVert \rvvx_j - \mathbb{E} \rvvx_j \rVert}_2^2
    &&\quad\text{(Distributing the sums)}
    \\
    &=
    \frac{N}{2}
    \sum_{i=1}^N
    \mathbb{E} \norm{ \rvvx_i - \mathbb{E} \rvvx_i}_2^2
    +
    \frac{N}{2}
    \sum_{j=1}^N 
    \mathbb{E} {\lVert \rvvx_j - \mathbb{E} \rvvx_j \rVert}_2^2
    &&\quad\text{(Solving the sums)}
    \\
    &=
    N
    \sum_{i=1}^N
    \mathbb{E} \norm{ \rvvx_i - \mathbb{E} \rvvx_i}_2^2
    \\
    &=
    N
    \sum_{i=1}^N
    \mathrm{tr}\mathbb{V} \, \rvvx_i.
\end{alignat*}
\end{proofEnd}

\begin{theoremEnd}[all end, category=auxiliary]{lemma}
\label{thm:variance_bound_squared_difference}
Let \(\rvvx\) be some \(d\)-dimensional vector-variate random variables.
Then, the variance is upper-bounded as
\[
    \mathrm{tr}\mathbb{V}{\rvvx}
    \leq
    \mathbb{E} \norm{\rvvx - \vy}_2^2
\]
for any vector \(\vy \in \mathbb{R}^d\).
\end{theoremEnd}
\begin{proofEnd}
\begin{alignat*}{2}
    \mathrm{tr}\mathbb{V}{\rvvx}
    &=
    \mathbb{E}\norm{\rvvx - \mathbb{E}\rvvx}_2^2
    \\
    &=
    \mathbb{E}\norm{\rvvx - \vy + \vy - \mathbb{E}\rvvx}_2^2
    \\    
    &=
    \mathbb{E}\norm{\rvvx - \vy}_2^2
    +2 
    \mathbb{E}\inner{\rvvx - \vy}{\vy - \mathbb{E}\rvvx}
    + 
    \norm{\vy - \mathbb{E}\rvvx}_2^2
    &&\quad\text{(expanding the quadratic)}
    \\    
    &=
    \mathbb{E}\norm{\rvvx - \vy}_2^2
    -2 
    \inner{\vy - \mathbb{E}\rvvx}{\vy - \mathbb{E}\rvvx}
    + 
    \norm{\vy - \mathbb{E}\rvvx}_2^2
    &&\quad\text{(linearity of expectation)}
    \\    
    &=
    \mathbb{E}\norm{\rvvx - \vy}_2^2
    -2 
    \norm{\vy - \mathbb{E}\rvvx}_2^2
    + 
    \norm{\vy - \mathbb{E}\rvvx}2^2
    \\    
    &=
    \mathbb{E}\norm{\rvvx - \vy}_2^2
    - \norm{\vy - \mathbb{E}\rvvx}2^2
    \\    
    &\leq
    \mathbb{E}\norm{\rvvx - \vy}_2^2.
\end{alignat*}
\end{proofEnd}

%% file: theorems/thm_matrix_trace_holder.tex
\begin{lemma}[\citealp{rammus_upper_2021}]\label{thm:matrix_trace_pd_bound}
    Consider \(\mA,\mB \in \mathbb{R}^{d \times d}\), where \(\mB\) is positive semidefinite such that \(\mB \succ 0\).
    Then,
    \[
      \mathrm{tr}\left(\mA \mB\right)
      \leq
      \norm{\mA}_{2,2} \mathrm{tr}\left(\mB\right).
    \]
\end{lemma}
\begin{proof}
  We restate the proof of \citet{rammus_upper_2021} for completeness.

  Consider an inner product space of matrices, where, for \(\mX,\mY \in \mathbb{C}^{d \times d}\), its inner product is defined as \(\inner{\mX}{\mY} = \mathrm{tr}\left(\mX^* \mY\right)\).
  Here, \(\mX^*\) is the conjugate transpose of \(\mX\).
  This generates a \(p\)-norm as
  \[ 
    \norm{\mX}_p \triangleq {\left(\sum_{i} {\sigma_{i}\left(\mX\right)}^{p}\right)}^{1/p},
  \] 
  where \(\sigma_{i}\left(\mX\right)\) is the \(i\)th singular value of \(\mX\).
  This norm is known as the Schatten norm, where its limiting case \(p \to \infty\) is, in fact, the \(\ell_2\) operator norm.
  
  We can now apply H\"older's inequality 
  \[
    \abs{\inner{\mA}{\mB}} \leq \norm{\mA}_{p} \norm{\mB}_{q},
  \]
  which is valid for \(\frac{1}{p} + \frac{1}{q} = 1\).
  By choosing \(p \to \infty\) and \(q \to 1\), we have
  \[
    \mathrm{tr}\left(\mA \mB\right)
    =
    \inner{\mA^{\top}}{\mB}
    \leq
    \norm{\mA}_{\infty} 
    \norm{\mB}_1
    =
    \norm{\mA}_{2,2} 
    \norm{\mB}_1
    =
    \norm{\mA}_{\infty} 
    \mathrm{tr}\left(\mB\right),
  \]
  where the last equality follows from the fact that, for positive semidefinite matrices such as \(\mB\), it follows that  
  \[
    \mathrm{tr}\left(\mB\right) = \sum_{i} \lambda_{i}\left(\mB\right) = \sum_{i} \sigma_{i}\left(\mB\right),
  \]
  where \(\lambda_{i}\left(\mB\right)\) is the \(i\)th eigenvalue of \(\mB\).
    
\end{proof}

%% file: sections/section_proxsgd_convergence.tex
\subsubsection{Overview}
Consider a general class of optimization problems of  the form
\[
  \minimize_{\vlambda \in \Lambda}\quad F\left(\vlambda\right) \triangleq f\left(\vlambda\right) + h\left(\vlambda\right),
\]
where \(f\) is convex and smooth while \(h\) is a possibly non-smooth but convex regularizer.
Stochastic proximal gradient descent is a family of methods aimed to solve these types of problems by employing a \textit{proximal operator} denoted as
\[
\mathrm{prox}_{\gamma, h}\left(\vlambda\right)
=
\argmin_{\vlambda' \in \Lambda} \left[  h\left(\vlambda\right) + \frac{1}{2 \gamma} \norm{\vlambda - \vlambda'}_2^2
\right].
\]
Naturally, we expect the proximal operator to have closed form expression that can easily be computed.

Stochastic proximal gradient descent performs a gradient descent step on \(f\), followed by a proximal step on \(h\) such that
\[
  \vlambda_{t+1} = \mathrm{prox}_{\gamma_t, h}\left( \vlambda_t - \gamma_t \rvvg\left( \vlambda_t \right) \right),
\]
where \(\rvvg\) is some unbiased stochastic estimate of \(\nabla f\).
This type of algorithm has first been studied by \citet{duchi_adaptive_2011,ghadimi_minibatch_2016} followed by many more.

\subsubsection{Technical Assumptions}
Any proximal operator needs to satisfy the following basic property:
\begin{assumption}[\textbf{Non-Expansiveness}]\label{assumption:proximal_non_expansive}
The proximal operator is non-expansive such that 
\[
  {
  \lVert\mathrm{prox}_{\gamma, h}\left(\vlambda\right) - \mathrm{prox}_{\gamma, h}\left(\vlambda'\right)
  \rVert}
  \leq
  \norm{\vlambda - \vlambda'}
\]
for all \(\vlambda, \vlambda' \in \Lambda\).
\end{assumption}
This assumption can be satisfied as long as \(h\) is lower semi-continuous and convex.

For the convergence guarantee of proximal SGD, we follow the ``variance transfer'' strategy.
That is, instead of directly bounding the gradient variance on an arbitrary point \(\vlambda\), \citet{moulines_nonasymptotic_2011,nguyen_sgd_2018} ``transfer'' the gradient variance to the global optimum \(\vlambda^*\), proving convergence for gradient estimators with variance that grows with \(f(\vlambda)\).
The key assumptions are as follows:
\begin{assumption}[\textbf{Convex Expected Smoothness}]\label{assumption:convex_expected_smoothness}
The gradient estimator \(\rvvg\) is an unbiased estimator of \(f\) and is convex-smooth in expectation such that
\[
  \mathbb{E}\norm{ \rvvg\left(\vlambda\right) - \rvvg\left(\vlambda'\right) }_2^2
  \leq
  2 \mathcal{L} \, \mathrm{D}_f\left(\vlambda, \vlambda'\right),
\]
for any \(\vlambda, \vlambda' \in \Lambda\) and some constant \(0 < \mathcal{L} < \infty\), where \(\mathrm{D}_f\) is the Bregman divergence generated by \(f\).
\end{assumption}
This assumption has first been used by \citet{gower_stochastic_2021} and since has been popularly used throughout stochastic optimization.
(See \citet{khaled_better_2023} for an overview of conditions on the gradient variance.)

\begin{assumption}[\textbf{Bounded Gradient Variance}]\label{assumption:bounded_gradient_variance}
The variance of the gradient estimator \(\rvvg\) is bounded as
\[
  \mathrm{tr} \mathbb{V}\rvvg\left(\vlambda^*\right)
  \leq
  \sigma^2
\]
for some constant \(0 \leq \sigma^2 < \infty\), where  \(\vlambda^* = \argmin F\left(\vlambda\right)\) is the global optimum.
\end{assumption}
This assumption only requires the gradient variance on the global optimum to be finite.

\clearpage
\subsubsection{Convergence (\cref{thm:proxsgd_convergence_fixed})}
For completeness, we restate the proof by \citet[Theorem 11.9]{garrigos_handbook_2023}, which is based on the more general results of \citet{gorbunov_unified_2020}.
\input{theorems/thm_proxsgd_convergence}
\printProofs[proxsgdconvergencefixed]

\subsubsection{Complexity (\cref{thm:proxsgd_complexity_fixed})}
\vspace{2ex}

\input{theorems/thm_proxsgd_complexity_fixed}

\printProofs[complexityproxsgdfixed]

%% file: theorems/thm_proxsgd_convergence.tex
\begin{theoremEnd}[all end, category=proxsgdconvergencefixed]{lemma}[\citealp{gorbunov_unified_2020}]\label{thm:proxsgd_convergence_fixed}
    Let \(F = f + h\) be a composite objective on a convex domain \(\Lambda\), where \(f\) is \(\mu\)-strongly convex and \(h\) satisfies \cref{assumption:proximal_non_expansive}.
    Also, the gradient estimator \(\rvvg\) satisfies \cref{assumption:convex_expected_smoothness,assumption:bounded_gradient_variance}. 
    The last iterate \(\vlambda_T\) after \(T\) iterations of the stochastic proximal gradient descent with a constant step size \(\gamma\) satisfying \(\gamma \in \left(0, \min \left\{\frac{1}{2 \mathcal{L}}, \frac{1}{\mu} \right\}\right]\) achieves the bound on the distance to the global optimum \(\vlambda^* = \argmin_{\vlambda \in \Lambda} F(\vlambda)\) as 
    \[
      \mathbb{E} 
      \norm{\vlambda_{T+1} - \vlambda^*}_2^2 
      \leq 
      \left(1 - \gamma \mu\right)^T 
      \norm{\vlambda_0 - \vlambda^*}_2^2 
      + \frac{2\gamma \sigma^2}{\mu}.
    \]
\end{theoremEnd}
\begin{proofEnd}
    First, the iterate at time \(t\) satisfies
    \begin{align}
        \norm{\vlambda_{t+1} - \vlambda^*}_2^2 
        &=
        {\lVert
        \textrm{prox}_{\gamma h}\left( \vlambda_t - \gamma \rvvg \left(\vlambda_t\right)\right)
        - 
        \textrm{prox}_{\gamma h}\left( \vlambda^* - \gamma \nabla f \left(\vlambda^*\right)\right)
        \rVert}_2^2,
        \nonumber
\shortintertext{and from \cref{assumption:proximal_non_expansive},}
        &\leq
        \norm{\left(\vlambda_t - \gamma \rvvg \left(\vlambda_t\right) \right) - \left(\vlambda^* - \gamma \nabla f \left(\vlambda^*\right) \right)}_2^2 \nonumber
        \\
        &=
        \norm{\vlambda_t - \vlambda^*}_2^2 
        + \gamma^2 \norm{\rvvg \left(\vlambda_t\right) - \nabla f \left( \vlambda^*\right)}_2^2
        - 2 \gamma \langle \rvvg \left( \vlambda_t \right) - \nabla f \left( \vlambda^*\right), \vlambda_t - \vlambda^*\rangle.
        \nonumber
    \end{align}
    Therefore, denoting the filtration as \(\mathcal{F}_t\), which is the \(\sigma\)-field generated by iterates up to \(t\), the conditional expectation is bounded as
    \begin{align}
        &\E{
          \norm{\vlambda_{t+1} - \vlambda^*}_2^2 \mid \mathcal{F}_t
        }
        \nonumber
        \\
        &\qquad\leq
        \E{
          \norm{\vlambda_t - \vlambda^*}_2^2 \mid \mathcal{F}_t
        }
        + 
        \gamma^2 
        \E{
        \norm{
          \rvvg \left(\vlambda_t\right) - \nabla f \left( \vlambda^*\right) \mid \mathcal{F}_t
        }
        }_2^2
        - 2 \gamma 
        \E{
        \inner{
          \rvvg \left( \vlambda_t \right) - \nabla f \left( \vlambda^*\right)
        }{
          \vlambda_t - \vlambda^*
        }
        \mid \mathcal{F}_t
        }.
        \label{eq:convergence_bound_1}
    \end{align}

    The second term, or the gradient variance term of \cref{eq:convergence_bound_1}, can be dealt with the ``variance transfer'' strategy pioneered by~\cite{nguyen_sgd_2018}:
    \begin{alignat}{2}
        \gamma^2 \E{ \norm{\rvvg \left(\vlambda_t\right) - \nabla f \left( \vlambda^*\right)}_2^2 \mid \mathcal{F}_t }
        &\leq
        2 \gamma^2 
        \E{ \norm{ \rvvg \left(\vlambda_t\right) - \rvvg \left( \vlambda^*\right)}_2^2 \mid \mathcal{F}_t }
        +
        2 \gamma^2 \E{\norm{ \rvvg \left(\vlambda^*\right) - \nabla f \left( \vlambda^*\right)}_2^2 \mid \mathcal{F}_t },
        \nonumber
\shortintertext{and applying \cref{assumption:convex_expected_smoothness,assumption:bounded_gradient_variance},}
        &\leq
        2 \gamma^2
        \left(
        2 \mathcal{L} \,
        \textrm{D}_f\left(\vlambda_t, \vlambda^*\right)
        \right)
        +
        2 \gamma^2 \sigma^2.
        \label{eq:convergence_bound_2}
    \end{alignat}
    For the third term of \cref{eq:convergence_bound_1}, we have
    \begin{alignat}{2}
        - 2 \gamma \E{
        \langle \rvvg \left( \vlambda_t \right) - \nabla f \left( \vlambda^*\right), \vlambda_t - \vlambda^*\rangle
        \mid \mathcal{F}_t 
        }
        &=
        - 2 \gamma 
        \inner{
          \E{ \rvvg \left( \vlambda_t \right) \mid \mathcal{F}_t } - \nabla f \left( \vlambda^*\right)
        }{ \vlambda_t - \vlambda^* }
        \nonumber
        \\
        &=
        - 2 \gamma \inner{ \nabla f \left( \vlambda_t \right) - \nabla f \left( \vlambda^*\right)}{ \vlambda_t - \vlambda^*}
        \nonumber
        \\
        &=
        2 \gamma \inner{ \nabla f\left(\vlambda_t\right) }{ \vlambda^* - \vlambda_t }
        +
        2 \gamma \inner{ \nabla f\left( \vlambda^* \right) } { \vlambda_t - \vlambda^* },
        \nonumber
\shortintertext{and applying the \(\mu\)-strong convexity of \(f\),}
        &\leq
        - \gamma \mu \norm{\vlambda_t - \vlambda^*}_2^2 - 2\gamma \, \mathrm{D}_f \left(\vlambda_t, \vlambda^*\right).
        \label{eq:convergence_bound_3}
    \end{alignat}
    Applying \cref{eq:convergence_bound_2,eq:convergence_bound_3} to \cref{eq:convergence_bound_1} and taking the full expectation, we get 
    \begin{align*}
        \mathbb{E} \norm{\vlambda_{t+1} - \vlambda^*}_2^2 
        &\leq
        \mathbb{E} \norm{\vlambda_t - \vlambda^*}_2^2
        +
        4 \gamma^2 \mathcal{L} \mathbb{E} \left[\mathrm{D}_f \left( \vlambda_t, \vlambda^* \right)\right] + 2 \gamma^2 \sigma^2
        - \gamma \mu \mathbb{E} \norm{\vlambda_t - \vlambda^*}_2^2
        -
        2 \gamma \mathbb{E}\left[ \mathrm{D}_f \left(\vlambda_t, \vlambda^*\right)\right]
        \\
        &=
        \left(1-\gamma \mu\right) \mathbb{E} \norm{\vlambda_t - \vlambda^*}_2^2
        +
        2 \gamma \left( 2\gamma \mathcal{L} -1 \right) \mathbb{E} \left[\mathrm{D}_f \left(\vlambda_t, \vlambda^*\right) \right]
        +
        2\gamma^2 \sigma^2,
\shortintertext{by ensuring that \(2 \gamma \mathcal{L} \leq 1\), we have}
        &\leq
        \left(1-\gamma \mu\right) \mathbb{E} \norm{\vlambda_t - \vlambda^*}_2^2
        +
        2\gamma^2 \sigma^2.
    \end{align*}
    Provided that \(1-\gamma \mu \geq 0\), we can solve the recursion above and get
    \begin{align*}
        \mathbb{E} \norm{\vlambda_{t+1} - \vlambda^*}_2^2 
        &\leq
        \left(1 - \gamma \mu \right)^t \norm{\vlambda_0 - \vlambda^*}_2^2
        +
        2 \gamma^2 \sigma^2 \sum_{k=1}^{t-1}\left(1-\gamma \mu \right)^k. 
    \end{align*}
    Since the geometric sum can be upper-bounded by
    \begin{align*}
        \sum_{k=1}^{t-1}\left(1-\gamma \mu \right)^k 
        =
        \frac{1-\left(1- \gamma \mu \right)^t}{\gamma \mu}
        \leq
        \frac{1}{\gamma \mu},
    \end{align*}
    which yields
    \begin{align*}
        \mathbb{E} \norm{\vlambda_{t+1} - \vlambda^*}_2^2 
        &\leq
        \left(1 - \gamma \mu \right)^t \norm{\vlambda_0 - \vlambda^*}_2^2
        +
        \frac{2 \gamma \sigma^2}{\mu}.
    \end{align*}
\end{proofEnd}

\begin{theoremEnd}[all end, category=proxsgdcomplexitydecreased]{lemma}[\citealp{gower_sgd_2019}]
    Let \(\vlambda_t\) be the iterate with a step size of \(\gamma_t \leq \frac{1}{2\mathcal{L}}\). Given constants \(\mu>0\) and \(\sigma^2 >0\), suppose we have 
    \[
    \mathbb{E} \lVert \vlambda_{t+1} - \vlambda^* \rVert_2^2 
    \leq
    \left(1-\gamma_t\right) \mathbb{E} \lVert \vlambda_t - \vlambda^*\rVert_2^2 + 2\gamma_t^2 \sigma^2.
    \]
    Then, by switching to a stepsize that decays over time, as shown,
    \[
    \gamma_t = \begin{cases}
        \, \frac{1}{2\mathcal{L}} &\quad \text{for } t< t^*\\
        \vspace{0.5ex}\\
        \, \frac{1}{\mu}\frac{2t+1}{\left(t+1\right)^2} &\quad \text{for } t\geq t^*
    \end{cases}
    \]
    where \(t^*=4\lfloor \mathcal{L} / \mu\rfloor\), we have that 
    \[
    \mathbb{E} \lVert \vlambda_{T+1} - \vlambda^* \rVert_2^2
    \leq
    \frac{16\lfloor\mathcal{L} / \mu\rfloor^2}{(T+1)^2}\lVert \vlambda^0-\vlambda^*\rVert_2^2
    +
    \frac{\sigma^2}{\mu^2} \frac{8}{T+1}.
    \]
\end{theoremEnd}
\begin{proofEnd}
    The proof divides its analysis into two phases, based on whether the step number \(t\) is before or after a specific threshold value, \(t^*\). First, for \(t \leq t^*\), with a fixed step size of \(\gamma_t = \gamma\), we have
    \[
    \mathbb{E} \lVert \vlambda_{t+1} - \vlambda^* \rVert_2^2 
    \leq
    \left(1-\gamma\right) \mathbb{E} \lVert \vlambda_t - \vlambda^*\rVert_2^2 + 2\gamma^2 \sigma^2,
    \]
    and by recursively solving the above, we get
    \begin{align}
    \mathbb{E} \norm{\vlambda_{t^*} - \vlambda^*}_2^2 
    \leq
    \left(1-\gamma \mu\right)^{t^*} \norm{\vlambda_0 - \vlambda^*}_2^2
    +
    \frac{2 \sigma^2 \gamma}{\mu}.
    \label{eq:complexity_decay_stepsize_1}
    \end{align}
    Next, for \(t \geq t^*\), we switch to decaying step size of \(\gamma_t = \frac{1}{\mu}\frac{2t+1}{\left(t+1\right)^2}\). Then, we get
    \[
    \mathbb{E} \norm{\vlambda_{t+1} - \vlambda^*}_2^2
    \leq
    \frac{t^2}{\left(t+1\right)^2}\mathbb{E}\norm{\vlambda_t - \vlambda^*}_2^2 
    + \sigma^2 \frac{2}{\mu^2} \frac{(2 t+1)^2}{(t+1)^4},
    \]
    and multiplying \((t+1)^2\) to the both side, we get
    \begin{align*}
        \left(t+1\right)^2 \mathbb{E} \norm{\vlambda_{t+1}-\vlambda^*}_2^2
        &\leq
        t^2 \mathbb{E}\norm{\vlambda_t - \vlambda^*}_2^2 
        + \sigma^2 \frac{2}{\mu^2} \frac{(2 t+1)^2}{(t+1)^2}
\shortintertext{as \(\frac{2t+1}{t+1}\leq \frac{2t+2}{t+1}=2\),}
        &\leq
        t^2 \mathbb{E}\norm{\vlambda_t - \vlambda^*}_2^2 
        +  \frac{8\sigma^2}{\mu^2}.
    \end{align*}
    By summing up from \(t=t^*, \dots, T\), we have
    \begin{align*}
        \sum_{t=t^*}^T 
        \left(t+1\right)^2 \mathbb{E} \norm{\vlambda_{t+1}-\vlambda^*}_2^2
        &\leq
        \sum_{t=t^*}^T \left(
        t^2 \mathbb{E}\norm{\vlambda_t - \vlambda^*}_2^2 
        +  \frac{8\sigma^2}{\mu^2}\right),
\shortintertext{with telescopic cancellation, we have}
        \left(T+1\right)^2 \mathbb{E} \norm{\vlambda_{T+1} - \vlambda^*}_2^2
        &\leq
        \left(t^*\right)^2
        \mathbb{E} \norm{\vlambda_{t^*}- \vlambda^*}_2^2 + \left(T- t^*\right) \frac{8\sigma^2}{\mu^2}.
    \end{align*}
    Given the result for \(t<t^*\) in the \cref{eq:complexity_decay_stepsize_1}, we use \(\gamma=\frac{1}{\mathcal{L}}\) and can replace the \(\mathbb{E} \norm{\vlambda_{t^*}- \vlambda^*}_2^2\) term as
    \begin{align*}
        \left(T+1\right)^2 \mathbb{E} \norm{\vlambda_{T+1} - \vlambda^*}_2^2
        &\leq
        \left(t^*\right)^2
        \left((1-\gamma \mu)^{t^*}\norm{\vlambda^0-\vlambda^*}_2^2+ \frac{2 \sigma^2\gamma}{\mu}\right)+\left(T-t^*\right) \frac{8 \sigma^2}{\mu^2},
\shortintertext{since \(\left(1-\gamma \mu\right) \leq 1\),}
        &\leq
        \left(t^*\right)^2
        \left(
        \norm{\vlambda^0-\vlambda^*}_2^2+ \frac{2 \sigma^2 \gamma}{\mu}\right)+\left(T-t^*\right) \frac{8 \sigma^2}{\mu^2}.
    \end{align*}
    By substituting \(\gamma = \frac{1}{2\mathcal{L}}, t^* = 4\lfloor\mathcal{L} / \mu\rfloor\) and dividing by \((T+1_^2\) on both sides, we get
    \begin{align*}
        \mathbb{E} \norm{\vlambda_{T+1} - \vlambda^*}_2^2
        &\leq
        \frac{16 \lfloor\mathcal{L} / \mu\rfloor^2}{(T+1)^2}
        \left(
        \norm{\vlambda^0-\vlambda^*}_2^2+ \frac{\sigma^2 }{\mu \mathcal{L}}\right)+ \frac{\left(T-t^*\right)}{(T+1)^2} \frac{8 \sigma^2}{\mu^2}
        \\
        &\leq
        \frac{16 \lfloor\mathcal{L} / \mu\rfloor^2}{(T+1)^2}
        \norm{\vlambda^0-\vlambda^*}_2^2 
        + \frac{\sigma^2}{\mu^2(T+1)^2}\left(8\left(T-t^*\right)+\left(t^*\right)^2 \frac{\mu}{\mathcal{L}}\right)
        \\
        &\leq
        \frac{16 \lfloor\mathcal{L} / \mu\rfloor^2}{(T+1)^2}
        \norm{\vlambda^0-\vlambda^*}_2^2 
        + \frac{\sigma^2}{\mu^2} \frac{8(T-2t^*)}{(T+1)^2},
\shortintertext{since \((T-2t^*) \leq (T+1)\),}
        &\leq
        \frac{16 \lfloor\mathcal{L} / \mu\rfloor^2}{(T+1)^2}
        \norm{\vlambda^0-\vlambda^*}_2^2 
        + \frac{\sigma^2}{\mu^2} \frac{8}{T+1}
    \end{align*}
\end{proofEnd}

%% file: theorems/thm_proxsgd_complexity_fixed.tex
\begin{theoremEnd}[all end, category=complexityproxsgdfixed]{corollary}\label{thm:proxsgd_complexity_fixed}
  Let the assumptions of \cref{thm:proxsgd_convergence_fixed} be satisfied.
  Then, the last iterate \(\vlambda_{T+1}\) of proximal SGD with a fixed stepsize satisfies \( \mathbb{E}{\lVert \vlambda_{T+1} - \vlambda^* \rVert}_2^2 \leq \epsilon \), where \(\vlambda^* = \argmax_{\vlambda \in \Lambda} F\left(\vlambda\right)\) is the global optimum, if
{%
\setlength{\abovedisplayskip}{1.ex} \setlength{\abovedisplayshortskip}{1.ex}
\setlength{\belowdisplayskip}{1.ex} \setlength{\belowdisplayshortskip}{1.ex}
  \begin{align*}
        \gamma &= 
        \min \left\{
            \frac{\epsilon \mu}{4 \sigma^2},
            \frac{2 \mathcal{L}}{\mu}, \mu
        \right\} \quad\text{and}\\
        T
        &\geq
        \max \left\{
            \frac{4 \sigma^2}{\mu^2}
            \frac{1}{\epsilon},
            \frac{2 \mathcal{L}}{\mu},
            1
        \right\}
        \log \left( 2 \lVert\vlambda^* - \vlambda_0\rVert_2^2 \, \frac{1}{\epsilon} \right).
  \end{align*}
}
\end{theoremEnd}
\begin{proofEnd}
    We can convert \cref{thm:proxsgd_convergence_fixed} to an iteration complexity guarantee for achieving an \(\epsilon\)-accurate solution using Lemma A.2 of \citet{garrigos_handbook_2023} where the constants are:
    \[
    \alpha_0 = \norm{ \vlambda_0 - \vlambda^* }_2^2, \qquad
    A = \frac{2 \sigma^2}{\mu}, 
    \;\text{and}\qquad
    C = \max\left\{2 \mathcal{L}, \mu \right\}.
    \]
    That is, we can satisfy \( \mathbb{E}{\lVert \vlambda_{T+1} - \vlambda^* \rVert}_2^2 \leq \epsilon \) with
    \begin{align*}
        \gamma 
        &= 
        \min\left\{ \frac{\epsilon}{2 A}, \frac{1}{C} \right\}
        =
        \min\left\{ \frac{\epsilon \mu}{4 \sigma^2},\, 2 \mathcal{L},\, \mu \right\}
    \end{align*}
    and
    \begin{align*}
        T
        &\geq
        \max\left\{
          \frac{1}{\epsilon} \frac{2 A}{\mu}, \frac{C}{\mu}
        \right\}
        \log \left(2 \norm{\vlambda_0 - \vlambda^*}_2^2 \frac{1}{\epsilon}\right)
        \\
        &=
        \max\left\{
          \frac{1}{\epsilon} \frac{4 \sigma^2}{\mu^2}, \frac{2 \mathcal{L}}{\mu}, 1
        \right\}
        \log \left(2 \norm{\vlambda_0 - \vlambda^*}_2^2 \frac{1}{\epsilon}\right).
    \end{align*}
\end{proofEnd}




%% file: sections/section_bbvi_complexity.tex
\subsubsection{Overview}\label{section:proxsgd_bbvi_details}
\vspace{1ex}
Now, we establish the complexity of BBVI applied to finite-sum likelihoods of the form of \cref{eq:finitesum}.

\paragraph{Domain}
As discussed in \cref{section:family}, we restrict our interest to location-scale variational families with a Cholesky scale parameterization.
This effectively means we restrict \(\mC\) to be lower triangular and have only positive diagonal entries.
The domain of the variational parameters is then
\[
  \Lambda \triangleq \{ (\vm, \mC) \mid (\vm, \mC) \in \mathbb{R}^d \times \mathbb{L}_{++}^d\left(\mathbb{R}\right) \text{ and \(C_{ii} > 0\) for all \(i = 1, \ldots, d\)} \},
\]
which is convex.
Furthermore, the negative entropy 
\[
  h\left(\vlambda\right)
  =
  -\mathbb{H}\left(q_{\vlambda}\right)
  =
  - \log \mathrm{det} \, \mC + \mathbb{1}_{\Lambda}\left(\vlambda\right)
  =
  -\sum_{i=1}^d \log C_{ii}
  +
  \mathbb{1}_{(0, +\infty)}\left(C_{ii}\right),
\]
where \(\mathbb{1}_{A}\left(x\right)\) is 0 if \(x \in A\) and \(\infty\) otherwise, is closed and convex \citep[Lemma 19]{domke_provable_2023}.

\paragraph{Proximal Operator}
Also, as noted in \cref{section:proximal_sgd}, we leverage proximal SGD, for which we provide detailed analysis in \cref{section:proximal_sgd_proofs} for completeness.
We use the proximal operator proposed by \citet{domke_provable_2020}, which, at each application updates the diagonal of \(\mC\) such that
\begin{equation}
  \mathrm{prox}_{\gamma, h}\left(\vlambda\right)
  =
  \left(\vm, \mC + \Delta \mC\right),
  \label{eq:domke_proximal}
\end{equation}
where \(\vlambda = (\vm, \mC)\) and
\[
  \Delta C_{ii} = \frac{1}{2} \left( \sqrt{C_{ii}^2 + 4 \gamma} - C_{ii} \right),
\]
\(\Delta C_{ij} = 0\) for \(i \neq j\).
Conveniently, the complexity of applying this proximal operator is \(\Theta\left(d\right)\).
Furthermore, this proximal operator satisfies \cref{assumption:proximal_non_expansive}.
(See the recent work by \citet{domke_provable_2023} for more details.)

\paragraph{Proof Sketch}
The proof is based on the general analysis of proximal SGD in \cref{section:proximal_sgd_proofs}.
Specifically, we establish 
\begin{enumerate}
    \item \cref{assumption:convex_expected_smoothness} in \cref{section:convex_expected_smoothness} and
    \item \cref{assumption:bounded_gradient_variance} in \cref{section:quadratic_variance}.
\end{enumerate}
The key ingredients are in \cref{section:key_lemmas}.
\begin{enumerate}
    \item \cref{thm:jacobian_reparam} establishes the precise expression of the squared Jacobian of \(\mathcal{T}_{\vlambda}^n\).
    \item \cref{thm:gradnorm} connects the properties of the gradient variance with \(\mathcal{T}_{\vlambda}^n\) through its Jacobian.
    \item \cref{thm:location_scale_reparam} resolves the randomness by directly solving the expectations.
\end{enumerate}

%% file: sections/section_gradient_variance_key_lemma.tex
\begin{corollary}\label{thm:reparam_gradnorm}
Let the assumptions of \cref{thm:location_scale_reparam} hold. Then, 
\[
        \mathbb{E} 
        \left( 1 + {\textstyle{\textstyle\sum_j}} \delta_{n,j} \rvu_j^2 \right) {\lVert \mathcal{T}^n_{\vlambda}\left(\rvvu\right) 
        - \mathcal{T}^n_{\vlambda'}\left(\rvvu\right) \rVert}_2^2
        \leq
        \left(
      {\textstyle \sum_{j}}
      \delta_{n,j}
      +
      k_{\varphi}
      \right)
      {\lVert
        \vlambda - \vlambda^{'}
      \rVert}_2^2.
\]
\end{corollary}
\begin{proof}
    From the linearity of \(\mathcal{T}^n\), it follows that
    \[
       \mathcal{T}_{\vlambda - \vlambda'}^n\left(\rvvu\right)
       =
       \left( \mC_n - \mC_n' \right) \rvvu + (\vm_n - \vm_n'),
    \]
    where \(\vm_n, \mC_n\) are part of \(\vlambda\) and \(\vm_n', \mC_n'\) are part of \(\vlambda'\).
    Then, the result immediately follow from applying  \cref{thm:location_scale_reparam} with \(\vz = \boldupright{0}\) as
    \begin{align}
    \mathbb{E}
        \left(
        1 +  {\textstyle\sum_j} \delta_{n,j} \rvu_j^2
        \right)
        {\lVert
        \mathcal{T}_{\vlambda-\vlambda'}^{n}\left(\rvvu\right)
       \rVert}_2^2
       \nonumber
      &\leq
      \left(
      {\textstyle\sum_j}
        \delta_{n,j}
        +1
      \right)
      \norm{ 
        \vm_n - \vm_n'
      }_2^2
      +
      \left(
      {\textstyle\sum_{j}}
      \delta_{n,j}
      +
      k_{\varphi}
      \right)
      \norm{ 
        \mC_n - \mC_n'
      }_{\mathrm{F}}^2
      \nonumber
      \\
      &=
      \left(
        {\textstyle\sum_j} \delta_{n,j}
       + 1
      \right)
      \norm{ \vm_n - \vm_n'}_2^2
      +
      \left(
      {\textstyle\sum_{j}}
      \delta_{n,j}
      +
      k_{\varphi}
      \right)
      \norm{ \mC_n - \mC_n' }_{\mathrm{F}}^2,
      \nonumber
\shortintertext{since the kurtosis always satisfies \(k_{\varphi} \geq 1\),}
      &\leq
    \left(
      {\textstyle\sum_{j}}
      \delta_{n,j}
      +
      k_{\varphi}
      \right)
      \left(
      \norm{ \vm_n - \vm_n' }_2^2
      +
      \norm{ \mC_n - \mC_n' }_{\mathrm{F}}^2
      \right)
      \nonumber
\shortintertext{and since adding more components into a squared \(\ell_2\)-norm always results in an upper bound,}
      &\leq
      \left(
      {\textstyle\sum_{j}}
      \delta_{n,j}
      +
      k_{\varphi}
      \right)
      \left(
      \norm{ \vm - \vm' }_2^2
      +
      \norm{ \mC - \mC' }_{\mathrm{F}}^2
      \right)
      \nonumber
      \\
      &=
      \left(
      {\textstyle\sum_{j}}
      \delta_{n,j}
      +
      k_{\varphi}
      \right)
      \norm{
        \vlambda - \vlambda'
      }_2^2.
      \nonumber
\end{align}    
\end{proof}

%% file: sections/section_experimental_setup.tex
\section{Probabilistic Models and Datasets}\label{section:models}
\subsection{Robust Poisson Regression}
We use a robustified version of Poisson regression for modeling count data.
This model is known as the Poisson-log-normal model~\citep[\S 4.2.4]{cameron_regression_2013}, which is a ``localized'' version of regular Poisson regression~\citep[\S 3.2]{wang_general_2018}. 
That is, an additional hierarchy is added to model the local variations of each datapoint.
A closely related model is negative binomial regression, which is obtained by setting a conjugate prior to the local noise.
Here, the noise is modeled to be log-normal, resulting in a non-conjugate likelihood:
{%
\begin{alignat*}{2}
  \eta_i &\sim \mathcal{N}\left(\vx_i \vbeta + \alpha, \sigma_{\eta}\right) &&\quad i = 1, \ldots, n \\
  y_i    &\sim \text{\textsf{Poisson}}\left( \operatorname{exp}\left(\eta_i\right) \right) && \quad i = 1, \ldots, n,
\end{alignat*}
}%
where \((\eta_i)_{i=1}^n\) are the local variables.
The global variables are given by the priors 
{%
\begin{alignat*}{2}
  \sigma_{\alpha} &\sim \text{\textsf{Student-t}}_+\left(4, 0, 1\right) \\
  \sigma_{\beta}  &\sim \text{\textsf{Student-t}}_+\left(4, 0, 1\right) \\
  \sigma_{\eta}   &\sim \text{\textsf{Student-t}}_+\left(4, 0, 1\right) \\
  \alpha &\sim \mathcal{N}\left(0, \sigma_{\alpha}\right) \\
  \vbeta &\sim \mathcal{N}\left(\boldupright{0}, \sigma_{\beta}\right).
\end{alignat*}
}%
We use the \textsf{rwm5yr} German health registry doctor visit dataset~\citep{hilbe_negative_2011} from the \texttt{COUNT} package in R~\citep{hilbe_count_2016}.

\subsection{Item Response Theory}
Item response theory (IRT) is a family of models for estimating the response of humans to a set of items, often in the form of exams or questionnaires~\citep{lord_statistical_2008}.
BBVI has recently been shown to be very successful in estimating human ability from large-scale educational examination datasets~\citep{wu_variational_2020}.
We employ the so-called two-parameter logistic model, or ``2PL'' model, for which the likelihood is given as
{%
\begin{alignat*}{2}
  \mathrm{logit}_i &= \gamma_{\text{item}_i} \alpha_{\text{student}_i} +  \beta_{\text{item}_i} + \mu_{\beta}  \\
  y_i &\sim \text{\textsf{Bernoulli-logit}}\left(\mathrm{logit}_i\right) &&\quad i = 1, \ldots, n.
\end{alignat*}
}%
The global variables are given as
{%
\begin{alignat*}{2}
  \mu_{\beta}     &\sim \text{\textsf{Student-t}}\left(4, 0, 1\right) \\
  \sigma_{\beta}  &\sim \text{\textsf{Student-t}}_+\left(4, 0, 1\right) \\
  \sigma_{\gamma} &\sim \text{\textsf{Student-t}}_+\left(4, 0, 1\right) \\
  \gamma_k &\sim \text{\textsf{log-normal}}\left(0, \sigma_{\gamma}\right) &&\quad k = 1, \ldots, K,
\end{alignat*}
}%
while the local variables are
{%
\begin{alignat*}{2}
  \alpha_j &\sim \mathcal{N}\left(0, 1\right) && \quad j = 1, \ldots, J.
\end{alignat*}
}%
The log-normal prior on \(\gamma\) is inspired by \citet{patz_straightforward_1999}.

While here we only consider scaling with respect to the students, we can also consider scaling with respect to the number of items by also making \(q\left(\vbeta\right)\) and \(q\left(\vgamma\right)\) factor out.
While this is less important for our dataset of choice, which has a small \(K\), this is certainly attractive to other datasets where both \(K\) and \(J\) are large.

For the dataset, we take \textsf{CritLangAcq} from~\citet{wu_variational_2020}.
Unfortunately, this dataset is too large to fit in memory even for the mean-field approximation.
Therefore, we only use a 5\% random subset of the full dataset.
Scaling to the full dataset will require additional strategies amortization~\citep{kingma_autoencoding_2014,dayan_helmholtz_1995} as done in the original work by \citeauthor{wu_variational_2020}.

\newpage
\subsection{Multivariate Stochastic Volatility}
Multivariate stochastic volatility~\citep{chib_multivariate_2009}
The likelihood is given as
{%
\begin{align*}
  \vy_1 &\sim \mathcal{N}\left(\vmu, \mQ\right) \\
  \vy_t &\sim \mathcal{N}\left(\vmu + \vphi \left(\vy_{t-1} - \vmu\right), \mQ\right) \\
  \vx_t &\sim \mathcal{N}\left(0, \exp\left( \vy_t/2\right)\right),
\end{align*}
}%
where \({(\vy_t)}_{t=1}^T\), the latent stochastic volatilities, are the local variables.
For the hyperpriors, we develop a fully Bayesian variant of the model used by \citet[\S 5]{naesseth_variational_2018} who perform empirical Bayes inference on the hyperparameters.
Notably, following modern Bayesian modeling practice~\citep{gelman_bayesian_2020}, we assign a Cauchy-LKJ prior to the covariance \(\mQ\).
The global variables are given as
\begin{align*}
  \mL_{\mSigma} &\sim \text{\textsf{LKJ-Cholesky}}\left(d_{\vy}, 1\right) \\
  \vtau &\sim \text{\textsf{Cauchy}}_+\left(0, 5\right) \\
  \mL_{\mQ} &= \mathrm{diag}\left(\vtau\right) \mL_{\mSigma} \\
  \mQ       &= \mL_{\mQ}\mL_{\mQ}^{\top} \\
  \vmu  &\sim \text{\textsf{Cauchy}}\left(0, 10\right) \\
  \vphi &\sim \text{\textsf{uniform}}\left(-1, 1\right),
\end{align*}
where all vector operations are elementwise.

For the datasets, we use the exchange rate (``FX'') between 6 international currencies and the U.S. dollar.
In particular, we use the daily closing exchange rate of \textsf{EUR, JPY, GBP, AUD, CAD}, and \textsf{KRW} over the period from 2006-05-16 to 2023-08-30.
For the subsets, we deterministically slice a continuous period starting from 2006-05-16.

%% file: sections/section_additional_experiments.tex
\section{Additional Experimental Results}\label{section:additional_results}
\vspace{-1ex}
This section shows additional plots for the experimental results displayed in \cref{section:realworld}.
In particular, we show convergence plots with respect to the number of iterations for a fixed stepsize.
Note that all methods use the same number of gradient evaluations per iteration.
Therefore, ``iteration'' is synonymous with ``number of gradient queries.'' 

\vspace{-1ex}
\subsection{Results on \textsf{rpoisson}}
\vspace{-1ex}
    
\begin{figure}[H]
    \centering
    \subfloat[\(T = 50\mathrm{k}\)]{
      \hspace{-1em}
      \includegraphics[scale=0.9]{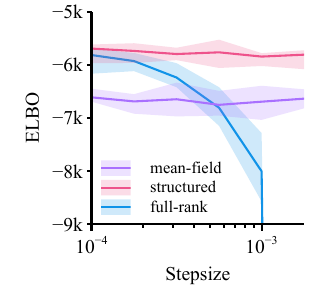}
    }
    \subfloat[\(\gamma = 10^{-4}\)]{
      \includegraphics[scale=0.9]{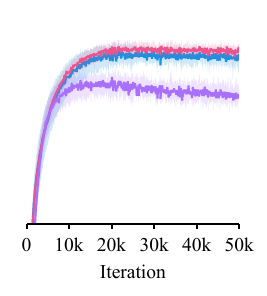}
    }
    \subfloat[\(\gamma = 10^{-3.5}\)]{
      \includegraphics[scale=0.9]{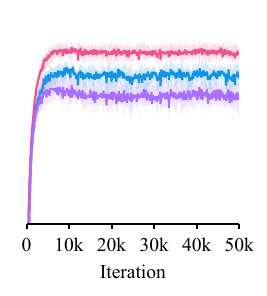}
    }
    \subfloat[\(\gamma = 10^{-3}\)]{
      \includegraphics[scale=0.9]{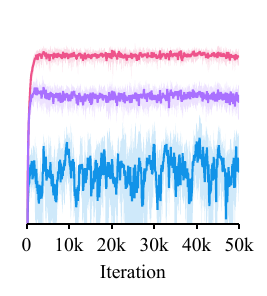}
    }
    \vspace{-1ex}
    \caption{
    \textbf{ELBO versus stepsize on \textsf{rpoisson-small}}
    The solid lines are the median, while the shaded regions are the 80\% quantiles computed from 4 independent replications.
    Notice that the performance gap between \textbf{full-rank} and \textbf{structured} becomes narrower as we reduce the stepsize.
    }
    \vspace{-2ex}
\end{figure}

\begin{figure}[H]
    \centering
    \subfloat[\(T = 50\mathrm{k}\)]{
      \hspace{-1em}
      \includegraphics[scale=0.90]{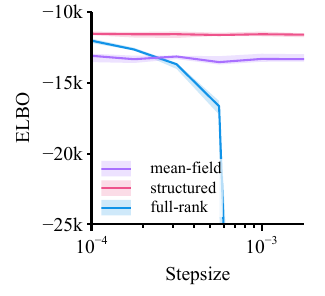}
    }
    \subfloat[\(\gamma = 10^{-4}\)]{
      \includegraphics[scale=0.90]{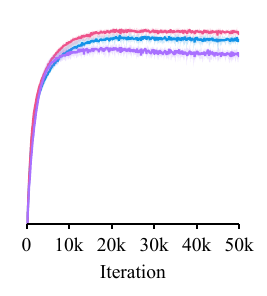}
    }
    \subfloat[\(\gamma = 10^{-3.5}\)]{
      \includegraphics[scale=0.90]{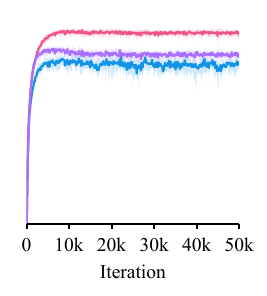}
    }
    \subfloat[\(\gamma = 10^{-3}\)]{
      \includegraphics[scale=0.90]{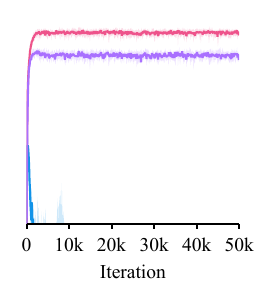}
    }
    \vspace{-1ex}
    \caption{
    \textbf{ELBO versus stepsize on \textsf{rpoisson-middle}}
    The solid lines are the median, while the shaded regions are the 80\% quantiles computed from 4 independent replications.
    Notice that the performance gap between \textbf{full-rank} and \textbf{structured} becomes narrower as we reduce the stepsize.
    }
    \vspace{-2ex}
\end{figure}

\begin{figure}[H]
    \centering
    \subfloat[\(T = 50\mathrm{k}\)]{
      \hspace{-1em}
      \includegraphics[scale=0.90]{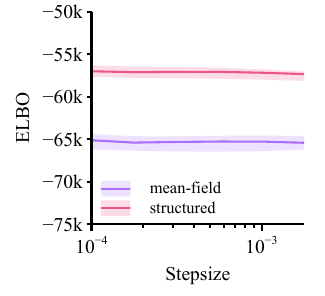}
    }
    \subfloat[\(\gamma = 10^{-4}\)]{
      \includegraphics[scale=0.90]{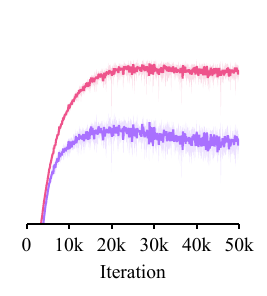}
    }
    \subfloat[\(\gamma = 10^{-3.5}\)]{
      \includegraphics[scale=0.90]{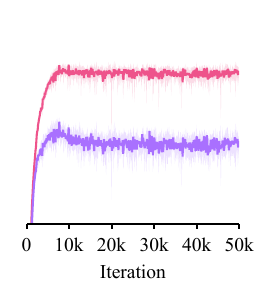}
    }
    \subfloat[\(\gamma = 10^{-3}\)]{
      \includegraphics[scale=0.90]{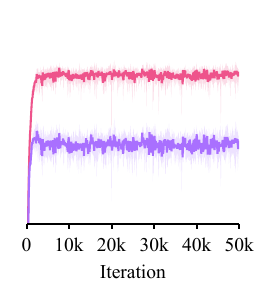}
    }
    \vspace{-1ex}
    \caption{
    \textbf{ELBO versus stepsize on \textsf{rpoisson-large)}}
    \textbf{Full-rank} is omitted as it didn't fit in memory.
    The solid lines are the median, while the shaded regions are the 80\% quantiles computed from 4 independent replications.
    }
\end{figure}

\subsection{Results on \textsf{volatility}}

\begin{figure}[H]
    \centering
    \subfloat[\(T = 50\mathrm{k}\)]{
      \hspace{-1em}
      \includegraphics[scale=0.90]{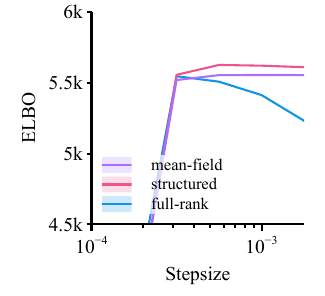}
    }
    \subfloat[\(\gamma = 10^{-3.5}\)]{
      \includegraphics[scale=0.90]{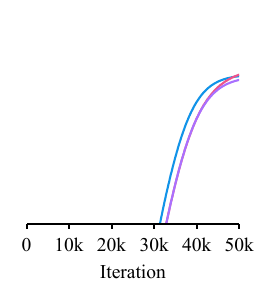}
    }
    \subfloat[\(\gamma = 10^{-3}\)]{
      \includegraphics[scale=0.90]{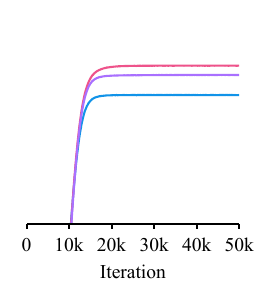}
    }
    \subfloat[\(\gamma = 10^{-2.75}\)]{
      \includegraphics[scale=0.90]{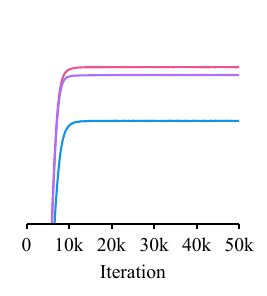}
    }
    \caption{
    \textbf{ELBO versus stepsize on \textsf{volatility-small}}
    The solid lines are the median, while the shaded regions are the 80\% quantiles computed from 4 independent replications.
    }
    \vspace{-2ex}
\end{figure}

\begin{figure}[H]
    \centering
    \subfloat[\(T = 50\mathrm{k}\)]{
      \hspace{-1em}
      \includegraphics[scale=0.90]{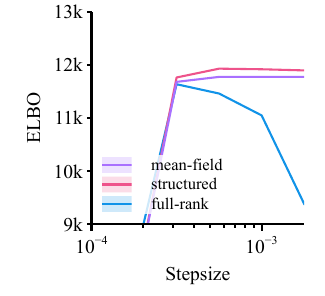}
    }
    \subfloat[\(\gamma = 10^{-3.5}\)]{
      \includegraphics[scale=0.90]{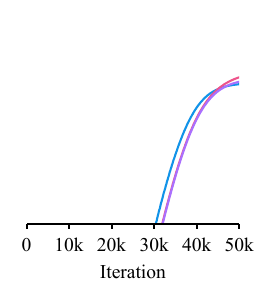}
    }
    \subfloat[\(\gamma = 10^{-3}\)]{
      \includegraphics[scale=0.90]{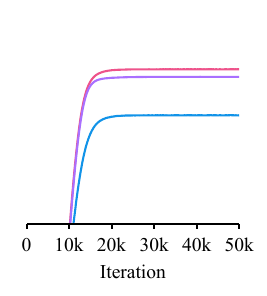}
    }
    \subfloat[\(\gamma = 10^{-2.75}\)]{
      \includegraphics[scale=0.90]{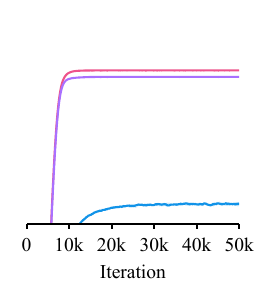}
    }
    \caption{
    \textbf{ELBO versus stepsize on \textsf{volatility-middle}}
    The solid lines are the median, while the shaded regions are the 80\% quantiles computed from 4 independent replications.
    }
    \vspace{-2ex}
\end{figure}

\begin{figure}[H]
    \centering
    \subfloat[\(T = 50\mathrm{k}\)]{
      \hspace{-1em}
      \includegraphics[scale=0.90]{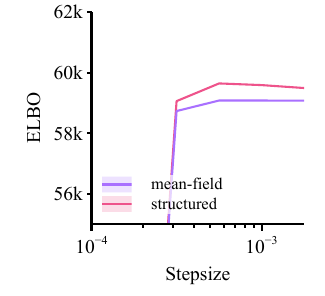}
    }
    \subfloat[\(\gamma = 10^{-3.5}\)]{
      \includegraphics[scale=0.90]{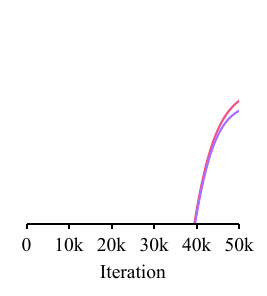}
    }
    \subfloat[\(\gamma = 10^{-3}\)]{
      \includegraphics[scale=0.90]{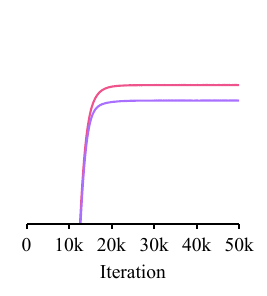}
    }
    \subfloat[\(\gamma = 10^{-2.75}\)]{
      \includegraphics[scale=0.90]{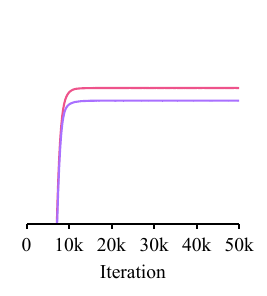}
    }
    \caption{
    \textbf{ELBO versus stepsize on \textsf{volatility-large}.}
    Full-rank is omitted as it didn't fit in memory.
    The solid lines are the median, while the shaded regions are the 80\% quantiles computed from 4 independent replications.
    }
\end{figure}

\subsection{Results on \textsf{irt}}

\begin{figure}[H]
    \centering
    \subfloat[\(T = 50\mathrm{k}\)]{
      \hspace{-1em}
      \includegraphics[scale=0.90]{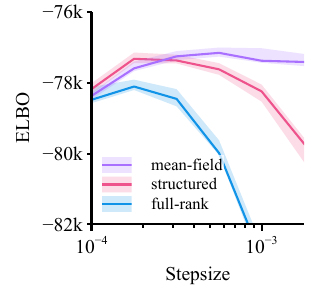}
    }
    \subfloat[\(\gamma = 10^{-4}\)]{
      \includegraphics[scale=0.90]{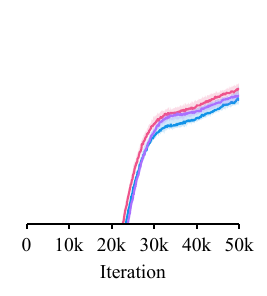}
    }
    \subfloat[\(\gamma = 10^{-3.5}\)]{
      \includegraphics[scale=0.90]{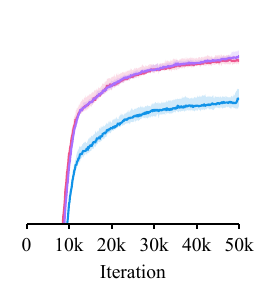}
    }
    \subfloat[\(\gamma = 10^{-3}\)]{
      \includegraphics[scale=0.90]{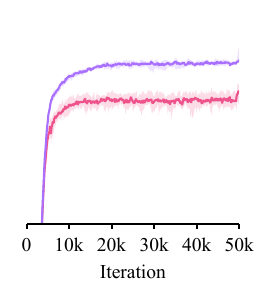}
    }
    \caption{
    \textbf{ELBO versus stepsize on \textsf{irt-small}.}
    The solid lines are the median, while the shaded regions are the 80\% quantiles computed from 4 independent replications.
    Notice that \textbf{structured} performs worse than mean-field at the largest stepsize (\(\gamma = 10^{-3}\)), but becomes comparable as we reduce the stepsize.
    }
    \vspace{-2ex}
\end{figure}

\begin{figure}[H]
    \centering
    \subfloat[\(T = 50\mathrm{k}\)]{
      \hspace{-1em}
      \includegraphics[scale=0.90]{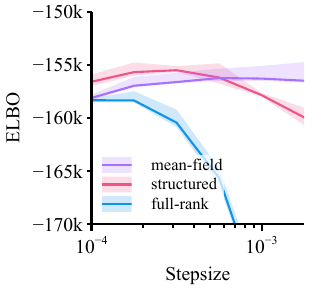}
    }
    \subfloat[\(\gamma = 10^{-4}\)]{
      \includegraphics[scale=0.90]{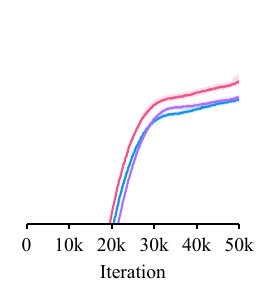}
    }
    \subfloat[\(\gamma = 10^{-3.5}\)]{
      \includegraphics[scale=0.90]{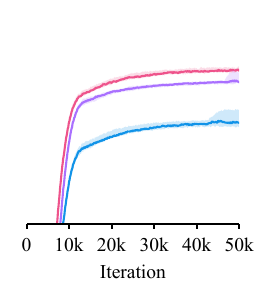}
    }
    \subfloat[\(\gamma = 10^{-3}\)]{
      \includegraphics[scale=0.90]{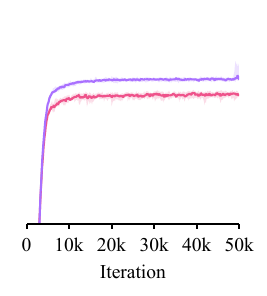}
    }
    \caption{
    \textbf{ELBO versus stepsize on \textsf{irt-middle}.}
    The solid lines are the median, while the shaded regions are the 80\% quantiles computed from 4 independent replications.
    Notice that \textbf{structured} performs slightly worse than \textbf{mean-field} at the largest stepsize (\(\gamma = 10^{-3}\)), but becomes superior as we reduce the stepsize.
    }
    \vspace{-2ex}
\end{figure}

\begin{figure}[H]
    \centering
    \subfloat[\(T = 50\mathrm{k}\)]{
      \hspace{-1em}
      \includegraphics[scale=0.90]{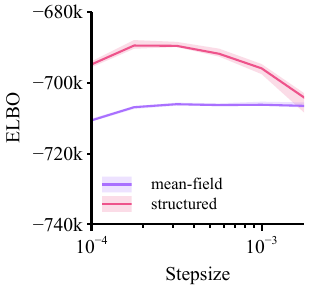}
    }
    \subfloat[\(\gamma = 10^{-4}\)]{
      \includegraphics[scale=0.90]{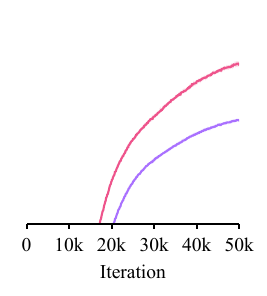}
    }
    \subfloat[\(\gamma = 10^{-3.5}\)]{
      \includegraphics[scale=0.90]{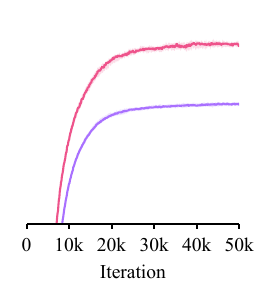}
    }
    \subfloat[\(\gamma = 10^{-3}\)]{
      \includegraphics[scale=0.90]{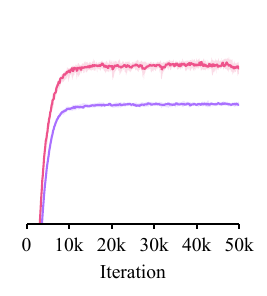}
    }
    \caption{
    \textbf{ELBO versus stepsize on \textsf{irt-large}}.
    Full-rank is omitted as it didn't fit in memory.
    The solid lines are the median, while the shaded regions are the 80\% quantiles computed from 4 independent replications.
    }
\end{figure}